\documentclass{article}

\usepackage{microtype}
\usepackage{graphicx}
\graphicspath{{figures/}{./}}
\usepackage{subfigure}
\usepackage{booktabs}
\usepackage{multirow}
\usepackage{placeins}
\usepackage{enumerate}
\PassOptionsToPackage{table}{xcolor}
\usepackage{xcolor}
\usepackage{tikz}
\usetikzlibrary{arrows.meta, positioning, calc, fit, backgrounds, decorations.pathreplacing}
\usepackage{hyperref}

\usepackage[accepted]{icml2026}

\usepackage{amsmath}
\usepackage{amssymb}
\usepackage{mathtools}
\usepackage{amsthm}
\usepackage{thmtools}
\usepackage{thm-restate}

\usepackage[capitalize,noabbrev]{cleveref}

\newcommand{\R}{\mathbb{R}}
\newcommand{\E}{\mathbb{E}}
\newcommand{\ones}{\mathbf{1}}
\newcommand{\Tne}{T_{\mathrm{NE}}}
\newcommand{\fwne}{f_\theta^{\mathrm{WNE}}}
\newcommand{\fwnea}{f^{\mathrm{WNE}}}
\newcommand{\defterm}[1]{#1}
\newcommand{\methodBaseline}{Baseline}
\newcommand{\methodWNE}{WNE}
\newcommand{\methodSEarch}{SE-arch}
\newcommand{\methodNEarch}{NE-arch}
\newcommand{\std}{\mathrm{std}}

\newcommand{\qualpanel}[3]{%
    \begin{tabular}[t]{@{}c@{}}
        \includegraphics[width=#1]{#2}\\[-0.2ex]
        \small #3
    \end{tabular}%
}

\theoremstyle{plain}
\newtheorem{proposition}{Proposition}

\newtheorem{lemma}{Lemma}
\theoremstyle{definition}

\newtheorem{assumption}{Assumption}
\newtheorem{observation}{Observation}
\theoremstyle{remark}
\newtheorem{remark}{Remark}

\AddToHook{env/proposition/begin}{\crefalias{section}{proposition}}
\AddToHook{env/characterization/begin}{\crefalias{section}{characterization}}
\AddToHook{env/lemma/begin}{\crefalias{section}{lemma}}
\AddToHook{env/definition/begin}{\crefalias{section}{definition}}
\AddToHook{env/assumption/begin}{\crefalias{section}{assumption}}
\AddToHook{env/observation/begin}{\crefalias{section}{observation}}
\AddToHook{env/remark/begin}{\crefalias{section}{remark}}

\icmltitlerunning{Normalization Equivariance for Arbitrary Backbones}

\begin{document}

\twocolumn[
    \icmltitle{Normalization Equivariance for Arbitrary Backbones, with Application to Image Denoising}

    \begin{icmlauthorlist}
        \icmlauthor{Youssef Saied}{yyy}
        \icmlauthor{François Fleuret}{yyy,comp}
    \end{icmlauthorlist}

    \icmlaffiliation{yyy}{Department of Computer Science, University of Geneva, Switzerland}
    \icmlaffiliation{comp}{Meta FAIR}
    \icmlcorrespondingauthor{Youssef Saied}{youssef.saied@unige.ch}

    \icmlkeywords{Machine Learning, ICML}

    \vskip 0.3in
]

\begin{NoHyper}
    \printAffiliationsAndNotice{}
\end{NoHyper}

\begin{abstract}
    Normalization Equivariance (NE) is a structural prior that improves robustness to distribution shift in image-to-image tasks.
    A function $f$ is normalization equivariant iff $f(a y + b\mathbf{1}) = a f(y) + b\mathbf{1}$ for all $a>0$ and $b\in\mathbb{R}$.
    Existing NE methods constrain every internal layer to NE-compatible operations.
    These constraints add runtime cost and exclude standard transformer components such as softmax attention and LayerNorm.
    We introduce Wrapped Normalization Equivariance (WNE), a parameter-free wrapper that normalizes the input, applies any backbone, and denormalizes the output.
    We prove every NE function admits this factorization, so the wrapper exactly parameterizes the class of NE functions.
    On blind denoising, wrapping CNN and transformer architectures improves robustness under noise-level mismatch with no measurable GPU overhead, while architectural NE baselines are up to $1.6\times$ slower.
\end{abstract}

\section{Introduction}%
Equivariance priors improve robustness to distribution shift in denoising and related image-to-image tasks.
Scale equivariance (SE) and the strictly stronger normalization equivariance (NE) are properties shared by many classical denoisers~\citep{rudin1992tv,buades2005review,dabov2007image,gu2014weighted}:

\begin{align}
     & \text{SE:} & f(ay) =          & af(y)          &  & \text{for } a>0.           \\
     & \text{NE:} & f(ay + b\ones) = & af(y) + b\ones &  & \text{for } a>0,\, b\in\R.
\end{align}

These equivariances were recently adopted to improve neural network robustness in blind denoising, where the noise level is unknown at test time and may differ from training~\citep{mohan2020biasfree,herbreteau2023normalization}.
This robustness has since motivated applications beyond denoising, including inverse problems and diffusion-style pipelines~\citep{kadkhodaie2021stochastic,hong2024provablepnp,guth2025learning,levac2025normalization}.

Existing methods enforce SE in neural networks by removing biases and leaving only positively 1-homogeneous operations~\citep{mohan2020biasfree}.
Enforcing NE additionally requires preserving the affine action \(y\mapsto ay+b\ones\) through every layer, e.g., via affine-constrained convolutions, special nonlinearities, and affine residuals~\citep{herbreteau2023normalization}.
Both approaches rule out standard transformer components—softmax and LayerNorm do not commute with \(y\mapsto ay+b\ones\).
The NE-specific layers also add runtime cost (Table~\ref{tab:speed_gpu_fdncnn}).

We show that internal-layer constraints are unnecessary.
Every NE function admits a \defterm{normalize-process-denormalize factorization} (Characterization~\ref{prop:ne_characterization}).
Using the characterization, we construct a parameter-free wrapper (\methodWNE{}) that enforces input-output NE around any backbone without modifying it, including transformers (Figure~\ref{fig:method}).
In blind denoising under noise-level mismatch, \methodWNE{} matches architectural NE trends on DnCNN.\@
\methodWNE{} also brings the same robustness pattern to SwinIR and Restormer.
In our GPU benchmarks, \methodWNE{} adds no measured overhead.

\begin{figure}[t]
    \centering
    \begin{tikzpicture}[
            >=Stealth,
            normbox/.style={
                    draw, rounded corners=3pt,
                    minimum width=1.0cm, minimum height=0.75cm,
                    align=center, fill=orange!18,
                    font=\footnotesize,
                    draw=orange!50
                },
            backbone/.style={
                    draw, rounded corners=4pt,
                    minimum width=1.5cm, minimum height=0.9cm,
                    align=center, fill=blue!12,
                    font=\normalsize,
                    draw=blue!50,
                    line width=0.9pt
                },
            io/.style={font=\normalsize},
            statslabel/.style={font=\scriptsize, text=black!55},
            arrow/.style={->, semithick, black!65},
            statsarrow/.style={->, thin, black!40},
            wrapperbox/.style={draw, rounded corners=8pt, line width=0.7pt, black!30},
        ]

        \node[io] (y) at (0, 0) {\( y \)};
        \node[normbox, right=1cm of y] (norm) {\( T_{\mathrm{NE}} \)};
        \node[backbone, right=1cm of norm] (g) {\( g_\theta \)};
        \node[normbox, right=1cm of g] (denorm) {\( T_{\mathrm{NE}}^{-1} \)};
        \node[io, right=0.8cm of denorm] (xhat) {\( \hat{x} \)};

        \coordinate (forkpt) at ($(norm.west)+(-0.4cm,0)$);
        \coordinate (bypassy) at ($(forkpt)+(0,-0.7cm)$);
        \coordinate (bypassend) at (bypassy -| denorm.south);

        \begin{scope}[on background layer]
            \node[wrapperbox, fill=black!3,
                fit=(norm)(g)(denorm)(bypassy)(bypassend),
                inner xsep=10pt, inner ysep=12pt] (wrapper) {};
        \end{scope}

        \draw[arrow] (y) -- (norm);
        \draw[arrow] (norm) -- node[above, font=\scriptsize, text=black!60] {\( \tilde{y} \)} (g);
        \draw[arrow] (g) -- (denorm);
        \draw[arrow] (denorm) -- (xhat);

        \draw[statsarrow, rounded corners=5pt]
        (forkpt) -- (bypassy) -- (bypassend) -- (denorm.south);

        \node[statslabel] at ($(g.south)+(0,-0.45cm)$) {\( \mu(y),\;\std(y) \)};
        \fill[black!45] (forkpt) circle (1.2pt);

    \end{tikzpicture}
    \caption{\textbf{Normalization-equivariant wrapper (\emph{\methodWNE{}}).}
        Given a noisy input \(y\), we compute instance statistics \(\mu(y)\) and \(\std(y)\), normalize to \(\tilde y = T_{\mathrm{NE}}(y)\), apply an arbitrary backbone \(g_\theta\), and denormalize with the stored statistics to produce \(\hat x = \std(y)\,g_\theta(\tilde y)+\mu(y)\ones\).
        This enforces input-output normalization equivariance without modifying the backbone.}%
    \label{fig:method}
    \vspace{-1em}
\end{figure}
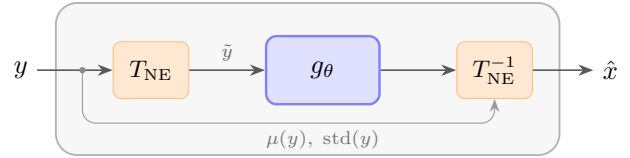

\paragraph{Our contributions:}%
\begin{itemize}
    \vspace{-1em}
    \item \textbf{Method:} parameter-free drop-in wrapper enforcing NE on any backbone, including transformers (Section~\ref{subsec:ne_wrapper}).
          \newpage
    \item \textbf{Theory:} characterization showing that normalize-process-denormalize completely parameterizes the NE function class (Section~\ref{subsec:characterization}).
    \item \textbf{Experiments and analysis:} mismatch-robustness gains across CNN and transformer denoisers with negligible runtime cost; a normalized-coordinate analysis explains the trend (Sections~\ref{sec:experiments}--\ref{sec:analysis}).
\end{itemize}

\paragraph{Code:}%
The implementation is available in the repository
\href{https://github.com/YoussefSaied/normalization_equivariance}{\texttt{YoussefSaied/normalization\_equivariance}}.

\section{Related Work}%
\label{sec:related}
\begin{figure*}[t]
    \centering
    \begingroup
    \newcommand{\pixelpng}[2]{%
        \leavevmode
        \pdfximage width #1 attr {/Interpolate false} {#2}%
        \pdfrefximage\pdflastximage
    }
    \newcommand{\swinirqualpanel}[2]{%
        \begin{tabular}[t]{@{}c@{}}
            \pixelpng{1.55in}{#1}\\[-0.2ex]
            \small #2
        \end{tabular}%
    }

    \setlength{\tabcolsep}{0pt}%
    \begin{tabular}{@{}c@{\hspace{0.12in}}c@{\hspace{0.12in}}c@{\hspace{0.12in}}c@{}}
    \swinirqualpanel{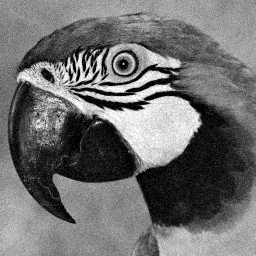}{\shortstack{Noisy training image\\\(\sigma=10\)}} &
    \swinirqualpanel{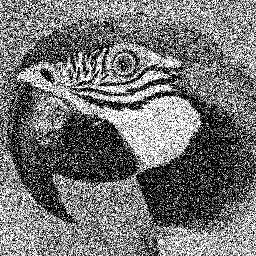}{\shortstack{Noisy test image\\\(\sigma=90\)}} &
    \swinirqualpanel{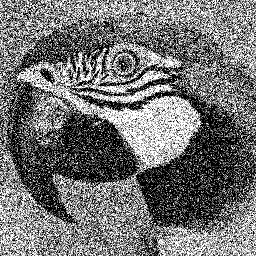}{\shortstack{Test image denoised by\\SwinIR}} &
    \swinirqualpanel{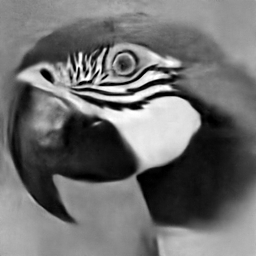}{\shortstack{Test image denoised by\\SwinIR-WNE (ours)}}
    \end{tabular}
    \endgroup
    \caption{\textbf{SwinIR qualitative denoising under noise-level mismatch.}
        Models are trained at \(\sigma_{\mathrm{train}}=10\) and evaluated at \(\sigma_{\mathrm{test}}=90\), far beyond the training noise.
        The unwrapped \methodBaseline{} SwinIR degrades severely, while \emph{SwinIR-\methodWNE{}} remains stable.
    }%
    \label{fig:swinir_qualitative}
\end{figure*}

\paragraph{Equivariance priors: scale and normalization equivariance.}%
Two related equivariance priors appear in denoising and related tasks.

Scale equivariance (SE), or positive \(1\)-homogeneity, requires \(f(ay)=af(y)\) for \(a>0\).
\citet{mohan2020biasfree} propose SE bias-free CNN denoisers, inspiring similar bias-free designs in later work~\citep{zhang2021plug,zamir2022restormer}.
Scale equivariant denoisers also serve as implicit priors in inverse problems~\citep{kadkhodaie2021stochastic} and appear in energy-based probability density and diffusion models~\citep{guth2025learning}.

Normalization equivariance (NE) is strictly stronger than SE, requiring \(f(ay+b\ones)=af(y)+b\ones\) for \(a>0,b\in\R\).
\citet{herbreteau2024survey} survey NE properties across supervised and self-supervised methods, noting that some classical denoisers (e.g.,~TV, non-local means) are NE~\citep{rudin1992tv,buades2005review} while others such as BM3D~\citep{dabov2007image} and WNNM~\citep{gu2014weighted} satisfy SE but not NE\@.
\citet{herbreteau2023normalization} introduce NE as a robustness prior for denoising, and NE predictors have also been leveraged in plug-and-play MRI reconstruction~\citep{hong2024provablepnp} and diffusion-based posterior samplers learned from noisy measurements~\citep{levac2025normalization}.

\paragraph{Enforcing normalization equivariance.}%
For clarity, we describe three levels at which NE can be enforced.

\textsc{Architecture-level:} \citet{herbreteau2023normalization} constrain every internal operation to NE-compatible families (affine-constrained convolutions, sorting nonlinearities, affine residuals); this is incompatible with attention and LayerNorm.

\textsc{Objective-level:} \citet{levac2025normalization} embed an approximate NE property into a SURE-based loss to learn diffusion models from noisy measurements alone (no clean targets), enabling denoiser generalization below the measurement noise level. Their setting is self-supervised and complementary to ours.
\newpage

\textsc{Input-output-level:} (this work) we enforce NE analytically around an arbitrary backbone.
Our if and only if characterization (Characterization~\ref{prop:ne_characterization}) shows that normalize-process-denormalize reparameterizes the full NE function class rather than restricting it.

\paragraph{Canonicalization and normalize-process-denormalize.}%
A general recipe for equivariant maps is to canonicalize the input to a group-orbit representative, apply an arbitrary function, and map back~\citep{kaba2023equivariance,mondal2023equivariant}.
In our setting, the orbit representative is obtained by removing the global brightness and contrast degrees of freedom.
Characterization~\ref{prop:ne_characterization} shows that this construction parameterizes the full NE function class for the affine action \(y\mapsto ay+b\ones\).
Reversible Instance Normalization (RevIN) uses a related normalize-process-denormalize template, motivated empirically by distribution shift in time-series forecasting~\citep{kim2022revin,liu2022non}.
We show that the same template, with global \(\mu\) and \(\std\) statistics, is exactly the form taken by every normalization-equivariant map (Characterization~\ref{prop:ne_characterization}).

\paragraph{Single-noise blind denoising.}%
We use blind denoising as a controlled stress test for robustness to unknown corruption strength.
Following \citet{mohan2020biasfree,herbreteau2023normalization}, we train at a single \(\sigma_{\mathrm{train}}\) and test at \(\sigma_{\mathrm{test}}\neq\sigma_{\mathrm{train}}\), probing out-of-distribution generalization rather than interpolation within a trained \(\sigma\)-range.
The same robustness matters in plug-and-play reconstruction, where \(\sigma\) controls the denoising strength~\citep{zhang2021plug} and is usually a tuned hyperparameter.
In dynamic preconditioned PnP this is sharpened: the iteration-dependent preconditioner means the right \(\sigma_k\) changes every iteration and is not known a priori.
\citet{hong2024provablepnp} use an NE denoiser, which allows a single network to absorb this variation, avoiding per-iteration \(\sigma\)-tuning.

\section{Method}%
\label{sec:method}

\paragraph{Setup and notation.}%
An instance is the tensor passed to the backbone: during training we use patches and at test time we use full images.
For an image of size \(C \times H \times W\), we view an instance \(y\) as a vector in \(\R^d\) with \(d=CHW\).
All definitions below apply per instance for arbitrary spatial sizes \(H,W\).
Let \(\ones\in\R^d\) denote the all-ones instance.
In all our experiments, backbones are size-agnostic, so \(g_\theta\) is defined for any \(H,W\).

\paragraph{Pooling domain.}%
All statistics below are computed jointly over all entries (all channels, all pixels).
This choice matches the NE group action \(y\mapsto ay+b\ones\), which acts identically on all entries.
Pooling statistics per channel would instead correspond to enforcing a larger, channel-wise affine equivariance, which we do not assume.
Thus, the pooling domain is part of the equivariance specification, not an implementation convention.

For any single instance \(y\in\R^d\), define the empirical mean and standard deviation:
\begin{equation}
    \mu(y) := \frac{1}{d}\sum_{i=1}^d y_i,
    \qquad
    \std(y) := \frac{1}{\sqrt d}\,\|y-\mu(y)\ones\|_2.
\end{equation}
We also define the normalized manifold:
\begin{align}
    \mathcal{M}
    := & \{ z\in\R^d : \mu(z)=0,\ \std(z)=1 \}                      \\
    =  & \{ z\in \mathrm{Span}{(\ones)}^\perp : \|z\|_2=\sqrt d \}.
\end{align}

\subsection{Normalization-Equivariant Wrapper}%
\label{subsec:ne_wrapper}

We define the NE normalization map (with a constant-instance guardrail)
\begin{equation}
    \Tne(y)
    :=
    \begin{cases}
        \frac{y-\mu(y)\ones}{\std(y)} & \text{if } \std(y)>0, \\[6pt]
        0                             & \text{if } \std(y)=0,
    \end{cases}
\end{equation}
so that \(\Tne(y)\in\mathcal{M}\) whenever \(\std(y)>0\).
Given any backbone \(g_\theta:\R^d\to\R^d\), our NE-wrapped map \methodWNE{} is
\begin{equation}
    \fwne(y)
    := \std(y)\,g_\theta \big(\Tne(y)\big)+\mu(y)\ones.
    \label{eq:ne_wrapper}
\end{equation}
By construction, \(\fwne\) is normalization-equivariant: for all \(a>0\) and \(b\in\R\),
\begin{equation}
    \fwne(ay+b\ones)=a\,\fwne(y)+b\ones,
\end{equation}
with a short proof in Appendix~\ref{app:wrapper_ne}.
This guarantee is independent of the backbone internals because the wrapper removes and re-injects the global shift and scale via \(\mu(y)\) and \(\std(y)\), so \(g_\theta\) is only queried on normalized inputs \(\Tne(y)\in\mathcal{M}\) (when \(\std(y)>0\)).
In particular, \(g_\theta\) may use components incompatible with internal-layer NE---softmax attention, LayerNorm, and BatchNorm---without breaking the equivariance identity.
An equivalent residual-prediction formulation is given in Appendix~\ref{app:direct_vs_residual}.

\paragraph{Numerical stability.}
Our theory assumes the ideal wrapper with a guardrail for constant inputs (\(\std(y)=0\)).
For numerical stability, the implementation uses \(\std_{\varepsilon}(y):=\std(y)+\varepsilon\) (\(\varepsilon=10^{-5}\)).
The resulting map is approximately NE\@: it matches the ideal wrapper whenever \(\mathrm{std}(y)\) is not close to \(0\), and differs only in the near-constant regime where stabilization is active.

\paragraph{Training objective.}
We train by minimizing raw-space MSE\@:
\begin{equation}
    \min_\theta\ \E\big[ \| \fwne(y)-x\|_2^2 \big].
\end{equation}

\paragraph{Relationship to Instance Normalization.}
The normalization \(\Tne(y)=(y-\mu(y)\ones)/\std(y)\) has the same algebraic form as instance normalization when statistics are pooled over all channels and pixels (as opposed to the common per-channel variant).
The key difference is that the wrapper reuses the input instance statistics to analytically invert the transform at the output~\eqref{eq:ne_wrapper}.

\subsection{Characterization of NE Maps}%
\label{subsec:characterization}
Our method and analysis rely on the following complete characterization of normalization-equivariant maps.
\begin{restatable}[Characterization of NE maps]{characterization}{nechar}%
    \label{prop:ne_characterization}
    A function \(f:\R^d\to\R^d\) is normalization-equivariant if and only if there exists a function \(g:\mathcal{M}\to\R^d\) such that, for all \(y\) with \(\std(y)>0\),
    \[
        f(y)=\std(y)\,g\big(\Tne(y)\big)+\mu(y)\ones.
    \]
    Moreover, for \(\std(y)=0\) one necessarily has \(f(y)=\mu(y)\ones\), and \(g\) is uniquely determined by \(f\) on \(\mathcal{M}\).
\end{restatable}

\noindent Proofs are in Appendix~\ref{app:proofs}.
(Equivalently, an NE map is determined by its restriction to \(\mathcal{M}\), together with the forced behavior on constant inputs.)

\begin{table*}[t]
    \centering
    \setlength{\belowcaptionskip}{3pt}
    \caption{\textbf{Matched-noise PSNR (dB).}
        Average PSNR on Set12 and BSD68 for AWGN with \(\sigma\in\{10,25,50\}\) (8-bit units).
        Each model is trained at a single \(\sigma\) and evaluated at the same \(\sigma\). Our \methodWNE{} remains competitive while providing robustness under mismatch (Figures~\ref{fig:dncnn_psnr}--\ref{fig:swinir_psnr}).
        \emph{\methodSEarch{}} and \emph{\methodNEarch{}} are FDnCNN-based architectural variants (no BatchNorm, constrained layers), while \emph{\methodWNE{}} wraps the standard backbone.
    }%
    \label{tab:psnr_matched}
    \small
    \setlength{\tabcolsep}{3.5pt}
    \renewcommand{\arraystretch}{1.05}
    \begin{tabular}{llcccccc}
        \toprule
                                                           &                                     & \multicolumn{3}{c}{Set12} & \multicolumn{3}{c}{BSD68}                                 \\
        \multicolumn{2}{c}{Noise level \(\sigma\) (8-bit)} & 10                                  & 25                        & 50                        & 10    & 25    & 50            \\
        \midrule
        \multirow{4}{*}{DnCNN family~\citep{zhang2017dncnn}}
                                                           & \emph{\methodBaseline{}}            & 34.92                     & 30.50                     & 27.17 & 33.92 & 29.19 & 26.18 \\
                                                           & \emph{\methodSEarch{}}              & 34.41                     & 30.42                     & 27.19 & 33.51 & 29.12 & 26.12 \\
                                                           & \emph{\methodNEarch{}}              & 34.65                     & 30.37                     & 27.20 & 33.72 & 29.11 & 26.18 \\
                                                           & \textbf{\emph{\methodWNE{}}} (ours) & 34.77                     & 30.43                     & 27.22 & 33.80 & 29.15 & 26.20 \\
        \midrule
        \multirow{2}{*}{SwinIR~\citep{liang2021swinir}}
                                                           & \emph{\methodBaseline{}}            & 35.08                     & 30.79                     & 27.61 & 34.01 & 29.38 & 26.46 \\
                                                           & \textbf{\emph{\methodWNE{}}} (ours) & 34.95                     & 30.65                     & 27.47 & 33.94 & 29.30 & 26.37 \\
        \bottomrule
    \end{tabular}

\end{table*}

\section{Experimental Results}%
\label{sec:experiments}

\begin{figure*}[t]
    \centering
    \subfigure[\(\sigma_{\mathrm{train}} = 10\)]{
        \includegraphics[width=0.3\linewidth]{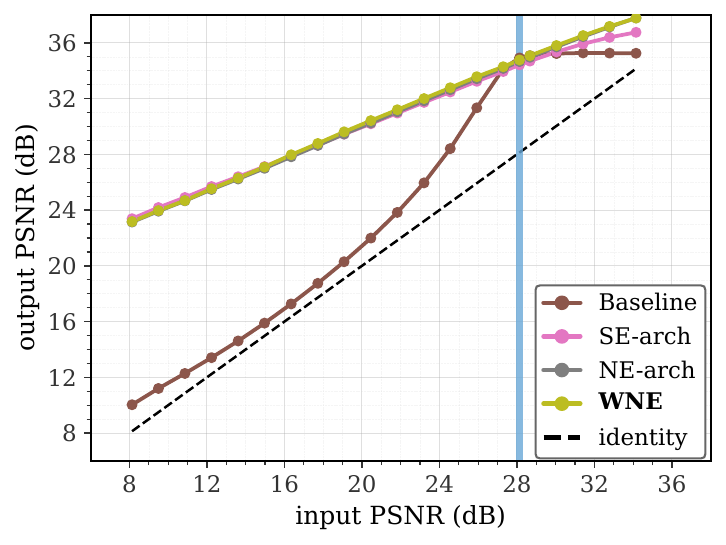}
    }\hfill
    \subfigure[\(\sigma_{\mathrm{train}} = 25\)]{
        \includegraphics[width=0.3\linewidth]{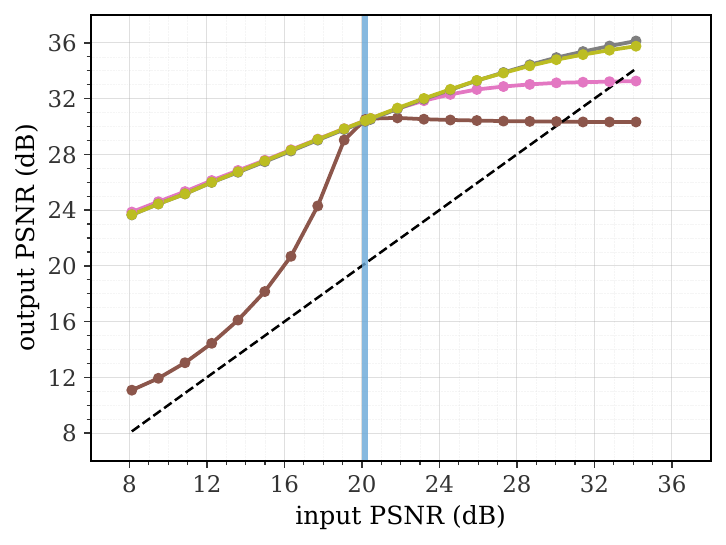}
    }\hfill
    \subfigure[\(\sigma_{\mathrm{train}} = 50\)]{
        \includegraphics[width=0.3\linewidth]{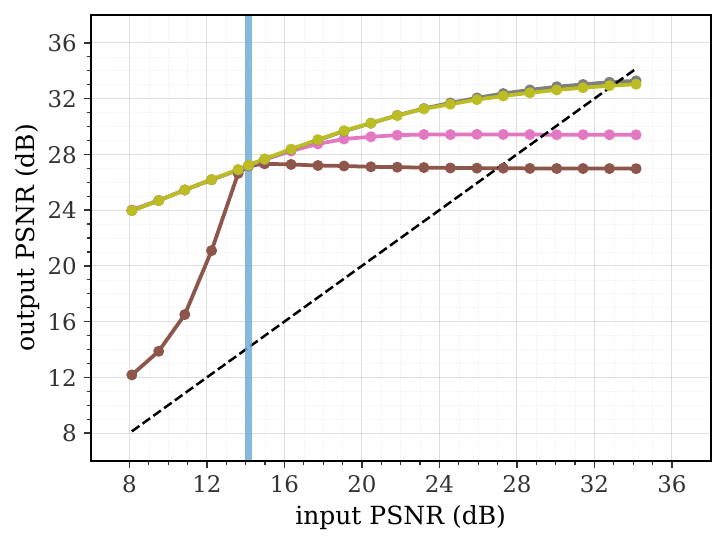}
    }
    \caption{
        Output PSNR versus input PSNR on Set12 for models trained at a single \(\sigma_{\mathrm{train}}\) (vertical reference line) and tested at varying \(\sigma_{\mathrm{test}}\), following \citet{herbreteau2023normalization}.
        The \methodBaseline{} is noise-level specific, while equivariance stabilizes performance under mismatch.
        Our \emph{\methodWNE{}} closely matches the mismatch trend of the architectural normalization-equivariant reference (\methodNEarch{}, FDnCNN-based), and improves over the architectural scale-equivariant reference (\methodSEarch{}) and the unwrapped baseline.
        Dashed line: no denoising (\(\hat x=y\)).
    }%
    \label{fig:dncnn_psnr}
\end{figure*}

\begin{figure*}[t]
    \centering
    \subfigure[\(\sigma_{\mathrm{train}} = 10\)]{
        \includegraphics[width=0.3\linewidth]{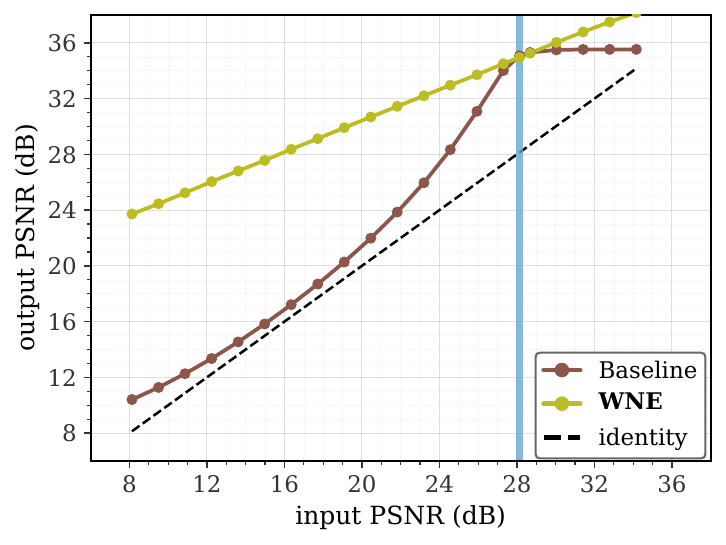}
    }\hfill
    \subfigure[\(\sigma_{\mathrm{train}} = 25\)]{
        \includegraphics[width=0.3\linewidth]{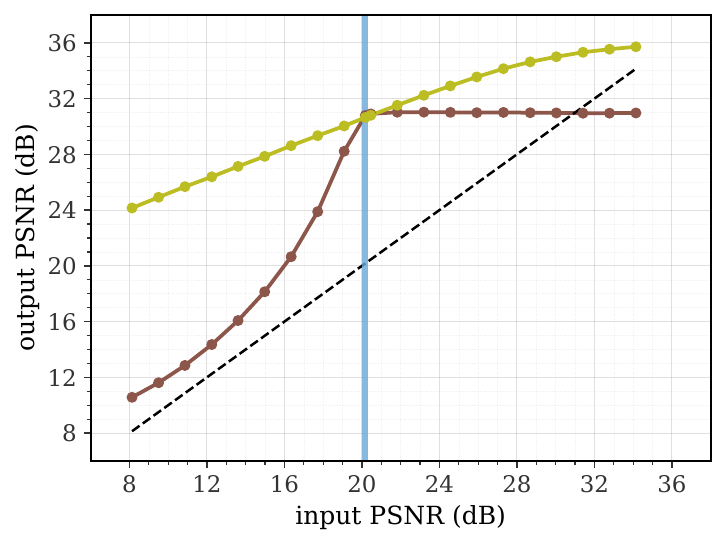}
    }\hfill
    \subfigure[\(\sigma_{\mathrm{train}} = 50\)]{
        \includegraphics[width=0.3\linewidth]{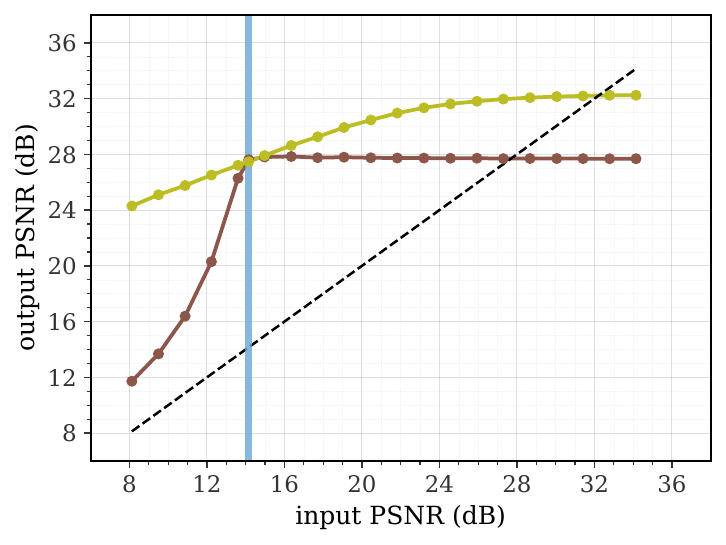}
    }
    \caption{
        \textbf{Noise-level mismatch on SwinIR.}
        Same diagnostic as Figure~\ref{fig:dncnn_psnr}.
        The \methodBaseline{} SwinIR degrades once \(\sigma_{\mathrm{test}}\neq\sigma_{\mathrm{train}}\), while \emph{SwinIR-\methodWNE{}} stabilizes performance away from the training noise.
        Dashed line: no denoising (\(\hat x=y\)).
    }%
    \label{fig:swinir_psnr}
\end{figure*}

We test whether input-output NE enforcement improves generalization away from the training noise level under a single-noise mismatch diagnostic.

\paragraph{Protocol and metric.}
We follow the single-noise diagnostic of \citet{herbreteau2023normalization}, the single-point limit of the limited-range training used to study mismatch in \mbox{\citet{mohan2020biasfree}}.
Each model is trained on additive white Gaussian noise (AWGN) at a fixed \(\sigma_{\mathrm{train}}\in\{10,25,50\}\) (8-bit units) and evaluated on Set12
over a wide range of \(\sigma_{\mathrm{test}}\).
For each test image, we compute input PSNR (noisy \(y\) vs.\ clean \(x\)) and output PSNR (denoised \(\hat x\) vs.\ \(x\)).
We then plot output PSNR versus input PSNR (Figures~\ref{fig:dncnn_psnr} and~\ref{fig:swinir_psnr}).
The vertical line indicates the training noise level: moving away from the vertical line corresponds to testing outside the training noise range.
The dashed identity line corresponds to leaving the input unchanged.

\paragraph{Models and compared variants.}
We evaluate two widely used reference backbones from distinct families: DnCNN~\citep{zhang2017dncnn} (CNN) and SwinIR~\citep{liang2021swinir} (transformer-based restoration).
We focus on mismatch trends rather than peak PSNR.\@
For each family, we compare the unmodified backbone (\methodBaseline{}) to its wrapped counterpart (\methodWNE{}).
Training details and implementation in Appendix~\ref{app:arch_impl} and Appendix~\ref{app:training_details}.
Additional diagnostics in the appendix cover input-only normalization versus the full wrapper (Appendix~\ref{app:in_vs_wne}), SSIM diagnostics (Appendix~\ref{app:ssim}), soft-NE and multi-noise controls with explicit equivariance-defect measurements (Appendices~\ref{app:soft_ne}--\ref{app:ne_violation}), and robustness beyond AWGN (Appendix~\ref{app:nongaussian}).
Breadth checks extend the evaluation to FDnCNN-family controls (Appendix~\ref{app:fdncnn_controlled}), Restormer and color denoising (Appendices~\ref{app:restormer}--\ref{app:restormer_color}), real sensor noise on SIDD (Appendix~\ref{app:sidd}), and a Noise2Noise setting where denoisers are reused for iterative sampling and inverse-problem solving (Appendix~\ref{app:n2n_wne}).

\paragraph{DnCNN family (CNN-based).}
We compare (i) \methodBaseline{}, the standard DnCNN implementation (with BatchNorm), and (ii) \methodWNE{}, our wrapper applied to the same DnCNN backbone.
We also report the architectural equivariant references (iii) \methodSEarch{} and (iv) \methodNEarch{} from \citet{herbreteau2023normalization}, built on FDnCNN with layer replacements throughout (Appendix~\ref{app:arch_variants}).
These are architectural references rather than controlled same-backbone ablations: relative to standard DnCNN, they remove BatchNorm and bias, and \methodNEarch{} also constrains the admissible filter class.
Appendix~\ref{app:fdncnn_controlled} brackets this confound with an FDnCNN-family comparison, where \methodWNE{}-FDnCNN still tracks \methodNEarch{} under mismatch.

\paragraph{SwinIR (transformer-based).}
For SwinIR, we compare the \methodBaseline{} and \methodWNE{}.
We include SwinIR as a modern transformer restoration backbone to test whether the wrapper improves mismatch robustness beyond CNN architectures.
Appendix~\ref{app:restormer} shows the same mismatch-stabilization pattern on Restormer, a second transformer backbone, and Appendix~\ref{app:restormer_color} extends the same construction to color denoising.
For transformers, architectural SE/NE would require modifying attention/LayerNorm; we therefore compare \methodBaseline{} vs.\ \methodWNE{}.\@

\subsection{DnCNN:\@ Input-Output NE Matches Architectural NE Trends}
Across all three training settings \(\sigma_{\mathrm{train}}\in\{10,25,50\}\), Figure~\ref{fig:dncnn_psnr} shows that
\methodWNE{} produces mismatch curves that closely match the robustness trend of the architectural normalization-equivariant \methodNEarch{} reference.
This indicates that enforcing NE at the input-output level recovers essentially the same robustness as architectural NE without constraining internal layers.
Relative to \methodSEarch{}, we observe an asymmetric pattern: when training at higher noise and testing at lower noise (high-to-low), \methodWNE{} improves over \methodSEarch{}; when training at lower noise and testing at higher noise (low-to-high), \methodWNE{} is comparable to \methodSEarch{}.
Overall, the additional shift-equivariance in NE yields consistent gains over SE in the high-to-low regime.

\subsection{SwinIR:\@ Improved Extrapolation for a Transformer Backbone}
Figure~\ref{fig:swinir_psnr} shows that SwinIR exhibits pronounced noise-level specificity under single-noise training:
the baseline curve flattens once \(\sigma_{\mathrm{test}}\neq \sigma_{\mathrm{train}}\), indicating poor extrapolation outside the training noise.
Wrapping SwinIR with \methodWNE{} stabilizes the mapping, yielding a smoother and more monotonic output-PSNR trend across low and high input noise.

Figure~\ref{fig:swinir_qualitative} illustrates this effect in a severe mismatch setting:
models trained at \(\sigma_{\mathrm{train}}=10\) are evaluated at a much higher noise level (\(\sigma_{\mathrm{test}}=90\)).
The \methodBaseline{} SwinIR degrades visibly, while the wrapped model remains qualitatively stable, consistent with the trend in Figure~\ref{fig:swinir_psnr}.

\subsection{The Wrapper Does Not Degrade Performance}
At matched noise levels (\(\sigma_{\text{test}}=\sigma_{\text{train}}\)), \methodWNE{} remains close to the unwrapped baseline for both backbones (Table~\ref{tab:psnr_matched}). We report results at the final training checkpoint for all models (no early stopping) to keep evaluation consistent across variants; our primary focus is robustness under mismatch (Figures~\ref{fig:dncnn_psnr} and~\ref{fig:swinir_psnr}) rather than peak matched-noise PSNR.\@

\FloatBarrier
\subsection{The Wrapper Has Negligible Overhead on GPU}%
\label{subsec:runtime}
Table~\ref{tab:speed_gpu_fdncnn} reports wall-clock seconds per batch for FDnCNN variants.
On GPU, the \methodWNE{} wrapper adds negligible overhead relative to the baseline, since it only introduces per-instance reductions (\(\mu,\std\)) and elementwise affine operations around the backbone.
The architectural NE reference \methodNEarch{} is slower in our benchmark, consistent with the cost of enforcing equivariance throughout internal layers rather than via an outer analytic transform.
Similar findings were also reported by \citet{herbreteau2023normalization}.
We observe the same behavior for SwinIR:\@ on GPU, \methodWNE{} matches the baseline within measurement noise (\(1.00\times\) backward, \(1.00\times\) inference; Appendix Table~\ref{tab:speed_swinir}).
Full GPU/CPU timings and the benchmarking protocol are reported in Appendix~\ref{app:speed_protocol}.

\begin{table}[!htbp]
    \centering
    \caption{\textbf{Runtime overhead for FDnCNN on GPU (seconds per batch).}
        Timing on inputs of shape \(16 \times 1 \times 128 \times 128\) with warmup \(=3\) and timed steps \(=10\).
        “Backward” measures one training iteration (forward + loss + backward), and “Inference” measures forward-only.
        Values in parentheses are ratios relative to the baseline.
        Hardware: \texttt{NVIDIA GeForce RTX 4090}.}%
    \label{tab:speed_gpu_fdncnn}
    \small
    \setlength{\tabcolsep}{6pt}
    \begin{tabular}{lcc}
        \toprule
        Variant                      & Backward (s)           & Inference (s)          \\
        \midrule
        \emph{\methodBaseline{}}     & 0.034                  & 0.011                  \\
        \emph{\methodNEarch{}}       & 0.054 (1.60\(\times\)) & 0.019 (1.69\(\times\)) \\
        \textbf{\emph{\methodWNE{}}} & 0.033 (0.99\(\times\)) & 0.011 (1.00\(\times\)) \\
        \bottomrule
    \end{tabular}
\end{table}

\begin{figure*}[t]
    \centering
    \includegraphics[width=\textwidth]{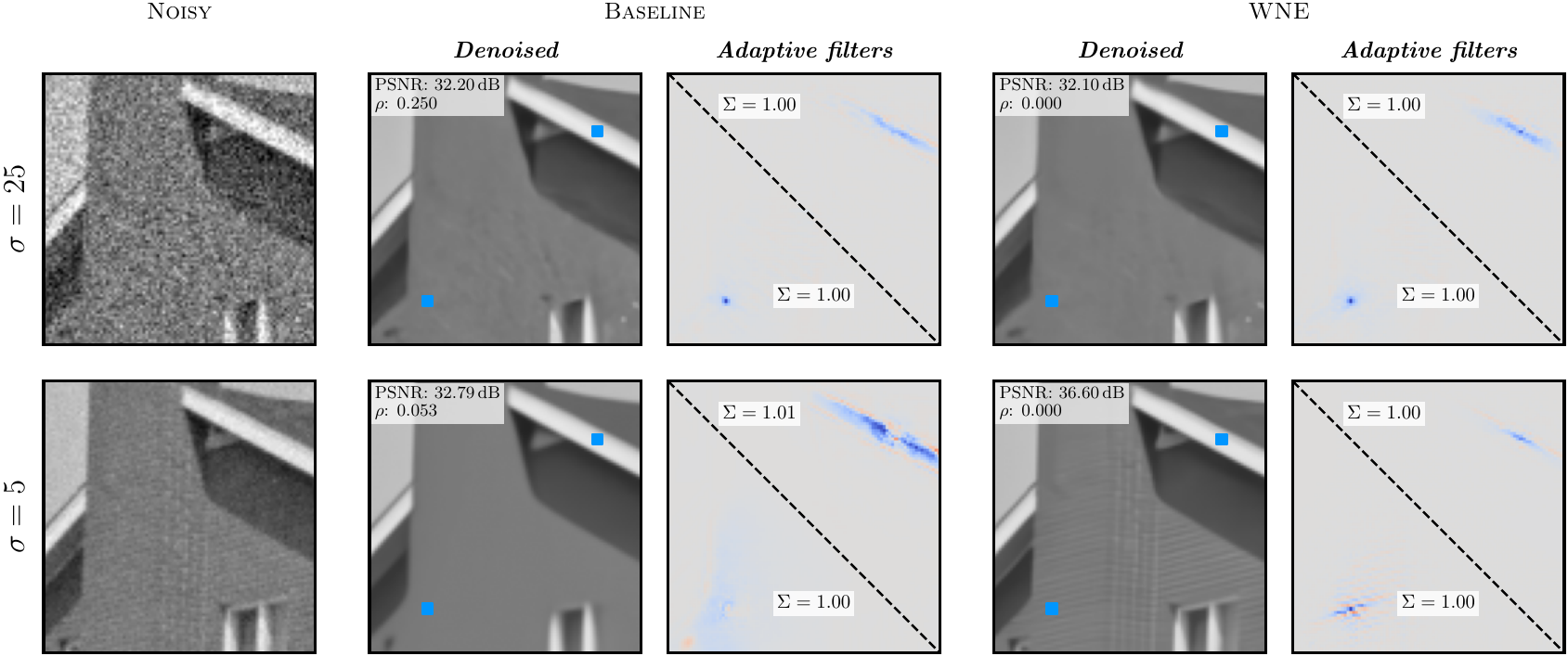}%
    \caption{\textbf{Input-adaptive (affine) filters from the Jacobian.}
        We visualize two rows of \(J_f(y)\) as per-output-pixel filters.
        For \emph{\methodWNE{}}, NE implies positive \(1\)-homogeneity, so at differentiable points \(f(y)=J_f(y)\,y\) and the filters can be read literally~\citep{mohan2020biasfree,herbreteau2023normalization}.
        Under mismatch (\(\sigma=5\)), the \methodBaseline{} SwinIR shows diffuse filters and oversmoothing, while SwinIR-\emph{\methodWNE{}} retains structured, edge-aligned filters and improves PSNR.\@
        Both networks were trained for Gaussian noise at noise level \(\sigma=25\) exclusively.
    }%
    \label{fig:swinir_adaptive_filters}
\end{figure*}

\subsection{Adaptive-Filter Interpretation Extends to Transformers}%
\label{subsec:adaptive_filters}
Figure~\ref{fig:swinir_adaptive_filters} visualizes two Jacobian rows \({(J_f(y))}_{i,:}\) of the denoiser \(f\) as spatial filters (reshaped to the input grid), following \citet{mohan2020biasfree,herbreteau2023normalization}.
If \(f\) is positively \(1\)-homogeneous and differentiable at \(y\), Euler's theorem gives
\begin{equation}
    f(y)=J_f(y)\,y,
\end{equation}
so each output pixel \(f{(y)}_i\) is exactly the inner product of \(y\) with its local filter \({(J_f(y))}_{i,:}\), making these rows literal input-adaptive filters.
Because \methodWNE{} enforces NE (hence SE), this interpretation applies to SwinIR despite attention and LayerNorm.
Moreover, NE implies shift equivariance, and at differentiable points this yields
\begin{equation}
    J_f(y)\,\ones=\ones,
\end{equation}
so each filter has unit sum and the local model is an adaptive affine combination of input pixels.
To verify the Euler identity in the displayed examples, we report
\begin{equation}
    \rho(y):=\frac{\|f(y)-J_f(y)\,y\|_2}{\|f(y)\|_2},
\end{equation}
which is \(0\) for differentiable positively \(1\)-homogeneous maps (up to numerical error).
More broadly, since the wrapper enforces SE/NE at the input-output level independently of the backbone, Jacobian-based analyses developed for bias-free SE denoisers~\citep{mohan2020biasfree,herbreteau2023normalization} extend directly to any wrapped architecture.
This includes recent generalization analyses based on harmonic structure~\citep{kadkhodaie2024harmonic}.

\section{Analysis}%
\label{sec:analysis}
We analyze a mechanism that explains the noise-level mismatch robustness of normalization-equivariant (NE) denoisers.
The analysis applies to any NE denoiser (architectural, classical, or wrapped).
We focus on (i) the stability of output PSNR under noise-level mismatch and (ii) why performance degrades differently when testing above versus below the training noise level.

All statistics reported in this section were computed using training patches (using the same training corpus as Section~\ref{sec:experiments}), corrupted with AWGN at the stated \(\sigma\).

\subsection{NE Reduces Denoising to Normalized Regression}

By Characterization~\ref{prop:ne_characterization}, any NE map \(f:\R^d\to\R^d\) is entirely determined by a function \(g:\mathcal{M}\to\R^d\) via
\begin{equation}
    f(y)=\std(y)\,g\bigl(\Tne(y)\bigr)+\mu(y)\ones.
\end{equation}
The normalization and denormalization steps are fixed analytic operations; only the map \(g\) acting on normalized inputs \(\tilde y \in \mathcal{M}\) varies across NE maps.

This factorization makes the source of robustness (or failure) transparent: cross-noise behavior is governed by how \(g\) performs on normalized pairs \((\tilde y, \tilde x)\) as \(\sigma\) varies.

\paragraph{PSNR decomposes into normalized and scale terms.}
A direct consequence of the factorization is that PSNR decomposes into a normalized error term and a scale term.

For a clean target \(x\) and noisy observation \(y\), define the matched normalized target:
\begin{equation}
    \tilde x := \Tne(x\,;\,y) := \frac{x - \mu(y)\ones}{\std(y)}.
\end{equation}
Then (Appendix~\ref{app:loss_identity}):
\begin{equation}
    f(y) - x = \std(y)\bigl(g(\tilde y) - \tilde x\bigr),
\end{equation}
so the raw-space MSE factorizes as:
\begin{equation}
    \| f(y)-x\|_2^2
    = \std{(y)}^2 \,\big\| g(\tilde y)-\tilde x \big\|_2^2
\end{equation}
Thus raw-space MSE is equivalent to a weighted regression in normalized coordinates, with weight \(\std{(y)}^2\) per instance.\@
Consequently, the raw-space PSNR decomposes as:
\begin{align}
    \mathrm{PSNR}
     & = \underbrace{10\log_{10}(dR^2)}_{\text{constant}}  \nonumber                                                                  \\
     & - \underbrace{10\log_{10}\bigl(\|g(\tilde y)-\tilde x\|_2^2\bigr)}_{\text{normalized error } Q(\tilde y, \tilde x) } \nonumber \\
     & - \underbrace{20\log_{10}\bigl(\std(y)\bigr)}_{\text{scale}},
\end{align}
where \(R\) denotes the pixel dynamic range\footnote{In our experiments images are scaled to \([0,1]\), so \(R=1\).}.
The scale term depends only on the instance statistic \(\std(y)\), while the normalized error
\begin{equation}
    Q(\tilde y,\tilde x) := -10\log_{10}\|g(\tilde y)-\tilde x\|_2^2
\end{equation}
captures regression quality in normalized coordinates (larger \(Q\) means smaller normalized squared error).
Since the scale term is independent of the denoiser, it acts as a common input-dependent offset; after accounting for this offset, all denoiser-dependent robustness and mismatch behavior is carried by \(Q\).

\subsection{The Geometry of Normalized Coordinates}%
\label{subsec:problem_difficulty}
In normalized coordinates, the input \(\tilde y = \Tne(y)\) always lies on \(\mathcal{M}\) (unit standard deviation, zero mean).
The target \(\tilde x = (x - \mu(y)\ones)/\std(y)\) does not: as noise increases, \(\std(y)\) grows, so \(\tilde x\) shrinks toward the origin.
Thus higher noise compresses the target while leaving the input on \(\mathcal{M}\).

We define the input-target distance in normalized coordinates:
\begin{equation}
    \delta := \tilde y - \tilde x,
    \qquad
    \Delta := \|\delta\|_2.
\end{equation}

Under AWGN (\(y = x + \sigma n\), \(n\sim\mathcal{N}(0,I_d)\)),
\begin{equation}
    \delta = \frac{y - x}{\std(y)} = \frac{\sigma}{\std(y)}\,n.
\end{equation}
As \(\sigma \to \infty\), \(\std(y) \approx \sigma\) and \(\Delta\) concentrates near \(\sqrt{d}\),
while as \(\sigma \to 0\), \(\Delta \to 0\) (Figure~\ref{fig:ne_delta_hist}).

\paragraph{SNR controls \(\Delta\).}
To understand how \(\Delta\) varies with noise level, we relate it to the signal-to-noise ratio.
Define the empirical signal variance \(\sigma_x^2 := \frac{1}{d}\|x - \mu(x)\ones\|_2^2\).
For high-dimensional instances, cross-terms concentrate and
\begin{equation}
    {\std(y)}^2 \approx \sigma_x^2 + \sigma^2.
\end{equation}
Substituting into \(\Delta^2 = \|\sigma n / \std(y)\|_2^2\) and taking expectations gives
\begin{equation}
    \E[\Delta^2 \mid x, \sigma] \approx \frac{d}{1 + \mathrm{SNR}}, \qquad \mathrm{SNR} := \sigma_x^2 / \sigma^2.
\end{equation}

\subsection{Normalized Error Is Largely \texorpdfstring{\(\Delta\)}{Delta}-Driven}
\begin{observation}[\(\Delta\) predicts normalized MSE]%
    \label{obs:delta_sufficiency}
    For our model and setup, the conditional normalized MSE satisfies
    \begin{equation}
        \E\bigl[\|g(\tilde y) - \tilde x\|_2^2 \mid \Delta = r,\, \sigma\bigr] \approx \psi_g(r),
    \end{equation}
    with residual dependence on \(\sigma\) that is small compared to variation across \(r\) (Figure~\ref{fig:ne_psnr_vs_delta}).
\end{observation}

The intuition is that, after conditioning on \(\Delta\), the distribution of normalized pairs \((\tilde y,\tilde x)\) changes little with \(\sigma\).
Thus the backbone faces a statistically similar normalized regression problem across noise levels.
We study why this holds in Appendix~\ref{app:obs1_explanation}.

To estimate \(\E\!\left[\|g(\tilde y) - \tilde x\|_2^2 \mid \Delta = r, \sigma\right]\) as a function of \(r\), we bin samples by \(\Delta\) and compute the mean and standard deviation of \(Q(\tilde y, \tilde x)\) within each bin, separately for each \(\sigma\).
Figure~\ref{fig:ne_psnr_vs_delta} shows that within the central 95\% of the training \(\Delta\) distribution, normalized error depends mainly on \(\Delta\), with little residual dependence on \(\sigma\).
An extended analysis comparing \methodBaseline{} and \methodWNE{} is given in Appendix~\ref{app:norm_error_vs_difficulty}. The \methodBaseline{} exhibits stronger dependence on \(\sigma\) at fixed \(\Delta\) (Figure~\ref{fig:app_q_vs_delta_grid}), suggesting that the collapse is a consequence of enforcing NE rather than an artifact of the difficulty measure itself.

\begin{figure}[t]
    \centering
    \includegraphics[width=\linewidth]{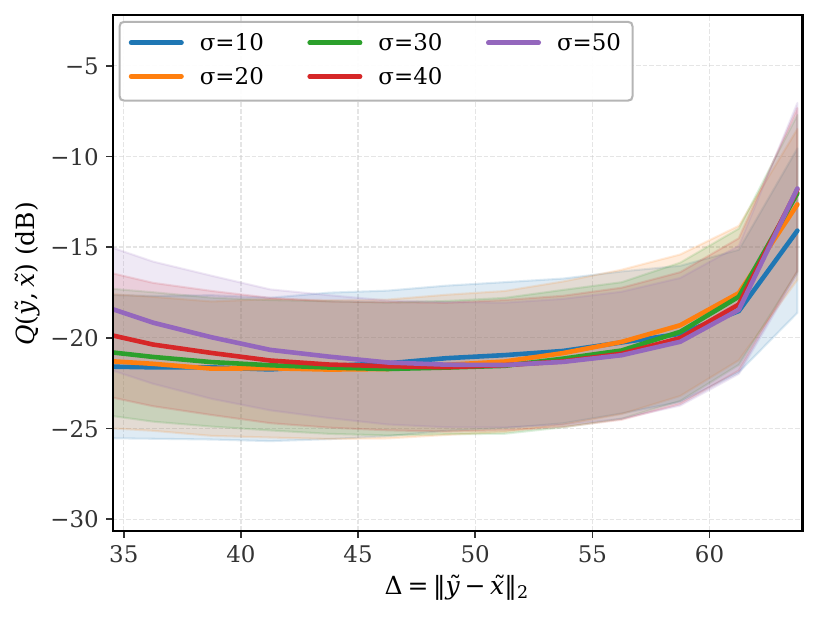}%
    \caption{\textbf{Normalized-space error vs difficulty for NE (central 95\% range).}
        For each \(\sigma\) we plot \(Q(\tilde y,\tilde x)\) versus \(\Delta=\|\tilde y-\tilde x\|_2\), reporting mean and standard deviation over samples.
        The \(x\)-axis shows the central \(95\%\) interval of the training \(\Delta\) distribution at \(\sigma_{\text{train}}=50\), used as a fixed reference range for all \(\sigma\).
        Within this shared difficulty band, curves for different \(\sigma\) show substantial collapse, suggesting that once difficulty is controlled by \(\Delta\), the remaining dependence of normalized regression error on \(\sigma\) is small.
        Model: SwinIR-\emph{\methodWNE{}}, trained at \(\sigma_{\text{train}}=50\).
    }%
    \label{fig:ne_psnr_vs_delta}
\end{figure}

\subsection{Train-Test Overlap in \texorpdfstring{\(\Delta\)}{Delta}}
\begin{observation}[Substantial overlap across noise levels]%
    \label{obs:overlap}
    The distribution of \(\Delta\) shows substantial overlap across noise levels (Figure~\ref{fig:ne_delta_hist}, Table~\ref{tab:ne_delta_coverage}).
\end{observation}

This overlap reflects the wide variation in \(\sigma_x^2\) across natural-image patches (flat regions vs.\ textured regions).
Changing \(\sigma\) therefore rescales a broad SNR distribution rather than shifting a point mass, preserving overlap in the induced \(\Delta\) values.

\begin{figure}[t]
    \centering
    \includegraphics[width=\linewidth]{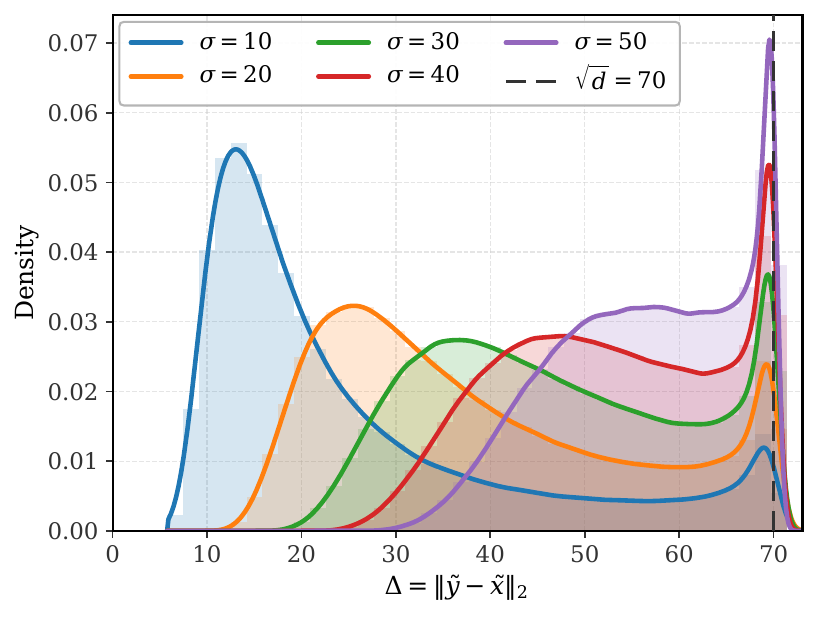}%
    \caption{\textbf{NE difficulty distributions.}
        Empirical histograms of \(\Delta=\|\tilde y-\tilde x\|_2\) at several noise levels \(\sigma\).
        As \(\sigma\) increases, mass shifts upward and concentrates near \(\sqrt d\), consistent with \(\Delta\approx\|n\|_2\) at high noise.
        Despite this shift, the distributions still overlap substantially across \(\sigma\).
        Computed from \(10^6\) training patches.
    }%
    \label{fig:ne_delta_hist}
\end{figure}

\begin{table}[t]
    \centering
    \caption{
        \textbf{Train-test \(\Delta\) overlap for the NE wrapper.}
        Entries are empirical coverages \(m(\sigma_{\mathrm{test}};\sigma_{\mathrm{train}})\) (percent),
        measured using the central \(95\%\) training interval at each \(\sigma_{\text{train}}\).
        Overlap is substantial but asymmetric: training at low noise covers many higher-noise difficulties, while training at high noise covers only a fraction of very low-noise difficulties.
        This matches the empirical asymmetry observed in Figure~\ref{fig:dncnn_psnr}.
    }%
    \label{tab:ne_delta_coverage}
    \small
    \definecolor{cov0}{RGB}{255,255,255}
    \definecolor{cov1}{RGB}{245,248,255}
    \definecolor{cov2}{RGB}{230,238,255}
    \definecolor{cov3}{RGB}{210,225,255}
    \definecolor{cov4}{RGB}{185,205,255}
    \definecolor{cov5}{RGB}{150,180,255}

    \newcommand{\covcell}[1]{%
        \ifnum#1<20  \cellcolor{cov0}#1%
        \else\ifnum#1<40 \cellcolor{cov1}#1%
            \else\ifnum#1<60 \cellcolor{cov2}#1%
                \else\ifnum#1<80 \cellcolor{cov3}#1%
                    \else\ifnum#1<95 \cellcolor{cov4}#1%
                        \else           \cellcolor{cov5}#1%
                        \fi\fi\fi\fi\fi
    }

    \setlength{\tabcolsep}{6pt}
    \renewcommand{\arraystretch}{1.15}
    \begin{tabular}{cccccc}
        \toprule
        \(\sigma_{\mathrm{train}}\backslash\sigma_{\mathrm{test}}\)
           & \(50\)       & \(40\)       & \(30\)       & \(20\)       & \(10\)
        \\
        \midrule
        50 & \covcell{95} & \covcell{87} & \covcell{68} & \covcell{42} & \covcell{19} \\
        40 & \covcell{97} & \covcell{95} & \covcell{83} & \covcell{56} & \covcell{24} \\
        30 & \covcell{96} & \covcell{96} & \covcell{95} & \covcell{76} & \covcell{34} \\
        20 & \covcell{93} & \covcell{95} & \covcell{96} & \covcell{95} & \covcell{56} \\
        10 & \covcell{88} & \covcell{90} & \covcell{92} & \covcell{95} & \covcell{95} \\
        \bottomrule
    \end{tabular}
\end{table}

To quantify the overlap of \(\Delta\) across noise levels, for a given \(\sigma_{\mathrm{train}}\) we define the central \(95\%\) interval of its \(\Delta\) distribution:
\begin{equation}
    I(\sigma_{\mathrm{train}})
    :=
    \bigl[q_{0.025}(\sigma_{\mathrm{train}}),\; q_{0.975}(\sigma_{\mathrm{train}})\bigr]
\end{equation}
where \(q_p(\sigma)\) denotes the empirical \(p\)-quantile of \(\Delta\) computed from samples
at noise level \(\sigma\).

For \(\sigma_{\mathrm{test}}\), we empirically estimate the in-range mass as follows:
\begin{equation}
    m(\sigma_{\mathrm{test}};\sigma_{\mathrm{train}})
    :=
    \frac{1}{N}
    \sum_{i=1}^{N}
    \mathbf{1}\!\left\{
    \Delta_{\sigma_{\mathrm{test}}}^i
    \in I(\sigma_{\mathrm{train}})
    \right\}
\end{equation}
where \(N\) is the number of samples.

Table~\ref{tab:ne_delta_coverage} shows that this in-range mass remains substantial off-diagonal, so mismatch test samples often query difficulties that are well-represented during training.

\subsection{Mechanism for Robustness and Asymmetry}
Normalized regression error depends mainly on \(\Delta\) rather than on \(\sigma\) (Observation~\ref{obs:delta_sufficiency}): an input whose difficulty \(\Delta\) was well represented in training is denoised with similar normalized error even when that \(\Delta\) now arises from an unseen noise level.
Observation~\ref{obs:overlap} (Figure~\ref{fig:ne_delta_hist}) shows that the \(\Delta\)-distributions can overlap substantially across \(\sigma\), so under mismatch \(g\) can still denoise when queried at familiar difficulties.
Figure~\ref{fig:dncnn_psnr} shows that generalization across noise levels is asymmetric: a denoiser generalizes to higher test noise more readily than to lower test noise.
The central \(\Delta\)-interval is wide at low \(\sigma_{\mathrm{train}}\) but narrower at high \(\sigma_{\mathrm{train}}\), so low-noise training covers most higher-\(\sigma_{\mathrm{test}}\) difficulties, whereas high-noise training covers progressively fewer lower-\(\sigma_{\mathrm{test}}\) ones.

\section{Conclusion}

We prove that a map is normalization equivariant if and only if it factors as normalize, process, denormalize.
The ``if'' direction means we can enforce NE by wrapping any backbone; the ``only if'' direction means we lose nothing by doing so, because every NE map is already of this form.
Enforcing exact NE therefore leaves the backbone architecture unrestricted.

This gives \methodWNE{}: wrap any backbone with the analytic normalize-denormalize pair, and the result is exactly NE.\@
No internal layers change.
Attention, LayerNorm, and BatchNorm all stay where they are.

In blind denoising, \methodWNE{} improves robustness to noise-level mismatch for DnCNN and SwinIR, remains competitive at matched noise, and adds no measured GPU overhead.
Restormer (Appendix~\ref{app:restormer}), color denoising (Appendix~\ref{app:restormer_color}), and SIDD real-noise results (Appendix~\ref{app:sidd}) follow the same pattern.

Our normalized-coordinate view explains why NE improves denoising robustness.
An NE denoiser is determined by its action on normalized inputs, so changing \(\sigma\) mainly shifts which normalized difficulties \(\Delta=\|\tilde y-\tilde x\|_2\) are queried, not the regression problem at fixed \(\Delta\) (Appendix~\ref{app:obs1_explanation}).
When train and test \(\Delta\)-distributions overlap, the wrapped backbone stays in an interpolation regime.
Table~\ref{tab:ne_delta_coverage} shows this overlap is substantial but asymmetric: training at low noise covers many high-noise difficulties, while training at high noise covers fewer low-noise ones.
The same asymmetry shows up in the PSNR curves of Figure~\ref{fig:dncnn_psnr}.

The same mechanism extends to iterative reuse, where a sampler queries the denoiser at many noise levels.
In Appendix~\ref{app:n2n_wne}, a Noise2Noise denoiser trained at \(\sigma_{\mathrm{train}}=10\) is reused unchanged inside a random-inpainting sampler initialized at full Gaussian noise, about \(25\times\) above the training scale.
With the same backbone, training pairs, and sampler hyperparameters, the wrapped denoiser reaches \(23.87\) dB while the unwrapped one stays at \(6.08\) dB.

These results point to natural directions for future work: preconditioned plug-and-play methods that rely on exact NE~\citep{hong2024provablepnp}, noisy-only diffusion learning~\citep{levac2025normalization}, and broader image restoration tasks.

\section*{Impact Statement}

This paper presents a theoretical characterization and a parameter-free wrapper for enforcing normalization equivariance in image-to-image neural networks.
The main intended benefit is improved robustness under distribution shift, reducing calibration burden in applications such as photography, medical imaging, and scientific imaging.
Potential risks mirror those of stronger image-processing systems in sensitive settings: denoising can alter visual evidence, and benchmark restoration performance alone does not validate real-world deployment.
The method releases no personal data or generative model, and we identify no additional misuse risks beyond these application-level concerns.

\section*{Acknowledgments}
We thank Marian Lukac, Eleni Kalogirou, and the members of the Machine Learning Group at the University of Geneva for helpful discussions.
This work was supported by the Hasler Foundation under the Responsible AI program.

\bibliography{paper}

\begin{thebibliography}{36}
\providecommand{\natexlab}[1]{#1}
\providecommand{\url}[1]{\texttt{#1}}
\expandafter\ifx\csname urlstyle\endcsname\relax
  \providecommand{\doi}[1]{doi: #1}\else
  \providecommand{\doi}{doi: \begingroup \urlstyle{rm}\Url}\fi

\bibitem[Abdelhamed et~al.(2018)Abdelhamed, Lin, and
  Brown]{abdelhamed2018highquality}
Abdelhamed, A., Lin, S., and Brown, M.~S.
\newblock A high-quality denoising dataset for smartphone cameras.
\newblock In \emph{Proceedings of the {IEEE/CVF} Conference on Computer Vision
  and Pattern Recognition}, pp.\  1692--1700, 2018.

\bibitem[Agustsson \& Timofte(2017)Agustsson and Timofte]{agustsson2017ntire}
Agustsson, E. and Timofte, R.
\newblock {NTIRE} 2017 challenge on single image super-resolution: Dataset and
  study.
\newblock In \emph{Proceedings of the {IEEE} Conference on Computer Vision and
  Pattern Recognition Workshops}, pp.\  126--135, 2017.

\bibitem[Buades et~al.(2005)Buades, Coll, and Morel]{buades2005review}
Buades, A., Coll, B., and Morel, J.-M.
\newblock A review of image denoising algorithms, with a new one.
\newblock \emph{Multiscale Modeling \& Simulation}, 4\penalty0 (2):\penalty0
  490--530, 2005.

\bibitem[Dabov et~al.(2007)Dabov, Foi, Katkovnik, and
  Egiazarian]{dabov2007image}
Dabov, K., Foi, A., Katkovnik, V., and Egiazarian, K.
\newblock Image denoising by sparse {3-D} transform-domain collaborative
  filtering.
\newblock \emph{{IEEE} Transactions on Image Processing}, 16\penalty0
  (8):\penalty0 2080--2095, 2007.

\bibitem[Desai et~al.(2023)Desai, Ozturkler, Sandino, Boutin, Willis,
  Vasanawala, Hargreaves, R{\'e}, Pauly, and Chaudhari]{desai2023noise2recon}
Desai, A.~D., Ozturkler, B.~M., Sandino, C.~M., Boutin, R., Willis, M.,
  Vasanawala, S., Hargreaves, B.~A., R{\'e}, C., Pauly, J.~M., and Chaudhari,
  A.~S.
\newblock {Noise2Recon}: Enabling {SNR}-robust {MRI} reconstruction with
  semi-supervised and self-supervised learning.
\newblock \emph{Magnetic Resonance in Medicine}, 90\penalty0 (5):\penalty0
  2052--2070, 2023.
\newblock \doi{10.1002/mrm.29759}.

\bibitem[Gu et~al.(2014)Gu, Zhang, Zuo, and Feng]{gu2014weighted}
Gu, S., Zhang, L., Zuo, W., and Feng, X.
\newblock Weighted nuclear norm minimization with application to image
  denoising.
\newblock In \emph{Proceedings of the {IEEE} Conference on Computer Vision and
  Pattern Recognition}, pp.\  2862--2869, 2014.

\bibitem[Guth et~al.(2025)Guth, Kadkhodaie, and Simoncelli]{guth2025learning}
Guth, F., Kadkhodaie, Z., and Simoncelli, E.
\newblock Learning normalized image densities via dual score matching.
\newblock In \emph{Advances in Neural Information Processing Systems},
  volume~38, 2025.

\bibitem[Hendriksen et~al.(2020)Hendriksen, Pelt, and
  Batenburg]{hendriksen2020noise2inverse}
Hendriksen, A.~A., Pelt, D.~M., and Batenburg, K.~J.
\newblock {Noise2Inverse}: Self-supervised deep convolutional denoising for
  tomography.
\newblock \emph{IEEE Transactions on Computational Imaging}, 6:\penalty0
  1320--1335, 2020.
\newblock \doi{10.1109/TCI.2020.3019647}.

\bibitem[Herbreteau \& Kervrann(2024)Herbreteau and
  Kervrann]{herbreteau2024survey}
Herbreteau, S. and Kervrann, C.
\newblock On normalization-equivariance properties of supervised and
  unsupervised denoising methods: a survey.
\newblock \emph{arXiv preprint arXiv:2402.15352}, 2024.

\bibitem[Herbreteau et~al.(2023)Herbreteau, Moebel, and
  Kervrann]{herbreteau2023normalization}
Herbreteau, S., Moebel, E., and Kervrann, C.
\newblock Normalization-equivariant neural networks with application to image
  denoising.
\newblock \emph{Advances in Neural Information Processing Systems},
  36:\penalty0 5706--5728, 2023.

\bibitem[Hong et~al.(2024)Hong, Xu, Hu, and Fessler]{hong2024provablepnp}
Hong, T., Xu, X., Hu, J., and Fessler, J.~A.
\newblock Provable preconditioned plug-and-play approach for compressed sensing
  {MRI} reconstruction.
\newblock \emph{IEEE Transactions on Computational Imaging}, 10:\penalty0
  1476--1488, 2024.

\bibitem[Huang et~al.(2024)Huang, Zhang, Zhang, Li, Dong, and
  Ying]{huang2024dured}
Huang, P., Zhang, C., Zhang, X., Li, X., Dong, L., and Ying, L.
\newblock Self-supervised deep unrolled reconstruction using regularization by
  denoising.
\newblock \emph{IEEE Transactions on Medical Imaging}, 43\penalty0
  (3):\penalty0 1203--1213, 2024.
\newblock \doi{10.1109/TMI.2023.3332614}.

\bibitem[Ioffe \& Szegedy(2015)Ioffe and Szegedy]{ioffe2015batch}
Ioffe, S. and Szegedy, C.
\newblock Batch normalization: Accelerating deep network training by reducing
  internal covariate shift.
\newblock In \emph{Proceedings of the 32nd International Conference on Machine
  Learning}, volume~37 of \emph{Proceedings of Machine Learning Research}, pp.\
   448--456. PMLR, 2015.

\bibitem[Kaba et~al.(2023)Kaba, Mondal, Zhang, Bengio, and
  Ravanbakhsh]{kaba2023equivariance}
Kaba, S.-O., Mondal, A.~K., Zhang, Y., Bengio, Y., and Ravanbakhsh, S.
\newblock Equivariance with learned canonicalization functions.
\newblock In \emph{Proceedings of the 40th International Conference on Machine
  Learning}, volume 202 of \emph{Proceedings of Machine Learning Research},
  pp.\  15546--15566. PMLR, 2023.
\newblock URL \url{https://proceedings.mlr.press/v202/kaba23a.html}.

\bibitem[Kadkhodaie \& Simoncelli(2021)Kadkhodaie and
  Simoncelli]{kadkhodaie2021stochastic}
Kadkhodaie, Z. and Simoncelli, E.
\newblock Stochastic solutions for linear inverse problems using the prior
  implicit in a denoiser.
\newblock \emph{Advances in Neural Information Processing Systems},
  34:\penalty0 13242--13254, 2021.

\bibitem[Kadkhodaie et~al.(2024)Kadkhodaie, Guth, Simoncelli, and
  Mallat]{kadkhodaie2024harmonic}
Kadkhodaie, Z., Guth, F., Simoncelli, E.~P., and Mallat, S.
\newblock Generalization in diffusion models arises from geometry-adaptive
  harmonic representations.
\newblock In \emph{International Conference on Learning Representations}, 2024.

\bibitem[Kim et~al.(2022)Kim, Kim, Tae, Park, Choi, and Choo]{kim2022revin}
Kim, T., Kim, J., Tae, Y., Park, C., Choi, J.-H., and Choo, J.
\newblock Reversible instance normalization for accurate time-series
  forecasting against distribution shift.
\newblock In \emph{International Conference on Learning Representations}, 2022.
\newblock URL \url{https://openreview.net/forum?id=cGDAkQo1C0p}.

\bibitem[Kingma \& Ba(2015)Kingma and Ba]{kingma2015adam}
Kingma, D.~P. and Ba, J.
\newblock Adam: A method for stochastic optimization.
\newblock In \emph{International Conference on Learning Representations}, 2015.

\bibitem[Lehtinen et~al.(2018)Lehtinen, Munkberg, Hasselgren, Laine, Karras,
  Aittala, and Aila]{lehtinen2018noise2noise}
Lehtinen, J., Munkberg, J., Hasselgren, J., Laine, S., Karras, T., Aittala, M.,
  and Aila, T.
\newblock {Noise2Noise}: Learning image restoration without clean data.
\newblock In \emph{Proceedings of the 35th International Conference on Machine
  Learning}, volume~80 of \emph{Proceedings of Machine Learning Research}, pp.\
   2965--2974. PMLR, 2018.

\bibitem[Levac et~al.(2025)Levac, Tamir, Pereyra, and
  Tachella]{levac2025normalization}
Levac, B., Tamir, J., Pereyra, M., and Tachella, J.
\newblock Normalization-equivariant diffusion models: Learning posterior
  samplers from noisy and partial measurements.
\newblock \emph{arXiv preprint arXiv:2510.11964}, 2025.

\bibitem[Liang et~al.(2021)Liang, Cao, Sun, Zhang, Van~Gool, and
  Timofte]{liang2021swinir}
Liang, J., Cao, J., Sun, G., Zhang, K., Van~Gool, L., and Timofte, R.
\newblock {SwinIR}: Image restoration using {Swin Transformer}.
\newblock In \emph{Proceedings of the {IEEE/CVF} International Conference on
  Computer Vision Workshops}, pp.\  1833--1844, 2021.

\bibitem[Lim et~al.(2017)Lim, Son, Kim, Nah, and Lee]{lim2017enhanced}
Lim, B., Son, S., Kim, H., Nah, S., and Lee, K.~M.
\newblock Enhanced deep residual networks for single image super-resolution.
\newblock In \emph{Proceedings of the {IEEE} Conference on Computer Vision and
  Pattern Recognition Workshops}, pp.\  136--144, 2017.

\bibitem[Liu et~al.(2020)Liu, Sun, Eldeniz, Gan, An, and Kamilov]{liu2020rare}
Liu, J., Sun, Y., Eldeniz, C., Gan, W., An, H., and Kamilov, U.~S.
\newblock {RARE}: Image reconstruction using deep priors learned without
  groundtruth.
\newblock \emph{IEEE Journal of Selected Topics in Signal Processing},
  14\penalty0 (6):\penalty0 1088--1099, 2020.
\newblock \doi{10.1109/JSTSP.2020.2998402}.

\bibitem[Liu et~al.(2022)Liu, Wu, Wang, and Long]{liu2022non}
Liu, Y., Wu, H., Wang, J., and Long, M.
\newblock Non-stationary transformers: Exploring the stationarity in time
  series forecasting.
\newblock \emph{Advances in Neural Information Processing Systems},
  35:\penalty0 9881--9893, 2022.

\bibitem[Liu et~al.(2021)Liu, Lin, Cao, Hu, Wei, Zhang, Lin, and
  Guo]{liu2021swin}
Liu, Z., Lin, Y., Cao, Y., Hu, H., Wei, Y., Zhang, Z., Lin, S., and Guo, B.
\newblock {Swin Transformer}: Hierarchical vision transformer using shifted
  windows.
\newblock In \emph{Proceedings of the {IEEE/CVF} International Conference on
  Computer Vision}, pp.\  10012--10022, 2021.

\bibitem[Ma et~al.(2017)Ma, Duanmu, Wu, Wang, Yong, Li, and
  Zhang]{ma2016waterloo}
Ma, K., Duanmu, Z., Wu, Q., Wang, Z., Yong, H., Li, H., and Zhang, L.
\newblock {Waterloo Exploration Database}: New challenges for image quality
  assessment models.
\newblock \emph{IEEE Transactions on Image Processing}, 26\penalty0
  (2):\penalty0 1004--1016, 2017.

\bibitem[Martin et~al.(2001)Martin, Fowlkes, Tal, and
  Malik]{martin2001database}
Martin, D., Fowlkes, C., Tal, D., and Malik, J.
\newblock A database of human segmented natural images and its application to
  evaluating segmentation algorithms and measuring ecological statistics.
\newblock In \emph{Proceedings of the Eighth {IEEE} International Conference on
  Computer Vision}, volume~2, pp.\  416--423. IEEE, 2001.

\bibitem[Mohan et~al.(2020)Mohan, Kadkhodaie, Simoncelli, and
  Fernandez-Granda]{mohan2020biasfree}
Mohan, S., Kadkhodaie, Z., Simoncelli, E.~P., and Fernandez-Granda, C.
\newblock Robust and interpretable blind image denoising via bias-free
  convolutional neural networks.
\newblock In \emph{International Conference on Learning Representations}, 2020.
\newblock URL \url{https://openreview.net/forum?id=HJlSmC4FPS}.

\bibitem[Mondal et~al.(2023)Mondal, Panigrahi, Kaba, Rajeswar~Mudumba, and
  Ravanbakhsh]{mondal2023equivariant}
Mondal, A.~K., Panigrahi, S.~S., Kaba, O., Rajeswar~Mudumba, S., and
  Ravanbakhsh, S.
\newblock Equivariant adaptation of large pretrained models.
\newblock In \emph{Advances in Neural Information Processing Systems},
  volume~36, pp.\  50293--50309. Curran Associates, Inc., 2023.

\bibitem[Paszke et~al.(2019)Paszke, Gross, Massa, Lerer, Bradbury, Chanan,
  Killeen, Lin, Gimelshein, Antiga, et~al.]{paszke2019pytorch}
Paszke, A., Gross, S., Massa, F., Lerer, A., Bradbury, J., Chanan, G., Killeen,
  T., Lin, Z., Gimelshein, N., Antiga, L., et~al.
\newblock {PyTorch}: An imperative style, high-performance deep learning
  library.
\newblock \emph{Advances in Neural Information Processing Systems},
  32:\penalty0 8026--8037, 2019.

\bibitem[Rudin et~al.(1992)Rudin, Osher, and Fatemi]{rudin1992tv}
Rudin, L.~I., Osher, S., and Fatemi, E.
\newblock Nonlinear total variation based noise removal algorithms.
\newblock \emph{Physica D: Nonlinear Phenomena}, 60\penalty0 (1--4):\penalty0
  259--268, 1992.

\bibitem[Ulyanov et~al.(2016)Ulyanov, Vedaldi, and
  Lempitsky]{ulyanov2016instance}
Ulyanov, D., Vedaldi, A., and Lempitsky, V.
\newblock Instance normalization: The missing ingredient for fast stylization.
\newblock \emph{arXiv preprint arXiv:1607.08022}, 2016.

\bibitem[Wang et~al.(2004)Wang, Bovik, Sheikh, and Simoncelli]{wang2004image}
Wang, Z., Bovik, A.~C., Sheikh, H.~R., and Simoncelli, E.~P.
\newblock Image quality assessment: from error visibility to structural
  similarity.
\newblock \emph{{IEEE} Transactions on Image Processing}, 13\penalty0
  (4):\penalty0 600--612, 2004.

\bibitem[Zamir et~al.(2022)Zamir, Arora, Khan, Hayat, Khan, and
  Yang]{zamir2022restormer}
Zamir, S.~W., Arora, A., Khan, S., Hayat, M., Khan, F.~S., and Yang, M.-H.
\newblock {Restormer}: Efficient transformer for high-resolution image
  restoration.
\newblock In \emph{Proceedings of the {IEEE/CVF} Conference on Computer Vision
  and Pattern Recognition}, pp.\  5728--5739, 2022.

\bibitem[Zhang et~al.(2017)Zhang, Zuo, Chen, Meng, and Zhang]{zhang2017dncnn}
Zhang, K., Zuo, W., Chen, Y., Meng, D., and Zhang, L.
\newblock Beyond a {Gaussian} denoiser: Residual learning of deep {CNN} for
  image denoising.
\newblock \emph{IEEE Transactions on Image Processing}, 26\penalty0
  (7):\penalty0 3142--3155, 2017.

\bibitem[Zhang et~al.(2022)Zhang, Li, Zuo, Zhang, Van~Gool, and
  Timofte]{zhang2021plug}
Zhang, K., Li, Y., Zuo, W., Zhang, L., Van~Gool, L., and Timofte, R.
\newblock Plug-and-play image restoration with deep denoiser prior.
\newblock \emph{IEEE Transactions on Pattern Analysis and Machine
  Intelligence}, 44\penalty0 (10):\penalty0 6360--6376, 2022.

\end{thebibliography}
\bibliographystyle{icml2026}
\appendix
\onecolumn
\raggedbottom{}
\setlength{\textfloatsep}{8pt plus 2pt minus 2pt}
\setlength{\floatsep}{8pt plus 2pt minus 2pt}
\setlength{\intextsep}{8pt plus 2pt minus 2pt}

\section{Description of the Denoising Architectures and Implementation}%
\label{app:arch_impl}

\subsection{Description of Models}%
\label{app:arch_models}

\paragraph{DnCNN and FDnCNN:}
We implement the Denoising CNN (DnCNN) of \citet{zhang2017dncnn}, which consists of 20 convolutional layers with \(3\times 3\) filters and 64 channels, batch normalization~\citep{ioffe2015batch}, and ReLU nonlinearities.
To construct the bias-free/BatchNorm-free variant used by prior equivariant baselines, we follow \citet{herbreteau2023normalization} and use FDnCNN, an unpublished flexible DnCNN variant that removes batch normalization and additive bias.
In our experiments, we keep the standard DnCNN backbone unchanged when applying our wrapper (\methodWNE{}), while FDnCNN serves as the base architecture for the architectural scale-equivariant and normalization-equivariant baselines (\methodSEarch{}, \methodNEarch{}) where internal-layer equivariance is enforced.

\paragraph{SwinIR:}
We use SwinIR~\citep{liang2021swinir}, a transformer-based denoiser built from Swin Transformer blocks~\citep{liu2021swin} (windowed self-attention with LayerNorm and MLPs).
We use the denoising configuration (no upsampling) and instantiate the lightweight variant in our code: a shallow \(3\times 3\) convolutional feature extractor, followed by \(4\) Residual Swin Transformer Blocks (RSTBs) with depths \((6,6,6,6)\) (total \(24\) Swin Transformer blocks, embedding dimension \(60\), window size \(8\), \(6\) heads, MLP ratio \(2\); patch size \(1\) with patch-wise LayerNorm), and a \(3\times 3\) convolutional reconstruction head.
As architectural SE/NE counterparts are not available as drop-in replacements for standard transformer components (e.g., attention and LayerNorm), for SwinIR we compare only the unmodified backbone (\methodBaseline{}) and its wrapped version (\methodWNE{}).

\paragraph{Restormer:}
We use Restormer~\citep{zamir2022restormer}, a transformer-based denoiser built from Multi-DConv Head Transposed Self-Attention (MDTA) and Gated-Dconv Feed-Forward Network (GDFN) blocks.
In our code, we instantiate the denoising variant as a 4-level encoder-decoder with an overlapping \(3\times 3\) convolutional input embedding, base width \(48\), block counts \((4,6,6,8)\) across the encoder/latent levels, attention heads \((1,2,4,8)\), FFN expansion factor \(2.66\), and \(4\) refinement blocks at the final resolution.
Downsampling and upsampling are implemented with convolution plus PixelUnshuffle/PixelShuffle, and the network predicts the denoised output through a final \(3\times 3\) convolution with a global residual connection from the noisy input.
For the grayscale experiments, the model uses single-channel input/output and is trained on \(64\times 64\) patches.
As architectural SE/NE counterparts are not available as drop-in replacements for the standard attention and normalization blocks, for Restormer we compare only the unmodified backbone (\methodBaseline{}) and its wrapped version (\methodWNE{}).

\subsection{Description of Variants}%
\label{app:arch_variants}

\paragraph{\methodBaseline:}
The \methodBaseline{} variant is the unmodified reference implementation of each backbone (DnCNN or SwinIR), including its standard use of biases and (when present) internal normalization layers.

\paragraph{\methodSEarch:}
For the DnCNN family, the scale-equivariant baseline \methodSEarch{} follows the bias-free ReLU construction of \citet{herbreteau2023normalization} which is based on the construction from \citet{mohan2020biasfree}.
Concretely, architectural sources of scale breaking (most notably additive biases, and normalization layers that introduce additive shifts) are removed, leaving only positively 1-homogeneous operations.
This enforces scale equivariance \(f(ay)=a f(y)\) for all \(a>0\).

\paragraph{\methodNEarch:}
The normalization-equivariant baseline \methodNEarch{} follows the architectural construction of \citet{herbreteau2023normalization}.
In this design, NE is enforced throughout the network by restricting internal operations to NE-compatible families:
\begin{enumerate}[(i)]
    \item Standard convolutions are replaced by affine-constrained convolutions: bias-free convolutions whose weights are constrained so that a constant feature map is mapped to itself,
          \[
              \mathrm{Conv}_w(\ones) = \ones,
          \]
          which for a multi-channel \(3\times 3\) convolution is equivalent to requiring, for each output channel \(c'\),
          \[
              \sum_{c}\sum_{i,j} w_{c',c,i,j} = 1,
          \]
          with reflect padding.
    \item ReLU nonlinearities are replaced by SortPool channel-wise sorting patterns, implemented via the two-channel operator
          \[
              \mathrm{SortPool}(u,v) \;=\; \bigl(\min(u,v),\,\max(u,v)\bigr),
          \]
          applied in a fixed pairing scheme across channels.
    \item Classical residual sums are replaced by affine residual connections, where two branches \(\ell_1\) and \(\ell_2\) are combined as
          \[
              (1-t)\,\ell_1 + t\,\ell_2,
          \]
          with a trainable scalar \(t\).
\end{enumerate}

As in \citet{herbreteau2023normalization}, \methodNEarch{} is instantiated on FDnCNN (BatchNorm-free), since enforcing NE throughout internal layers is not compatible with BatchNorm.

\paragraph{Wrapped normalization-equivariant (\methodWNE{}):}
Our method \methodWNE{} enforces input-output NE without modifying internal operations.
Given an arbitrary backbone \(g_\theta\), the wrapped predictor is
\[
    \fwne(y)=\std(y)\,g_\theta(\Tne(y))+\mu(y)\ones,
\]
as defined in~\eqref{eq:ne_wrapper}.
This yields NE at the input-output level for any backbone \(g_\theta\), including off-the-shelf DnCNN (with BatchNorm) and SwinIR.\@

\subsection{Practical Implementation Notes}%
\label{app:arch_impl_notes}

Our primary comparisons are within-backbone: \methodBaseline{} versus \methodWNE{} for DnCNN and SwinIR, isolating the effect of the outer wrapper on a fixed backbone.
For the DnCNN family we also report \methodSEarch{} and \methodNEarch{} from \citet{herbreteau2023normalization} as architectural references.
These models enforce SE/NE by replacing core operations throughout the network (e.g., constrained convolutions and equivariant nonlinearities), so they constitute a distinct architecture rather than a drop-in modification of standard DnCNN\@.
We include them to place \methodWNE{} in context relative to established architectural-constraint designs.

\paragraph{Channel-wise sort pooling (\methodNEarch).}
We implement the channel-wise sorting nonlinearity of \citet{herbreteau2023normalization} using standard pooling primitives.
Given a pairing of channels, we compute the per-location maxima and minima for each pair (pooling with kernel size \(2\) along the channel axis) and concatenate the results.
Between successive sort stages, we apply the same fixed channel reindexing as in \citet{herbreteau2023normalization} to define the next pairing pattern.
Since CNN feature channels have no intrinsic ordering, the precise permutation is not semantically meaningful; it simply realizes the prescribed NE-compatible nonlinearity efficiently.

\paragraph{Affine-constrained convolutions (\methodNEarch).}
For the architectural NE baseline (\methodNEarch{}), we use the public implementation released by \citet{herbreteau2023normalization}, including their stable reparameterization of affine-constrained convolution kernels.
(We do not re-derive or redesign these constrained layers; we reuse their tested implementation to keep the architectural NE baseline faithful.)

\paragraph{Numerical stability (\methodWNE).}\label{app:numerical_stability}
For the wrapper, we compute \(\mu(y)\) and \(\std(y)\) jointly over all entries (all channels and pixels), matching the NE group action \(y\mapsto ay+b\ones\).
All theoretical results (and Proposition~\ref{prop:wrapper_is_ne}) assume the idealized wrapper using \(\std(y)\) together with the constant-instance guardrail \(\Tne(y)=0\) when \(\std(y)=0\).
In code, we use \(\std_{\varepsilon}(y):=\std(y)+\varepsilon\) with \(\varepsilon=10^{-5}\) to stabilize normalization when \(\std(y)\) is very small.
This yields an approximately NE implementation: it matches the theoretical wrapper whenever \(\std(y)\) is not close to \(0\), and differs only in the near-constant regime where the stabilization is active.

\paragraph{Code.}
All models are implemented in PyTorch~\citep{paszke2019pytorch}.
For \methodNEarch{}, we rely on the released code of \citet{herbreteau2023normalization}.
Our wrapper (\methodWNE{}) is implemented as a lightweight normalize-apply-backbone-denormalize module around the chosen backbone.

\section{Description of Datasets and Training Details}%
\label{app:training_details}

\paragraph{Relation to prior work (protocol).}
We follow the dataset construction and training protocol of
\citet[][Appendix~B, and implementation notes in Appendix~A]{herbreteau2023normalization},
which in turn adopts the large training set of \citet{zhang2021plug}; we use this same training set for all models and experiments.
To diagnose noise-level mismatch in a controlled setting, we also consider single-noise training with fixed \(\sigma_{\mathrm{train}}\), following the single-noise protocol of \citet{herbreteau2023normalization}.

\paragraph{Datasets.}
Following \citet{herbreteau2023normalization}, our training set is the standard large-scale denoising training corpus popularized by DRUNet/DPIR-style pipelines.
The dataset is composed of 8,694 images:
\begin{enumerate}[(i)]
    \item BSD400 (400 images) from the Berkeley Segmentation Dataset~\citep{martin2001database},
    \item the Waterloo Exploration Database (4,744 images)~\citep{ma2016waterloo},
    \item DIV2K (900 images)~\citep{agustsson2017ntire},
    \item Flickr2K (about 2.6k images)~\citep{lim2017enhanced}.
\end{enumerate}
We apply random flips and \(90^\circ\) rotations for data augmentation.
We report test results on Set12 and BSD68~\citep{martin2001database}.

\paragraph{Preprocessing and noise model.}
All experiments use additive white Gaussian noise (AWGN).
Images are processed in grayscale for Set12/BSD68 evaluation.
Pixel intensities are normalized to \([0,1]\); noise levels \(\sigma\) are reported in 8-bit units, so the injected noise standard deviation in \([0,1]\) units is \(\sigma/255\).
For the single-noise mismatch diagnostic, we train at a fixed
\(\sigma_{\mathrm{train}}\in\{10,25,50\}\) and evaluate at a range of \(\sigma_{\mathrm{test}}\) values.

\paragraph{Patch sampling.}
We train from on-the-fly random crops: each iteration samples a clean training image uniformly, extracts a random \(P\times P\) patch, and corrupts it with AWGN at the chosen \(\sigma\).
Unless stated otherwise, we use \(P=70\) (matching \citet{herbreteau2023normalization}'s FDnCNN/DnCNN setting).

\paragraph{Optimization.}
We train with Adam and default \(\beta\) parameters~\citep{kingma2015adam}.
Unless stated otherwise, we use an initial learning rate \(10^{-4}\).
We train for a fixed number of iterations (not tied to epochs); in our code this corresponds to \texttt{num\_steps} iterations. For the \methodNEarch{} variants we train for longer following~\citet{herbreteau2023normalization}.
We use MSE for our NE-enforcing variants (\methodNEarch{} and \methodWNE{}) following~\citet{herbreteau2023normalization}.

\paragraph{Model-specific settings.}
Table~\ref{tab:train_hparams} summarizes the training hyperparameters used in our experiments.
For DnCNN-family architectural baselines (SE-arch and NE-arch), we follow the corresponding prescriptions from
\citet{mohan2020biasfree,herbreteau2023normalization}.
For SwinIR, we use the standard denoising SwinIR backbone~\citep{liang2021swinir} and train it under the same AWGN patch-based pipeline to keep the comparison controlled.

\paragraph{Validation.} We report results using the final checkpoint rather than early-stopping on validation PSNR, as our focus is mismatch behavior rather than peak matched-noise performance.

\begin{table}[t]
    \centering
    \caption{\textbf{Training hyperparameters.}
        All runs use on-the-fly patch sampling from the same training corpus and AWGN corruption.
        For the single-noise diagnostic, \(\sigma\) is fixed during training (\(\sigma_{\mathrm{train}}\in\{10,25,50\}\)) and varied at test time.
        \textsuperscript{*} indicates that the learning rate is halved every 100{,}000 iterations.
        \citet{herbreteau2023normalization} use a constant learning rate for FDnCNN variants, noting that ``speed improvements are certainly possible by adapting the learning rate throughout optimization.''}%
    \label{tab:train_hparams}
    \small
    \setlength{\tabcolsep}{6pt}
    \renewcommand{\arraystretch}{1.05}
    \begin{tabular}{lccccc}
        \toprule
        Model / variant   & Batch & Patch \(P\) & Loss       & LR                                     & Iterations \\
        \midrule
        DnCNN (Baseline)  & 128   & 70          & \(\ell_1\) & \(1\mathrm{e}{-4}\)                    & 500{,}000  \\
        FDnCNN (SE-arch)  & 128   & 70          & \(\ell_1\) & \(1\mathrm{e}{-4}\)                    & 500{,}000  \\
        FDnCNN (NE-arch)  & 128   & 70          & MSE        & \(1\mathrm{e}{-4}\)                    & 900{,}000  \\
        DnCNN (WNE)       & 128   & 70          & MSE        & \(1\mathrm{e}{-4}\)\textsuperscript{*} & 500{,}000  \\
        \midrule
        SwinIR (Baseline) & 32    & 64          & MSE        & \(1\mathrm{e}{-4}\)                    & 500{,}000  \\
        SwinIR (WNE)      & 32    & 64          & MSE        & \(1\mathrm{e}{-4}\)\textsuperscript{*} & 500{,}000  \\
        \bottomrule
    \end{tabular}
\end{table}

\begin{table}[t]
    \centering
    \caption{\textbf{Runtime overhead for SwinIR (seconds per batch).}
        Timing with batch size \(32\), spatial size \(64\times 64\), warmup \(=3\), and timed steps \(=10\).
        “Backward” measures one training iteration (forward + loss + backward), and “Inference” measures forward-only.
        Values in parentheses are ratios relative to the baseline on the same device.}%
    \label{tab:speed_swinir}
    \setlength{\tabcolsep}{6pt}
    \begin{tabular}{lcccc}
        \toprule
                          & \multicolumn{2}{c}{GPU} & \multicolumn{2}{c}{CPU}                                                  \\
        Variant           & Backward (s)            & Inference (s)           & Backward (s)           & Inference (s)         \\
        \midrule
        \methodBaseline{} & 0.271                   & 0.101                   & 16.31                  & 3.52                  \\
        \methodWNE{}      & 0.271 (\(1.00\times\))  & 0.101 (\(1.00\times\))  & 18.22 (\(1.12\times\)) & 3.81 (\(1.08\times\)) \\
        \bottomrule
    \end{tabular}
\end{table}
\begin{table}[!t]
    \centering
    \caption{\textbf{Runtime overhead for FDnCNN (seconds per batch).}
        Timing on inputs of shape \(16 \times 1 \times 128 \times 128\) with warmup \(=3\) and timed steps \(=10\).
        “Backward” measures one training iteration (forward + loss + backward), and “Inference” measures forward-only.
        Values in parentheses are ratios relative to the baseline on the same device.
        Hardware: GPU = \texttt{NVIDIA GeForce RTX 4090}, CPU = \texttt{AMD Ryzen Threadripper PRO 5955WX 16-Cores}.}%
    \label{tab:speed}
    \setlength{\tabcolsep}{6pt}
    \begin{tabular}{lcccc}
        \toprule
                          & \multicolumn{2}{c}{GPU} & \multicolumn{2}{c}{CPU}                                                 \\
        Variant           & Backward (s)            & Inference (s)           & Backward (s)          & Inference (s)         \\
        \midrule
        \methodBaseline{} & 0.034                   & 0.011                   & 2.53                  & 0.82                  \\
        \methodWNE{}      & 0.033 (0.99\(\times\))  & 0.011 (1.00\(\times\))  & 2.51 (0.99\(\times\)) & 0.83 (1.01\(\times\)) \\
        \methodNEarch{}   & 0.054 (1.60\(\times\))  & 0.019 (1.69\(\times\))  & 3.79 (1.50\(\times\)) & 1.57 (1.90\(\times\)) \\
        \bottomrule
    \end{tabular}
\end{table}
\subsection{Runtime Benchmarking Protocol}%
\label{app:speed_protocol}

We report wall-clock seconds per batch, averaged over 10 timed iterations after 3 warmup iterations, using the same input shapes across variants.
“Backward” corresponds to one training iteration (forward + loss + backward) and “Inference” corresponds to forward-only under the \texttt{torch.no\_grad()} context.
Absolute times depend on hardware and system settings, but ratios are informative because all variants share the same backbone family and are timed using the same harness.
For SwinIR, we additionally report timings on inputs of shape \(32\times 1\times 64\times 64\), reflecting the transformer’s training batch and patch size in our setup (Table~\ref{tab:speed_swinir}). On CPU, we observe a small but nonzero overhead in our current PyTorch implementation; this may be mitigated with more careful CPU benchmarking. Our focus is GPU training and inference, where modern denoisers are typically deployed.

\section{Mathematical Proofs}%
\label{app:proofs}

\subsection{A Lemma from \citet{herbreteau2023normalization}}%
\label{app:herb_lemma}

We reproduce the following lemma from \citet{herbreteau2023normalization} for completeness; it is not used elsewhere and is equivalent to our characterization via Remark~\ref{rem:M_vs_sphere}.

\begin{lemma}[Lemma~1 in \citet{herbreteau2023normalization}]%
    \label{lem:herb1}
    Let \(f:\R^d\to\R^d\) satisfy normalization equivariance:
    \[
        f(ay+b\ones)=a f(y)+b\ones
        \qquad \text{for all } a>0,\ b\in\R.
    \]
    Then \(f\) is entirely determined by its restriction to
    \[
        S^{d-1}\cap \mathrm{Span}{(\ones)}^\perp.
    \]
\end{lemma}

\begin{proof}
    First note that NE implies scale equivariance (SE) by taking \(b=0\):
    \[
        f(ay)=a f(y)\qquad (a>0).
    \]

    Fix \(x\in\R^d\). Decompose \(x\) orthogonally as
    \[
        x = x_1 + x_2,
        \qquad
        x_1\in \mathrm{Span}{(\ones)}^\perp,
        \quad
        x_2\in \mathrm{Span}(\ones).
    \]
    If \(x_1=0\), then \(x=x_2=\alpha\ones\) for some \(\alpha\in\R\). Using NE with \(y=0\) gives
    \[
        f(\alpha\ones)=f(0+\alpha\ones)=f(0)+\alpha\ones.
    \]
    Also, SE forces \(f(0)=0\): \(f(0)=f(a\cdot 0)=af(0)\) for all \(a>0\), so \(f(0)=0\).
    Hence \(f(\alpha\ones)=\alpha\ones\). Thus \(f(x)\) is determined on \(\mathrm{Span}(\ones)\).

    Assume \(x_1\neq 0\). Using NE with the shift \(x_2\in \mathrm{Span}(\ones)\),
    \[
        f(x)=f(x_1+x_2)=f(x_1)+x_2.
    \]
    Now let \(z:=x_1/\|x_1\|_2\). Then \(z\in S^{d-1}\cap \mathrm{Span}{(\ones)}^\perp\), and by SE,
    \[
        f(x_1)=f(\|x_1\|_2\,z)=\|x_1\|_2\,f(z).
    \]
    Therefore
    \[
        f(x)=\|x_1\|_2\,f(z)+x_2,
    \]
    so knowing \(f\) on \(S^{d-1}\cap \mathrm{Span}{(\ones)}^\perp\) determines \(f(x)\) for every \(x\in\R^d\).
\end{proof}

\begin{remark}[Connection to our normalized manifold \(\mathcal{M}\)]%
    \label{rem:M_vs_sphere}
    Recall
    \[
        \mathcal{M}:=\{z\in\R^d:\mu(z)=0,\ \std(z)=1\}
        =\{z\in \mathrm{Span}{(\ones)}^\perp:\|z\|_2=\sqrt d\}.
    \]
    Hence
    \[
        \mathcal{M}
        = \sqrt d\bigl(S^{d-1}\cap \mathrm{Span}{(\ones)}^\perp\bigr).
    \]
    Specifying a function on \(\mathcal{M}\) is equivalent to specifying it on
    \(S^{d-1}\cap \mathrm{Span}{(\ones)}^\perp\), up to the fixed rescaling by \(\sqrt d\).
\end{remark}

\subsection{Affine Identities for \(\mu\), \(\std\), and \(\Tne\)}%
\label{app:affine_identities}
We include the following proposition for completeness.
\begin{proposition}[Mean and standard deviation under affine transforms]%
    \label{prop:mu_sigma_affine}
    For all \(a>0\), \(b\in\R\), and \(y\in\R^d\),
    \[
        \mu(ay+b\ones)=a\,\mu(y)+b,
        \qquad
        \std(ay+b\ones)=a\,\std(y),
        \qquad
        \Tne(ay+b\ones)=\Tne(y),
    \]
    where
    \[
        \Tne(y)
        :=
        \begin{cases}
            \dfrac{y-\mu(y)\ones}{\std(y)} & \text{if } \std(y)>0, \\[6pt]
            0                              & \text{if } \std(y)=0.
        \end{cases}
    \]
\end{proposition}

\begin{proof}
    The mean identity follows from linearity of summation.

    For the \(\std\), note
    \[
        ay+b\ones-\mu(ay+b\ones)\ones
        = a\bigl(y-\mu(y)\ones\bigr),
    \]
    so \(\std(ay+b\ones)=(1/\sqrt d)\|a(y-\mu(y)\ones)\|_2=a\,\std(y)\).

    For \(\Tne\):
    If \(\std(y)>0\),
    \[
        \Tne(ay+b\ones)
        = \frac{a(y-\mu(y)\ones)}{a\,\std(y)}
        = \Tne(y).
    \]
    If \(\std(y)=0\), then \(y\) is constant, hence \(ay+b\ones\) is constant, and both map to \(0\) by definition.
\end{proof}

\subsection{Normalization Equivariance of the Wrapper}%
\label{app:wrapper_ne}

\begin{proposition}[NE of the wrapper]%
    \label{prop:wrapper_is_ne}
    Let \(g:\R^d\to\R^d\) be an arbitrary map and define
    \[
        \fwnea(y)
        := \std(y)\,g\bigl(\Tne(y)\bigr)+\mu(y)\ones.
    \]
    Then \(\fwnea\) is normalization equivariant:
    \[
        \fwnea(ay+b\ones)=a \,\fwnea(y)+b\ones
        \qquad \text{for all } a>0,\ b\in\R.
    \]
\end{proposition}

\begin{proof}
    Using Proposition~\ref{prop:mu_sigma_affine},
    \begin{align*}
        \fwnea(ay+b\ones)
         & = \std(ay+b\ones)\,g\bigl(\Tne(ay+b\ones)\bigr)
        + \mu(ay+b\ones)\ones
        \\
         & = a \, \std(y)\,g\bigl(\Tne(y)\bigr)
        + (a\mu(y)+b)\ones
        \\
         & = a \,\fwnea(y)+b\ones.
    \end{align*}
\end{proof}

\subsection{Complete Characterization of Normalization Equivariance}%
\label{app:characterization}

\nechar*
\begin{proof}
    (\(\Leftarrow\)) Follows directly from Proposition~\ref{prop:wrapper_is_ne}.

    (\(\Rightarrow\)) Assume \(f\) is NE.\@ Define \(g:\mathcal{M}\to\R^d\) as the restriction of \(f\) to \(\mathcal{M}\), i.e., \(g(z):=f(z)\) for \(z\in\mathcal{M}\).

    If \(\std(y)=0\), then \(y=\mu(y)\ones\). Scale equivariance (the case \(b=0\) of NE) implies \(f(0)=0\). Applying NE with \(a=1\) to the input \(0\) gives
    \[
        f(y)=f(\mu(y)\ones)=f(0)+\mu(y)\ones=\mu(y)\ones.
    \]

    If \(\std(y)>0\), then by definition
    \[
        y = \std(y)\,\Tne(y)+\mu(y)\ones,
        \qquad
        \Tne(y)\in\mathcal{M}.
    \]
    Applying NE gives:
    \[
        f(y)
        = f\bigl(\std(y)\,\Tne(y)+\mu(y)\ones\bigr)
        = \std(y)\,f\bigl(\Tne(y)\bigr)+\mu(y)\ones
        = \std(y)\,g\bigl(\Tne(y)\bigr)+\mu(y)\ones.
    \]

    Uniqueness: if \(g_1,g_2\) both satisfy the factorization, evaluate at \(z\in\mathcal{M}\) (where \(\mu(z)=0\), \(\std(z)=1\)) to obtain
    \[
        f(z)=g_i(z),
    \]
    so \(g_1(z)=g_2(z)\) for all \(z\in\mathcal{M}\).
\end{proof}

\subsection{Training Objective Identity}%
\label{app:loss_identity}

\begin{proposition}[Raw space MSE equals a weighted normalized MSE]%
    \label{prop:loss_identity}
    Let \(\tilde y := \Tne(y)\) and define the matched target transform
    \[
        \Tne(x\,;\,y) :=
        \begin{cases}
            \dfrac{x-\mu(y)\ones}{\std(y)} & \text{if } \std(y)>0, \\[6pt]
            0                              & \text{if } \std(y)=0.
        \end{cases}
    \]
    Let \(\tilde x := \Tne(x;\,y)\). For the wrapped model
    \(\fwne(y)=\std(y)\,g_\theta(\tilde y)+\mu(y)\ones\), when \(\std(y)>0\),
    \[
        \bigl\| \fwne(y)-x \bigr\|_2^2
        = \std{(y)}^2\,\bigl\| g_\theta(\tilde y)-\tilde x \bigr\|_2^2.
    \]
\end{proposition}
\begin{proof}
    Given \(\std(y)>0\), we write \(x=\std(y)\tilde x+\mu(y)\ones\). Then
    \begin{align*}
        \fwne(y)-x
         & = \bigl(\std(y)g_\theta(\tilde y)+\mu(y)\ones\bigr)
        - \bigl(\std(y)\tilde x+\mu(y)\ones\bigr)
        \\
         & = \std(y)\bigl(g_\theta(\tilde y)-\tilde x\bigr).
    \end{align*}
    Taking squared norms gives the claim.
\end{proof}

\section{Objective-Level Soft NE Versus Exact Parameterization}%
\label{app:soft_ne}

Because \methodWNE{} is loss-agnostic, it can in principle be paired with either supervised or self-supervised objectives.
\citet{levac2025normalization} consider a different route: in their noisy-only diffusion-learning setting, NE is encouraged through the training objective rather than imposed by parameterization.
In Gaussian denoising, they replace ordinary SURE with an affine-augmented objective.
For single-operator inverse problems, they add this denoising term to an equivariant-imaging loss that enforces consistency with the measurement model.
The comparison to \methodWNE{} is therefore between objective-level soft enforcement and exact model-level parameterization.

Porting that objective-level route to a different regime generally requires specifying a new loss, not merely reusing the objective from~\citet{levac2025normalization} verbatim.
For example, a supervised analogue is an orbit-averaged regression objective of the form
\[
    \mathcal{L}_{\mathrm{orbit}}
    :=
    \E_{(x,y),\alpha,\mu}
    \bigl\|
    f_\theta(\alpha y + \mu\ones)
    -
    (\alpha x + \mu\ones)
    \bigr\|_2^2,
\]
which introduces design choices through the sampling law of the affine variables \((\alpha,\mu)\).
Different orbit distributions weight different regions of the affine orbit and therefore define different soft-NE objectives.
By contrast, \methodWNE{} leaves the base loss unchanged and enforces the symmetry exactly by construction around the chosen backbone.

In noisy-only settings, exact parameterization can replace affine augmentation for enforcing NE, but losses such as SURE or equivariant-imaging terms are still needed to provide the self-supervised training signal.

\paragraph{Our supervised soft-NE implementation.}
For the supervised control used here, we instantiate the orbit-averaged objective by sampling a single affine pair \((\alpha,\mu)\) independently for each training image and applying the same transform to both the clean target and its noisy observation:
\[
    x'=\alpha x+\mu\ones,
    \qquad
    y'=\alpha y+\mu\ones.
\]
We then minimize the ordinary supervised regression loss
\[
    \mathcal{L}_{\mathrm{softNE}}
    =
    \E_{(x,y),\alpha,\mu}\!\left[
        \|f_\theta(y')-x'\|_2^2
        \right].
\]
In code, this is implemented by redrawing \((\alpha,\mu)\) on each training step, broadcasting each sampled scalar pair over all pixels and channels of the image, and replacing the original pair \((x,y)\) by \((x',y')\) before evaluating the standard MSE loss.
Following \citet{levac2025normalization}, we sample \(\alpha \sim \mathrm{U}(0,1)\) and \(\mu \sim \mathrm{U}(0,1)\); unlike a commutator penalty, this introduces no extra loss weight.

\paragraph{Control experiments.}
Figure~\ref{fig:app_soft_ne_controls} reports companion PSNR/SSIM curves for two SwinIR controls.
The left column compares a multi-noise SwinIR baseline trained on \(\sigma\in[0,55]\), matching the blind DnCNN-B range used by \citet{zhang2017dncnn}, against \methodWNE{} across the same test sweep.
The right column compares the supervised soft-NE SwinIR baseline at \(\sigma_{\mathrm{train}}=10\) against \methodWNE{} under the fixed-\(\sigma_{\mathrm{train}}\) mismatch diagnostic.
Both controls improve over an unconstrained baseline trained at a single noise level, but neither matches the wrapped model's flat mismatch behavior as cleanly as exact NE.\@
Broader noise coverage or affine augmentation can improve robustness within the sampled range, but they do not imply the algebraic identity \(f(ay+b\ones)=af(y)+b\ones\) for all \(a>0\) and \(b\in\R\).
The value of \methodWNE{} is precisely this exact structural guarantee: the affine nuisance is removed analytically rather than approximated from finite orbit coverage.
This matters in downstream settings that assume exact NE as part of the algorithmic design, such as the preconditioned PnP analysis of \citet{hong2024provablepnp}.

\begin{figure}[!htbp]
    \centering
    \subfigure[\textbf{Multi-noise Baseline vs.\ \methodWNE{} (PSNR).}]{
        \includegraphics[width=0.48\textwidth]{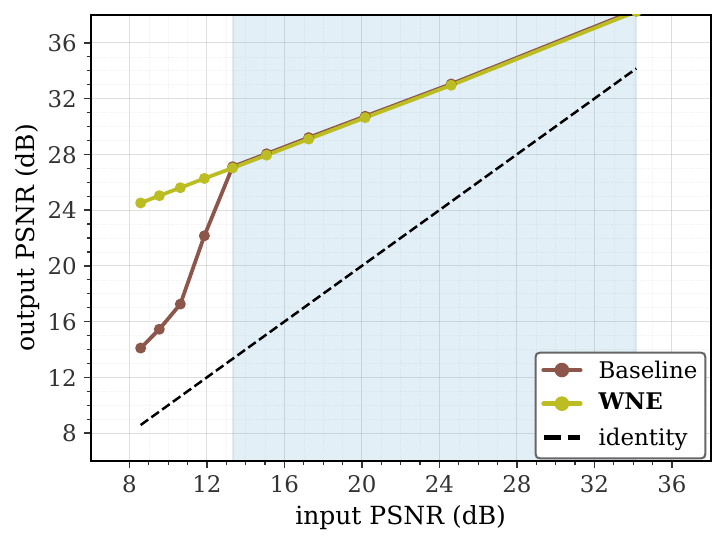}
    }\hfill
    \subfigure[\textbf{Soft-NE vs.\ \methodWNE{} (PSNR).}]{
        \includegraphics[width=0.48\textwidth]{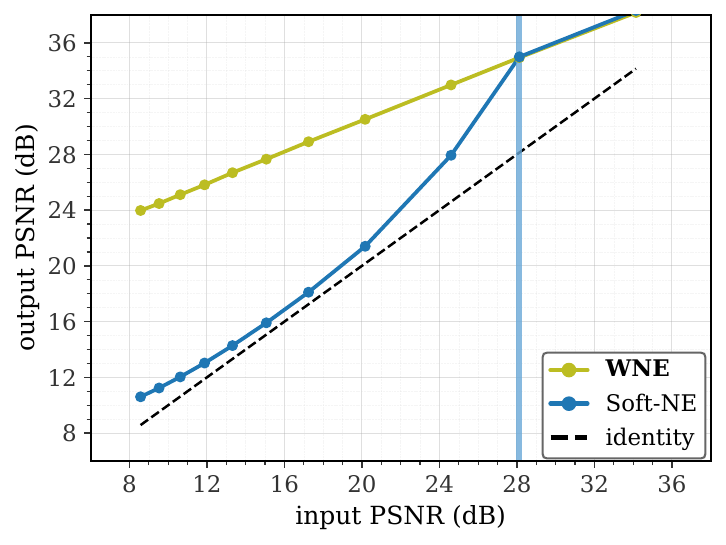}
    }\\[2pt]
    \subfigure[\textbf{Multi-noise Baseline vs.\ \methodWNE{} (SSIM).}]{
        \includegraphics[width=0.48\textwidth]{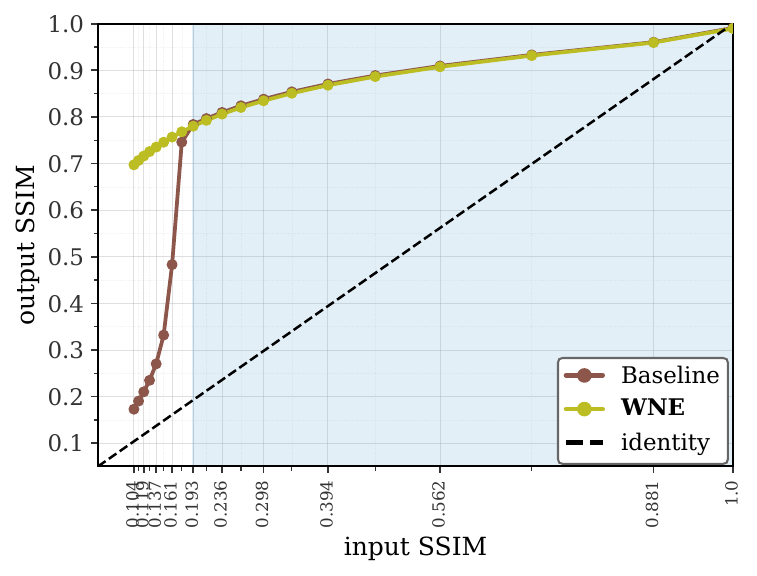}
    }\hfill
    \subfigure[\textbf{Soft-NE vs.\ \methodWNE{} (SSIM).}]{
        \includegraphics[width=0.48\textwidth]{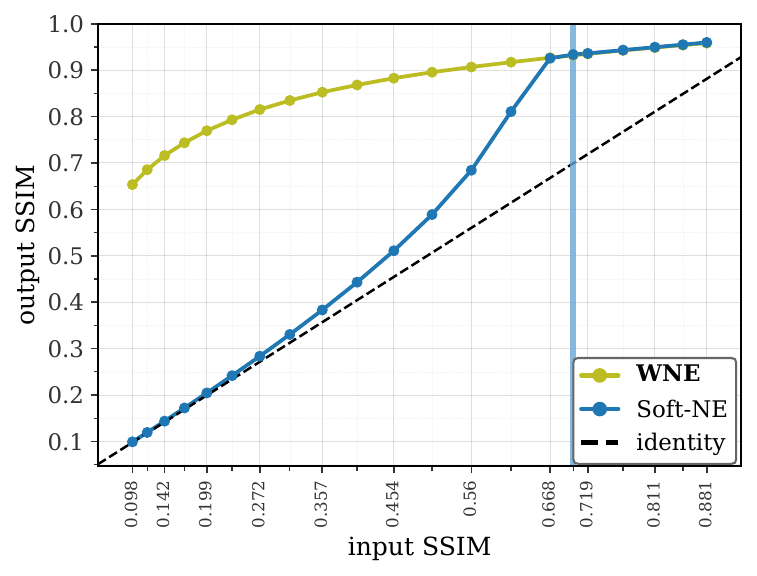}
    }
    \caption{\textbf{PSNR/SSIM controls for multi-noise and supervised soft-NE baselines on SwinIR.}
        Left column: a SwinIR baseline trained on a range of noise levels \(\sigma\in[0,55]\) against \methodWNE{} over the same test sweep.
        Right column: a supervised orbit-averaged soft-NE SwinIR baseline trained at \(\sigma_{\mathrm{train}}=10\) against \methodWNE{}.
        Top row uses output PSNR\@; bottom row uses output SSIM versus input SSIM\@.
        In both controls, broader orbit coverage narrows the performance gap, but \methodWNE{} remains the most stable under mismatch.
        Figure~\ref{fig:app_ne_violation} shows the corresponding explicit equivariance-defect curves.}%
    \label{fig:app_soft_ne_controls}
\end{figure}

\section{Equivariance-Defect Diagnostics for Multi-Noise and Soft-NE Baselines}%
\label{app:ne_violation}

To complement the PSNR/SSIM mismatch plots, we measure an explicit normalization-equivariance defect
\[
    \varepsilon_{\mathrm{NE}}
    =
    \E_{y,a,b}\!\left[
        \frac{
            \|f(ay+b\ones)-(af(y)+b\ones)\|_2
        }{
            \|af(y)+b\ones\|_2+\tau
        }
        \right],
\]
where the expectation averages over evaluation images and randomly sampled affine perturbations \(y\mapsto ay+b\ones\), with \(a\sim \mathrm{U}(0.5,1.5)\) and \(b\sim \mathrm{U}(-0.25,0.25)\).
Lower is better; exact NE gives zero, and the wrapped implementation is visibly zero at plot resolution in these experiments.
Accordingly, \(\varepsilon_{\mathrm{NE}}\) asks not whether multi-noise or soft-NE can recover some empirical robustness, but whether they recover the structural identity itself.

Figure~\ref{fig:app_ne_violation} reports two complementary SwinIR controls.
In the left panel, a baseline trained on a range of noise levels \(\sigma\in[0,55]\) has smaller defect inside the training range than outside it, but the defect remains visibly nonzero and increases beyond the range; the wrapped model stays visibly at zero across the sweep.
In the right panel, a supervised orbit-averaged soft-NE baseline reduces the defect relative to the unconstrained baseline but does not drive it to zero, whereas \methodWNE{} again stays visibly at zero.

\begin{figure}[!htbp]
    \centering
    \subfigure[\textbf{Multi-noise Baseline vs.\ \methodWNE{} (SwinIR).}]{
        \includegraphics[width=0.48\textwidth]{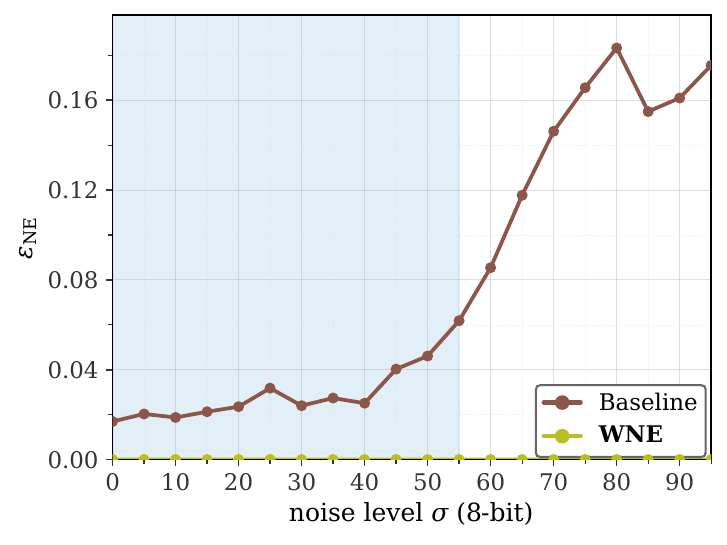}
    }\hfill
    \subfigure[\textbf{Soft-NE vs.\ \methodWNE{} (SwinIR, \(\sigma_{\mathrm{train}}=10\)).}]{
        \includegraphics[width=0.48\textwidth]{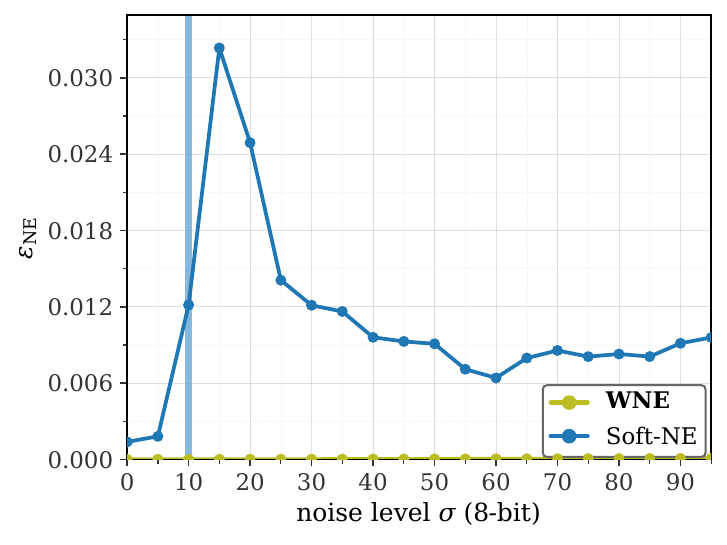}
    }
    \caption{\textbf{Explicit equivariance-defect diagnostics.}
        Left: for a multi-noise SwinIR baseline trained on \(\sigma\in[0,55]\) (shaded region), the equivariance defect \(\varepsilon_{\mathrm{NE}}\) is lower within the training range but remains visibly nonzero and increases beyond it; \methodWNE{} stays visibly at zero across the full sweep.
        Right: for a supervised orbit-averaged soft-NE baseline trained at \(\sigma_{\mathrm{train}}=10\) (vertical line), the defect is reduced but remains visibly nonzero, whereas \methodWNE{} stays visibly at zero.}%
    \label{fig:app_ne_violation}
\end{figure}

\section{Direct Prediction vs Residual Prediction Under the Wrapper}%
\label{app:direct_vs_residual}

This appendix records two equivalent implementations that yield a normalization-equivariant end-to-end predictor.

\paragraph{Direct (clean) prediction.}%
Given \(g_\theta\), define
\[
    \hat x_\theta(y)
    := \std(y)\,g_\theta\bigl(\Tne(y)\bigr)+\mu(y)\ones.
\]
By Proposition~\ref{prop:wrapper_is_ne}, \(\hat x_\theta\) is NE.\@

\paragraph{Residual (noise) prediction.}%
Let \(h_\theta\) predict the normalized residual \((y-x)/\std(y)\) in normalized coordinates and define
\[
    \hat r_\theta(y)
    :=
    \begin{cases}
        \std(y)\,h_\theta\bigl(\Tne(y)\bigr) & \text{if } \std(y)>0, \\[4pt]
        0                                    & \text{if } \std(y)=0.
    \end{cases}
\]
Output \(\hat x_\theta(y):=y-\hat r_\theta(y)\).
Using Proposition~\ref{prop:mu_sigma_affine},
\[
    \hat r_\theta(ay+b\ones)=a\,\hat r_\theta(y),
\]
so
\[
    \hat x_\theta(ay+b\ones)
    = ay+b\ones-a\hat r_\theta(y)
    = a\hat x_\theta(y)+b\ones.
\]
Thus \(\hat x_\theta\) is NE.\@

\paragraph{Conversion.}%
If \(f\) is NE, then the residual map \(r(y):=y-f(y)\) satisfies:
(i) scale equivariance \(r(ay)=a\,r(y)\) for all \(a>0\), and
(ii) shift-invariance \(r(y+b\ones)=r(y)\) for all \(b\in\R\).
Conversely, if a map \(r:\R^d\to\R^d\) satisfies (i) and (ii), then \(f(y):=y-r(y)\) is NE since for all \(a>0\) and \(b\in\R\),
\begin{align*}
    f(ay+b\ones)
    = ay+b\ones-r(ay+b\ones)
    = ay+b\ones-a\,r(y)
    = a f(y)+b\ones.
\end{align*}
This lets you match direct or residual conventions without changing the end-to-end equivariance guarantee.

\section{A Conditional Explanation for Observation~\ref{obs:delta_sufficiency}}%
\label{app:obs1_explanation}
We give a two-step explanation for the \(\Delta\)-driven collapse in Observation~\ref{obs:delta_sufficiency}.
Part~1 is an exact per-instance statement.
Part~2 lifts it to the population level under one modeling assumption.

\paragraph{Part 1: Exact per-instance \(\sigma\)-freeness.}
Recall that under AWGN,
\[
    y = x + \sigma n,
    \qquad
    n \sim \mathcal N(0,I_d),
\]
and that
\[
    \tilde y := \Tne(y),
    \qquad
    \tilde x := \frac{x-\mu(y)\ones}{\std(y)},
    \qquad
    \Delta := \|\tilde y-\tilde x\|_2.
\]
The normalized input \(\tilde y\) always lies on the normalized manifold
\(\mathcal M := \{z \in \R^d : \mu(z)=0,\ \std(z)=1\}\).

\begin{proposition}[Per-instance \(\sigma\)-freeness]\label{app:prop:per_instance_sigma_free}
    Fix a nonconstant clean instance \(x\).
    Then conditioning on \(\sigma\) is redundant given \((x,\tilde x,\Delta)\):
    \[
        p(\tilde y \mid x,\tilde x,\Delta=r,\sigma)
        =
        p(\tilde y \mid x,\tilde x,\Delta=r).
    \]
    Consequently, for any estimator \(g\),
    \[
        \E\!\left[\|g(\tilde y)-\tilde x\|_2^2 \mid x,\tilde x,\Delta=r,\sigma\right]
        =
        \E\!\left[\|g(\tilde y)-\tilde x\|_2^2 \mid x,\tilde x,\Delta=r\right].
    \]
\end{proposition}

\begin{proof}
    Write \(a:=\std(y)\) and \(m:=\mu(y)\).
    Since \(\tilde x=(x-m\ones)/a\), fixing \(x\) and the realized \(\tilde x\) determines \(a\) and \(m\):
    \[
        a = \frac{\std(x)}{\std(\tilde x)},
        \qquad
        m = \mu(x)-a\,\mu(\tilde x).
    \]
    Hence \(x=a\tilde x+m\ones\) and \(y=a\tilde y+m\ones\).
    Under AWGN,
    \[
        p(y \mid x,\sigma)
        \propto
        \exp\!\left(-\frac{\|y-x\|_2^2}{2\sigma^2}\right).
    \]
    Substituting \(y=a\tilde y+m\ones\) and \(x=a\tilde x+m\ones\) gives
    \[
        p(\tilde y \mid x,\tilde x,\sigma)
        \propto
        \exp\!\left(-\frac{a^2\|\tilde y-\tilde x\|_2^2}{2\sigma^2}\right)
        \mathbf 1\{\tilde y \in \mathcal M\},
    \]
    with respect to the natural surface measure on \(\mathcal M\).
    Conditioning further on \(\Delta=\|\tilde y-\tilde x\|_2=r\) makes the exponential factor constant on
    \[
        \mathcal S_r(x,\tilde x):=\{z \in \mathcal M : \|z-\tilde x\|_2=r\}.
    \]
    Therefore the conditional law of \(\tilde y\) given \((x,\tilde x,\Delta=r,\sigma)\) is the normalized surface measure on \(\mathcal S_r(x,\tilde x)\), which does not depend on \(\sigma\).
\end{proof}

\paragraph{High-dimensional simplification.}
Let
\[
    \sigma_x := \std(x),
    \qquad
    \hat x := \frac{x-\mu(x)\ones}{\|x-\mu(x)\ones\|_2}.
\]
In high dimension,
\[
    \mu(y) \approx \mu(x),
    \qquad
    {\std(y)}^2 \approx \sigma_x^2+\sigma^2,
    \qquad
    \Delta^2 \approx \frac{d\,\sigma^2}{\sigma_x^2+\sigma^2}.
\]
Eliminating \(\sigma\) gives \({\std(y)}^2 \approx d\,\sigma_x^2/(d-\Delta^2)\), and therefore
\[
    \tilde x
    =
    \frac{x-\mu(y)\ones}{\std(y)}
    \approx
    \sqrt{d-\Delta^2}\,\hat x.
\]
Thus, conditional on \((x,\Delta=r)\), the matched normalized target \(\tilde x\) is approximately determined, and Proposition~\ref{app:prop:per_instance_sigma_free} yields the approximate per-instance collapse
\[
    \E\!\left[\|g(\tilde y)-\tilde x\|_2^2 \mid x,\Delta=r,\sigma\right]
    \approx
    \phi_g(\hat x,r).
\]

\paragraph{Part 2: Population-level collapse under one assumption.}
To pass from the per-instance statement to the population-level \(\Delta\)-collapse, we must average over the instance distribution and account for how conditioning on \((\Delta,\sigma)\) changes which clean instances are selected.

\begin{assumption}[Contrast-independence of normalized instance structure]\label{app:ass:contrast_independence}
    The normalized instance direction
    \[
        \hat x := \frac{x-\mu(x)\ones}{\|x-\mu(x)\ones\|_2}
    \]
    is approximately independent of the empirical contrast \(\sigma_x := \std(x)\), in the sense that the conditional law of \(\hat x\) given \(\sigma_x=s\) varies only weakly with \(s\) over the range relevant to the analysis.
\end{assumption}

\noindent
Assumption~\ref{app:ass:contrast_independence} is a modeling assumption, not a closed-form consequence of AWGN\@.
It requires that conditioning on empirical contrast does not substantially change the distribution of normalized structure.
This is motivated by the image formation process: variation in empirical contrast across natural images reflects both scene content and acquisition conditions such as exposure, gain, and illumination.
Standard brightness/contrast augmentations implicitly rely on a similar independence.

\begin{proposition}[Population-level \(\Delta\)-collapse]\label{app:prop:population_delta_collapse}
    Under Assumption~\ref{app:ass:contrast_independence}, in the high-dimensional regime, for any estimator \(g : \mathcal M \to \R^d\),
    \[
        \E\!\left[\|g(\tilde y)-\tilde x\|_2^2 \mid \Delta=r,\sigma\right]
        \approx
        \psi_g(r),
    \]
    for some function \(\psi_g\) of \(r\) alone.
\end{proposition}

\begin{proof}
    Apply the tower property:
    \[
        \E\!\left[\|g(\tilde y)-\tilde x\|_2^2 \mid \Delta=r,\sigma\right]
        =
        \E\!\left[
            \E\!\left[\|g(\tilde y)-\tilde x\|_2^2 \mid x,\Delta=r,\sigma\right]
            \Bigm| \Delta=r,\sigma
            \right].
    \]
    By the high-dimensional simplification above, the inner conditional expectation is approximately \(\phi_g(\hat x,r)\).
    The relation \(\Delta^2 \approx d\sigma^2/(\sigma_x^2+\sigma^2)\) implies that at fixed \((\Delta=r,\sigma)\), conditioning selects a narrow band of contrasts \(\sigma_x\).
    Under Assumption~\ref{app:ass:contrast_independence}, selecting instances through this contrast variable does not substantially change the conditional distribution of \(\hat x\).
    Therefore the conditional expectation depends only weakly on \(\sigma\) and can be approximated by a function of \(r\) alone.
\end{proof}

\paragraph{Connection to Observation~\ref{obs:delta_sufficiency}.}
Proposition~\ref{app:prop:population_delta_collapse} should be read as a conditional explanation of Observation~\ref{obs:delta_sufficiency}: the \(\Delta\)-collapse follows exactly at the per-instance level and lifts approximately to the population level under high-dimensional concentration together with Assumption~\ref{app:ass:contrast_independence}.
The residual \(\sigma\)-spread visible in Figure~\ref{fig:ne_psnr_vs_delta} can arise from finite-dimensional concentration error and residual dependence of normalized structure on contrast.

\section{FDnCNN-Family Comparison}%
\label{app:fdncnn_controlled}
We complement the main-text DnCNN results with an FDnCNN-family comparison in which all four variants share the same FDnCNN base architecture.
This brackets one important confound in the main-text comparison: standard DnCNN is used for \methodBaseline{} and \methodWNE{}, whereas the architectural SE/NE references are built on FDnCNN\@.
It removes that standard-DnCNN-versus-FDnCNN mismatch while preserving the equivariance-enforcing modifications of the architectural baselines, especially the stronger internal replacements used by \methodNEarch{}.

\paragraph{Protocol.}
We follow the same fixed-\(\sigma_{\mathrm{train}}\) mismatch diagnostic as in Section~\ref{sec:experiments}: for each \(\sigma_{\mathrm{train}}\in\{10,25,50\}\), models are trained at a single noise level and evaluated on Set12 across varying \(\sigma_{\mathrm{test}}\), plotting output PSNR versus input PSNR\@.

\paragraph{Results.}
Figure~\ref{fig:fdncnn_psnr} and Table~\ref{tab:psnr_matched_fdncnn} show that \methodWNE{}-FDnCNN is competitive with \methodNEarch{} across all three noise levels and that the mismatch curves follow the same qualitative trend as in Figure~\ref{fig:dncnn_psnr}.
The comparison does not collapse these methods into the same model class; rather, it shows that after bracketing the standard-DnCNN-versus-FDnCNN mismatch, \methodWNE{}-FDnCNN still follows a robustness trend close to \methodNEarch{}.

\begin{figure}[!htbp]
    \centering
    \subfigure[\(\sigma_{\mathrm{train}} = 10\)]{
        \includegraphics[width=0.3\textwidth]{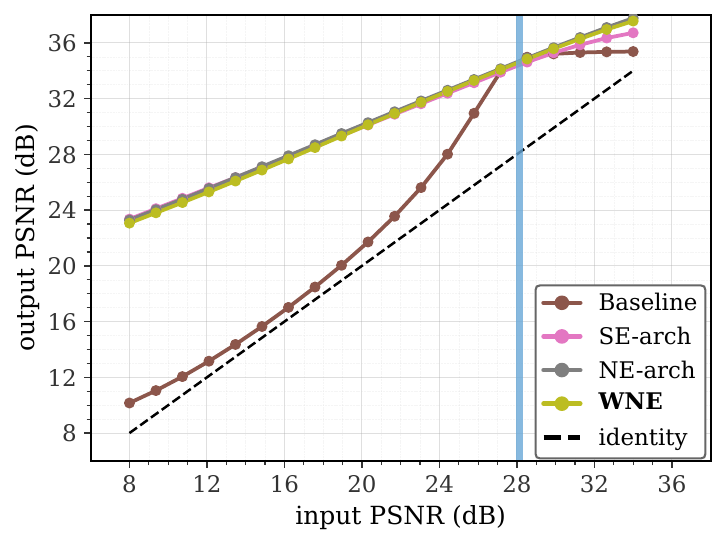}
    }\hfill
    \subfigure[\(\sigma_{\mathrm{train}} = 25\)]{
        \includegraphics[width=0.3\textwidth]{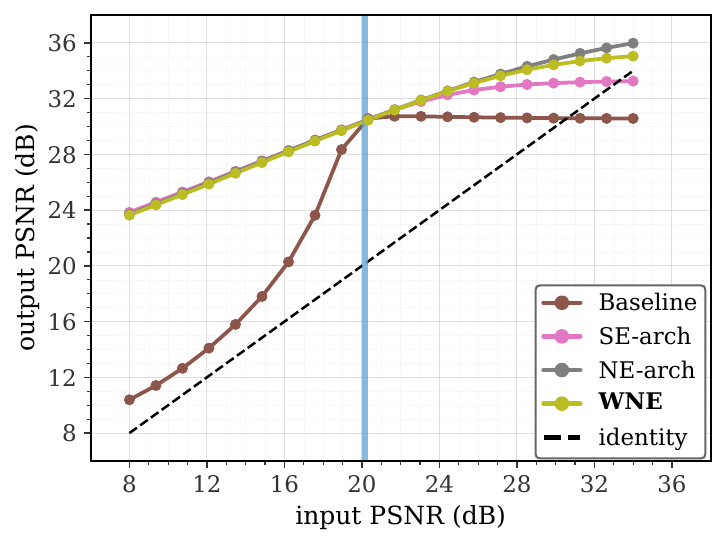}
    }\hfill
    \subfigure[\(\sigma_{\mathrm{train}} = 50\)]{
        \includegraphics[width=0.3\textwidth]{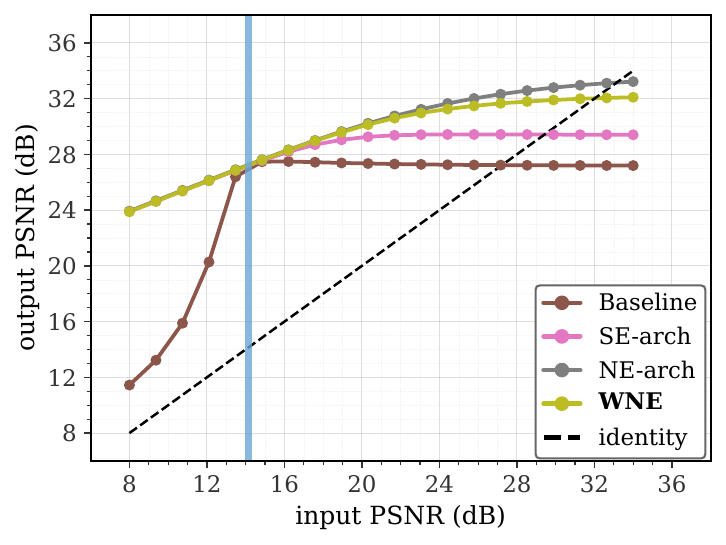}
    }
    \caption{\textbf{FDnCNN-family comparison under noise-level mismatch.}
        Output PSNR versus input PSNR on Set12 for models trained at a single \(\sigma_{\mathrm{train}}\) and tested at varying \(\sigma_{\mathrm{test}}\).
        All four variants share the same FDnCNN base architecture.
        \methodWNE{} closely matches the architectural normalization-equivariant reference \methodNEarch{} and improves over \methodSEarch{} and the unwrapped baseline.
        Dashed line: no denoising (\(\hat x=y\)).}%
    \label{fig:fdncnn_psnr}
\end{figure}

\begin{table}[!htbp]
    \centering
    \setlength{\belowcaptionskip}{3pt}
    \caption{\textbf{Matched-noise PSNR (dB) on the FDnCNN backbone.}
        Average PSNR on Set12 and BSD68 for AWGN with \(\sigma\in\{10,25,50\}\) (8-bit units).
        All four variants share the same FDnCNN base architecture.
        \emph{\methodSEarch{}} and \emph{\methodNEarch{}} modify internal layers to enforce equivariance, while \emph{\methodWNE{}} wraps the unmodified backbone.}%
    \label{tab:psnr_matched_fdncnn}
    \small
    \setlength{\tabcolsep}{3.5pt}
    \renewcommand{\arraystretch}{1.05}
    \begin{tabular}{lcccccc}
        \toprule
                                            & \multicolumn{3}{c}{Set12} & \multicolumn{3}{c}{BSD68}                                 \\
        Noise level \(\sigma\) (8-bit)      & 10                        & 25                        & 50    & 10    & 25    & 50    \\
        \midrule
        \emph{\methodBaseline{}}            & 34.82                     & 30.51                     & 27.29 & 33.87 & 29.21 & 26.25 \\
        \emph{\methodSEarch{}}              & 34.41                     & 30.42                     & 27.20 & 33.51 & 29.12 & 26.12 \\
        \emph{\methodNEarch{}}              & 34.69                     & 30.42                     & 27.26 & 33.77 & 29.13 & 26.21 \\
        \textbf{\emph{\methodWNE{}}} (ours) & 34.62                     & 30.38                     & 27.23 & 33.72 & 29.11 & 26.19 \\
        \bottomrule
    \end{tabular}
\end{table}

\section{Normalized-Space Error Versus Difficulty (Baseline vs WNE)}%
\label{app:norm_error_vs_difficulty}

In this appendix section, we extend the normalized-coordinate analysis of Section~\ref{sec:analysis} by comparing the unmodified \methodBaseline{} to the wrapped \methodWNE{} across multiple training noise levels.
Our aim is to test two related hypotheses: (i) whether the normalized regression error is primarily controlled by the difficulty variable \(\Delta=\|\tilde y-\tilde x\|_2\) rather than by the test noise level \(\sigma_{\mathrm{test}}\), and (ii) whether enforcing input-output normalization equivariance strengthens this difficulty-driven behavior.
All figures in this section use SwinIR as the backbone.

\paragraph{Normalized coordinates and difficulty.}
Given a clean target \(x\in\R^d\) and noisy observation \(y=x+\sigma n\), define
\begin{align*}
    \tilde y & := \Tne(y),                                        &
    \tilde x & := \Tne(x\,;\,y) := \frac{x-\mu(y)\ones}{\std(y)},
    \\
    \delta   & := \tilde y-\tilde x,                              &
    \Delta   & := \|\delta\|_2.
\end{align*}
We interpret \(\Delta\) as a one-dimensional index of denoising difficulty in normalized coordinates (Section~\ref{subsec:problem_difficulty}).

\paragraph{A model-agnostic normalized error.}
To compare \methodBaseline{} and \methodWNE{} in the same coordinate system, we re-express any model output \(\hat x=f(y)\) using the instance statistics of the input \(y\):
\[
    \tilde{\hat x}
    :=
    \frac{\hat x-\mu(y)\ones}{\std(y)}.
\]
This does not assume that \methodBaseline{} is NE;\@
it is a diagnostic that expresses any predictor in the \((\mu(y),\std(y))\)-anchored coordinates used by NE.\@
We then measure normalized-space regression quality via
\[
    Q_f(\tilde y,\tilde x)
    :=
    -10\log_{10}\big\|\tilde{\hat x}-\tilde x\big\|_2^2,
\]
where larger values indicate smaller normalized squared error.
For \methodWNE{}, we have \(\tilde{\hat x}=g_\theta(\tilde y)\) by construction, so \(Q_f\) coincides with the backbone error in normalized coordinates used in Section~\ref{sec:analysis}.
For \methodBaseline{}, \(Q_f\) evaluates how well the raw-space prediction aligns with the normalized target \(\tilde x\) after mapping the output into the \((\mu(y),\std(y))\)-anchored coordinates.

\paragraph{Binning and summary curves.}
For each \(\sigma_{\mathrm{test}}\), we bin samples by \(\Delta\) (fixed-width bins) and plot the empirical mean of \(Q_f(\tilde y,\tilde x)\) in each bin, with \(\pm 1\) standard deviation shading.
All panels share identical axis limits to make cross-\(\sigma_{\mathrm{test}}\) comparisons visually consistent.

\paragraph{Displayed difficulty range (per-curve central mass).}
In each panel, each \(\sigma_{\mathrm{test}}\) curve is displayed only over the central \(95\%\) mass of its own test-time \(\Delta\) distribution, that is,
\[
    \Delta \in \bigl[q_{0.025}(\sigma_{\mathrm{test}}),\; q_{0.975}(\sigma_{\mathrm{test}})\bigr].
\]
We do not interpret tail behavior from these plots.

\paragraph{Results.}
Figure~\ref{fig:app_q_vs_delta_grid} reports six plots arranged by training noise level (rows) and model variant (columns).
Across \(\sigma_{\mathrm{test}}\), \methodWNE{} exhibits an approximate collapse of \(Q_f\) versus \(\Delta\) within the bulk of each curve, indicating that once difficulty is controlled by \(\Delta\), the residual dependence of normalized regression error on \(\sigma_{\mathrm{test}}\) is small.
In contrast, \methodBaseline{} shows a stronger dependence on \(\sigma_{\mathrm{test}}\) at fixed \(\Delta\), suggesting that without explicit input-output normalization equivariance, normalized-coordinate behavior is less purely difficulty-driven.
Together with the train-test overlap in \(\Delta\) (Table~\ref{tab:ne_delta_coverage}), these results support the interpretation that improved cross-noise robustness under \methodWNE{} arises from reduced distribution shift in normalized coordinates and a backbone response that is closer to a difficulty-indexed regression.
\begin{figure}[!htbp]
    \centering

    \subfigure[\(\sigma_{\mathrm{train}}=10\), \methodWNE]{
        \includegraphics[width=0.45\textwidth]{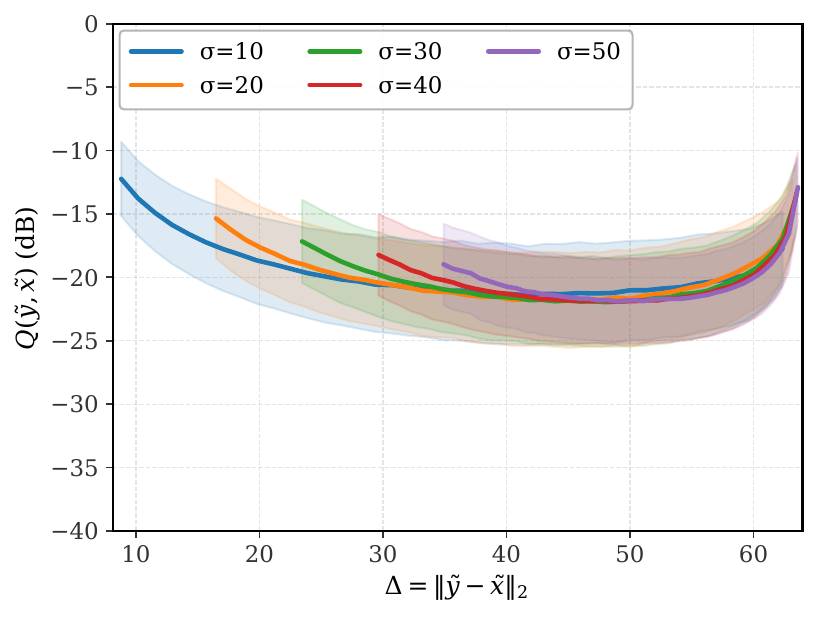}
    }
    \subfigure[\(\sigma_{\mathrm{train}}=10\), \methodBaseline]{
        \includegraphics[width=0.45\textwidth]{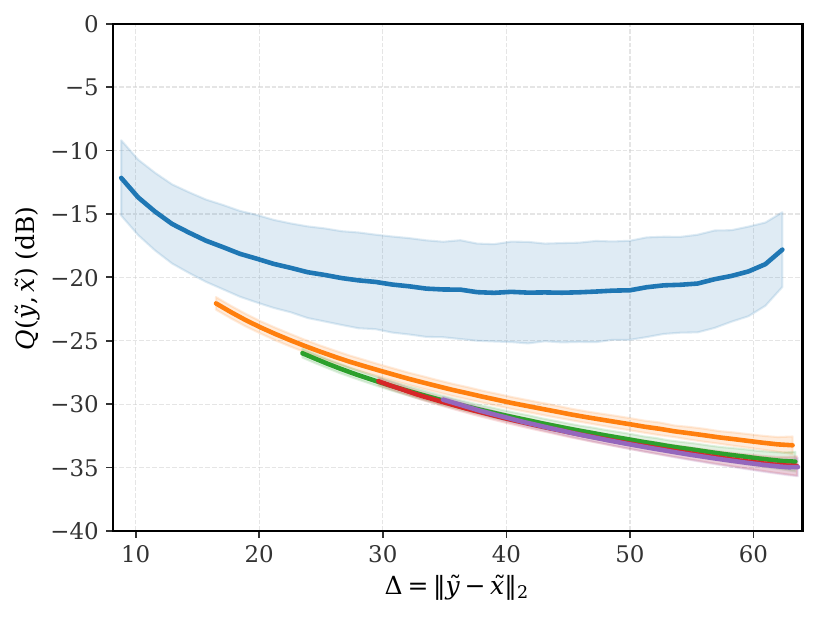}
    }\\[2pt]

    \subfigure[\(\sigma_{\mathrm{train}}=25\), \methodWNE{}]{
        \includegraphics[width=0.45\textwidth]{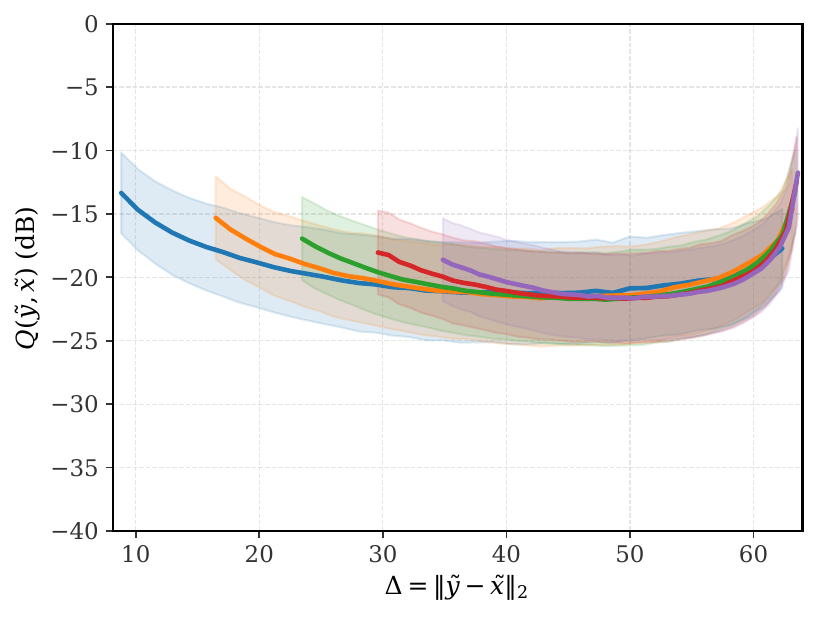}
    }
    \subfigure[\(\sigma_{\mathrm{train}}=25\), \methodBaseline]{
        \includegraphics[width=0.45\textwidth]{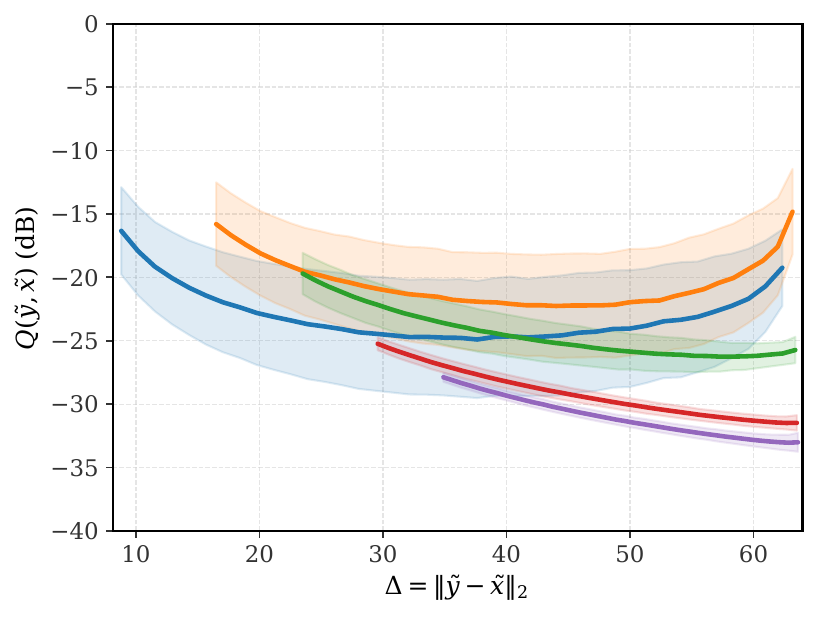}
    }\\[2pt]

    \subfigure[\(\sigma_{\mathrm{train}}=50\), \methodWNE{}]{
        \includegraphics[width=0.45\textwidth]{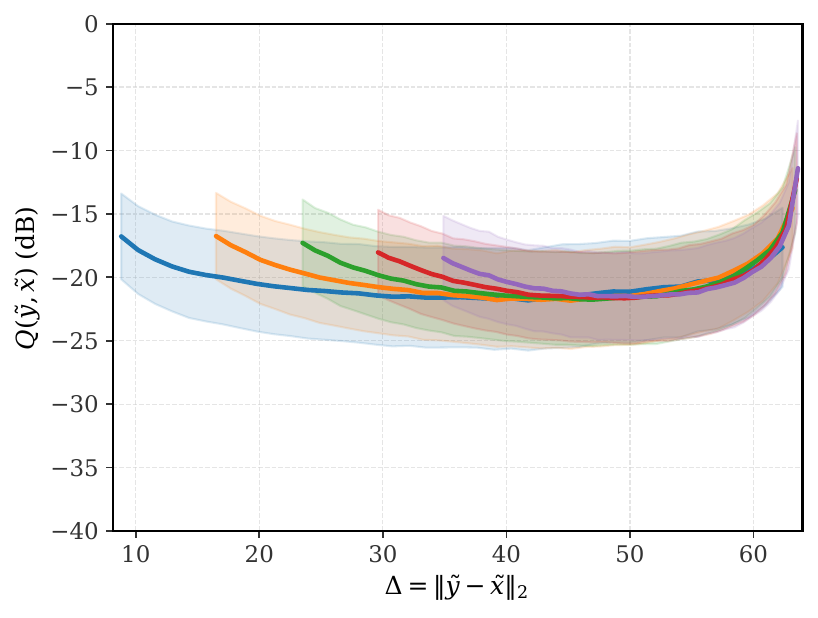}
    }
    \subfigure[\(\sigma_{\mathrm{train}}=50\), \methodBaseline]{
        \includegraphics[width=0.45\textwidth]{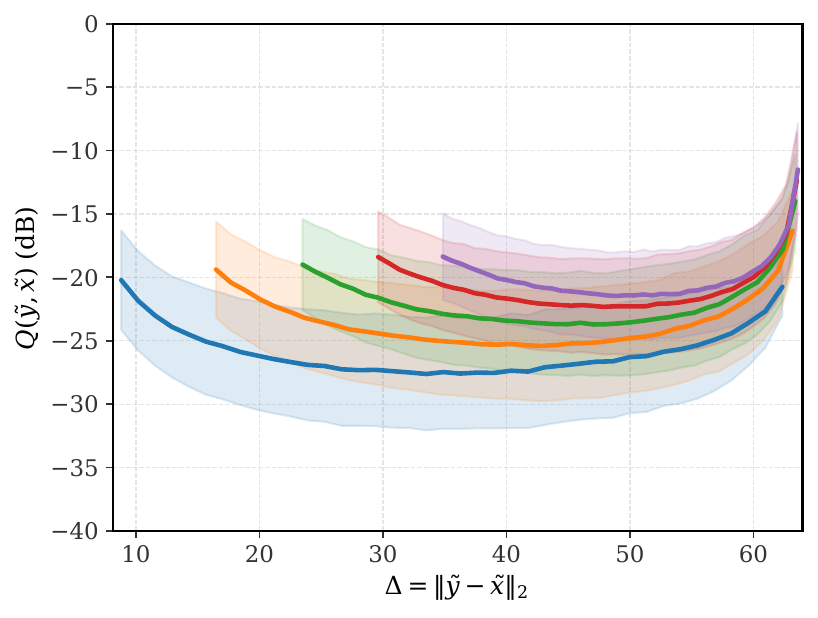}
    }

    \caption{\textbf{Normalized-space error versus difficulty (Baseline vs WNE).}
    For each \(\sigma_{\mathrm{train}}\in\{10,25,50\}\) (rows), we plot the normalized regression quality
    \(Q_f(\tilde y,\tilde x)\) versus difficulty \(\Delta=\|\tilde y-\tilde x\|_2\) for multiple test noise levels \(\sigma_{\mathrm{test}}\),
    comparing \methodWNE{} (left) to \methodBaseline{} (right).
    For readability, each \(\sigma_{\mathrm{test}}\) curve is shown only over the central \(95\%\) mass of its own test-time \(\Delta\) distribution,
    \(\Delta\in[q_{0.025}(\sigma_{\mathrm{test}}),q_{0.975}(\sigma_{\mathrm{test}})]\),
    and bins outside this range are omitted.}%
    \label{fig:app_q_vs_delta_grid}
\end{figure}

\section{SSIM Under Noise-Level Mismatch}%
\label{app:ssim}
We report the Structural Similarity Index (SSIM)~\citep{wang2004image} alongside PSNR for the main DnCNN and SwinIR mismatch experiments.
The same fixed-\(\sigma_{\mathrm{train}}\) diagnostic is used as in Section~\ref{sec:experiments}, but output quality is measured with SSIM instead of PSNR\@.

\begin{figure}[!htbp]
    \centering
    \subfigure[\(\sigma_{\mathrm{train}} = 10\)]{
        \includegraphics[width=0.3\textwidth]{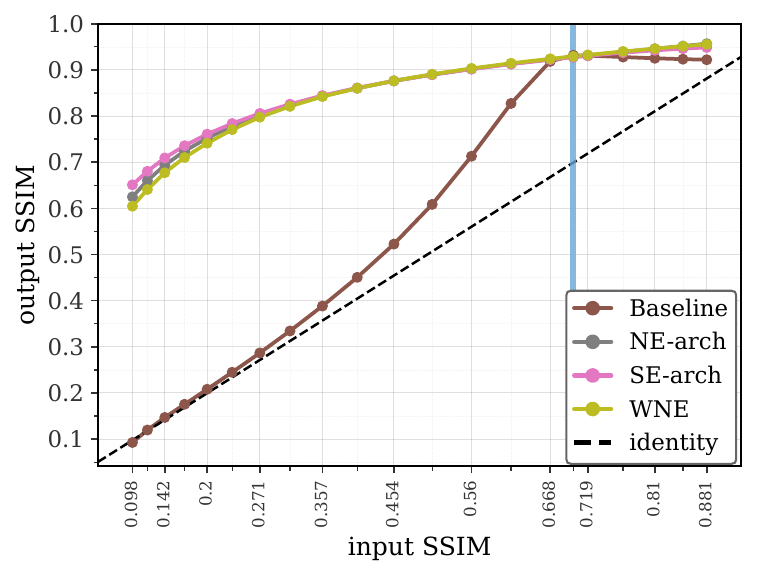}
    }\hfill
    \subfigure[\(\sigma_{\mathrm{train}} = 25\)]{
        \includegraphics[width=0.3\textwidth]{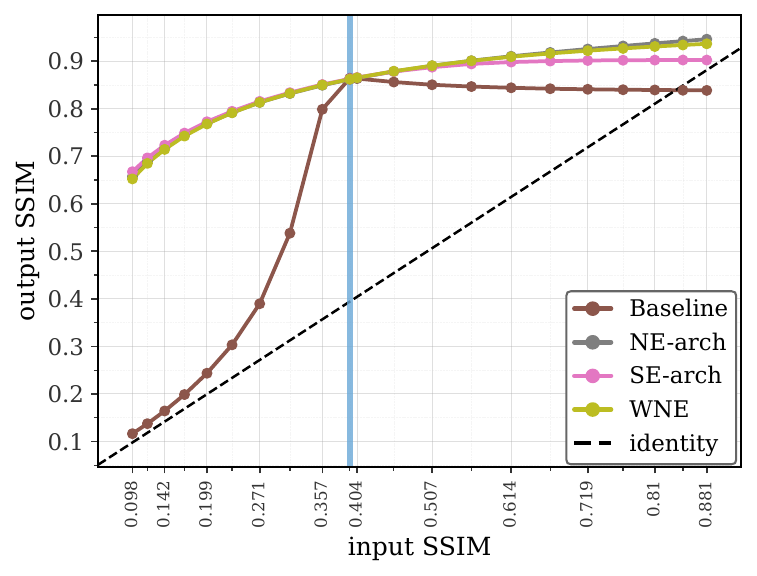}
    }\hfill
    \subfigure[\(\sigma_{\mathrm{train}} = 50\)]{
        \includegraphics[width=0.3\textwidth]{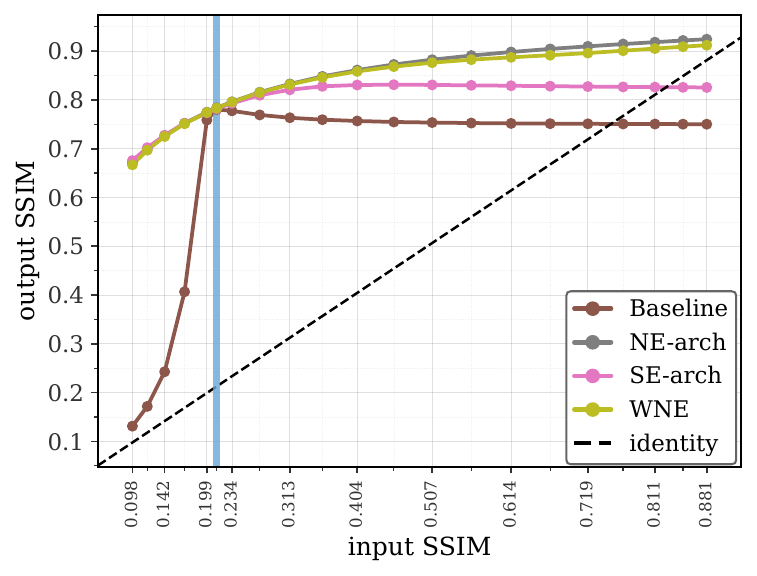}
    }
    \caption{\textbf{SSIM under noise-level mismatch for DnCNN.}
        Same mismatch diagnostic as Figure~\ref{fig:dncnn_psnr}, but measured with SSIM instead of PSNR\@.
        Across all three training noise levels, \methodWNE{} preserves the robustness trend of normalization-equivariant models under train-test mismatch, while the \methodBaseline{} remains noise-level specific.}%
    \label{fig:dncnn_ssim}
\end{figure}

\begin{figure}[!htbp]
    \centering
    \subfigure[\(\sigma_{\mathrm{train}} = 10\)]{
        \includegraphics[width=0.3\textwidth]{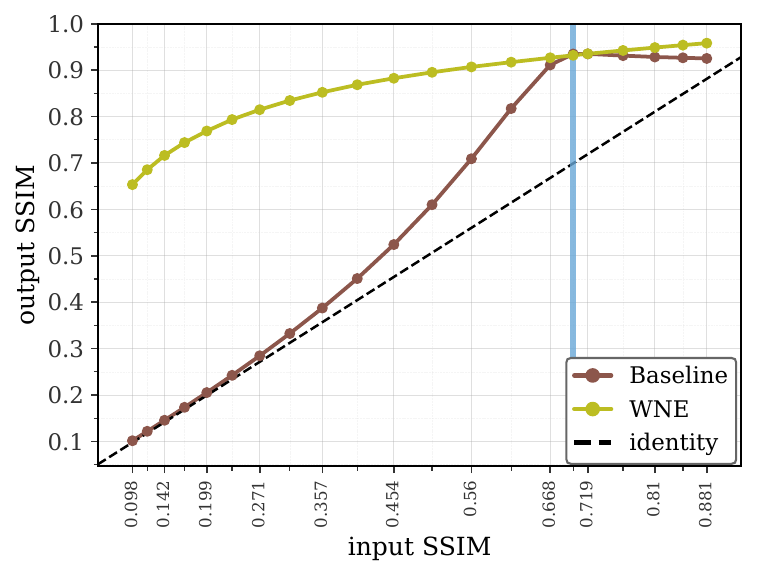}
    }\hfill
    \subfigure[\(\sigma_{\mathrm{train}} = 25\)]{
        \includegraphics[width=0.3\textwidth]{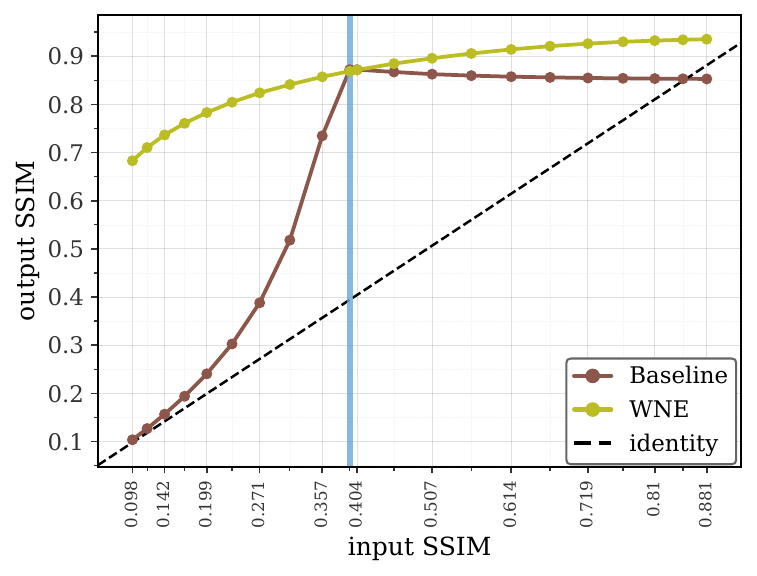}
    }\hfill
    \subfigure[\(\sigma_{\mathrm{train}} = 50\)]{
        \includegraphics[width=0.3\textwidth]{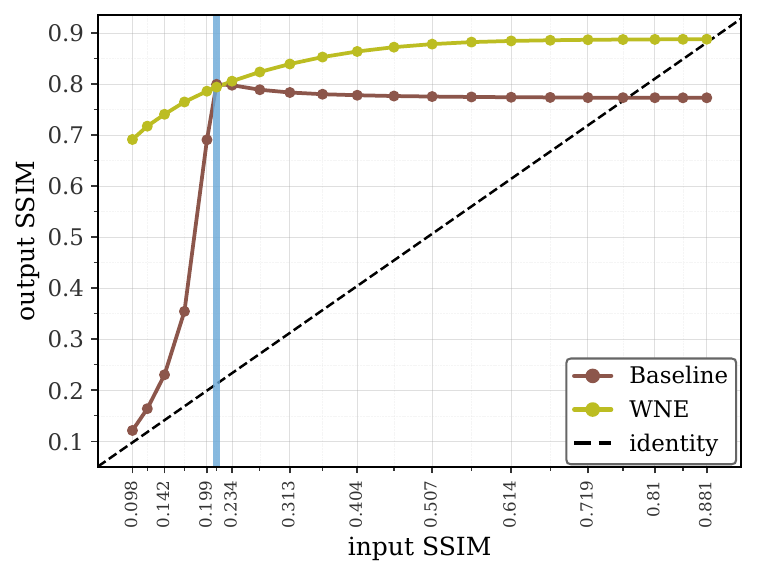}
    }
    \caption{\textbf{SSIM under noise-level mismatch for SwinIR.}
        Same mismatch diagnostic as Figure~\ref{fig:swinir_psnr}, but measured with SSIM instead of PSNR\@.
        The \methodBaseline{} SwinIR degrades once \(\sigma_{\mathrm{test}} \neq \sigma_{\mathrm{train}}\), while \methodWNE{} stabilizes structural quality across a broader range of test noise levels.}%
    \label{fig:swinir_ssim}
\end{figure}

\section{Restormer Backbone Experiments}%
\label{app:restormer}
We use Restormer to test whether the wrapper's robustness benefit extends to a second transformer backbone beyond SwinIR\@.
Restormer uses channel-wise self-attention, a substantially different mechanism from SwinIR's shifted-window attention.

\paragraph{Protocol.}
We follow the same fixed-\(\sigma_{\mathrm{train}}\) mismatch diagnostic as in Section~\ref{sec:experiments}: for each \(\sigma_{\mathrm{train}}\in\{10,25,50\}\), models are trained at a single noise level and evaluated on Set12 across varying \(\sigma_{\mathrm{test}}\), plotting output PSNR versus input PSNR\@.
Both variants share the same Restormer backbone described in Appendix~\ref{app:arch_models}, with only the outer \methodWNE{} wrapper added in the wrapped condition.

\begin{figure}[!htbp]
    \centering
    \subfigure[\(\sigma_{\mathrm{train}} = 10\)]{
        \includegraphics[width=0.3\textwidth]{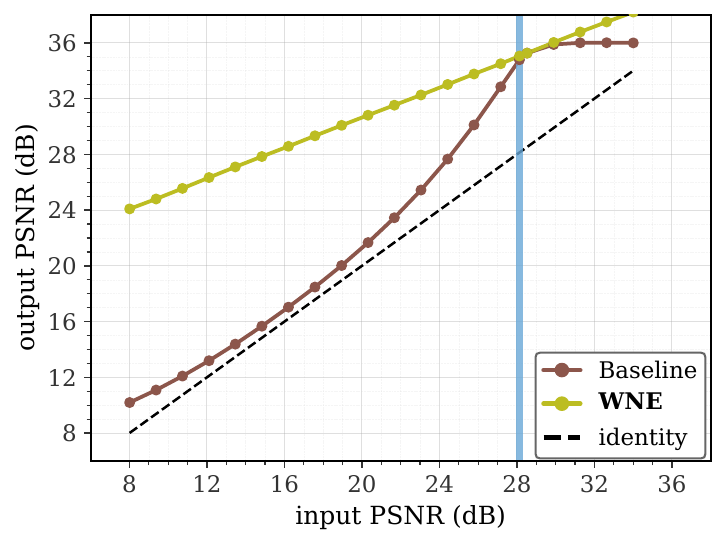}
    }\hfill
    \subfigure[\(\sigma_{\mathrm{train}} = 25\)]{
        \includegraphics[width=0.3\textwidth]{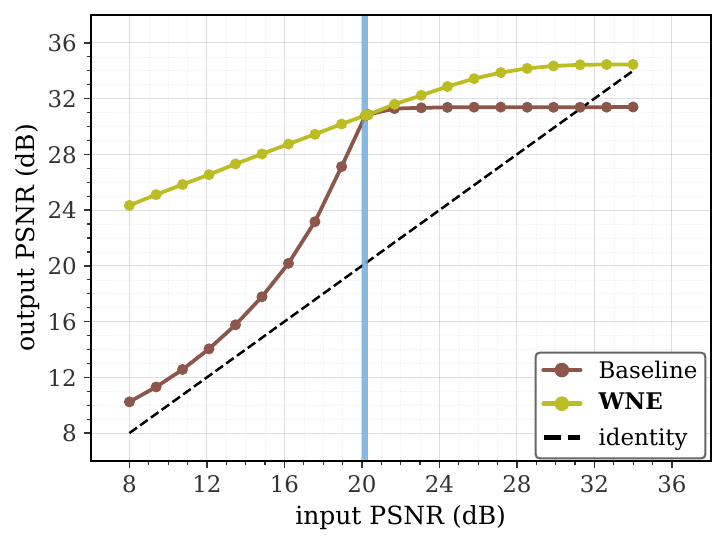}
    }\hfill
    \subfigure[\(\sigma_{\mathrm{train}} = 50\)]{
        \includegraphics[width=0.3\textwidth]{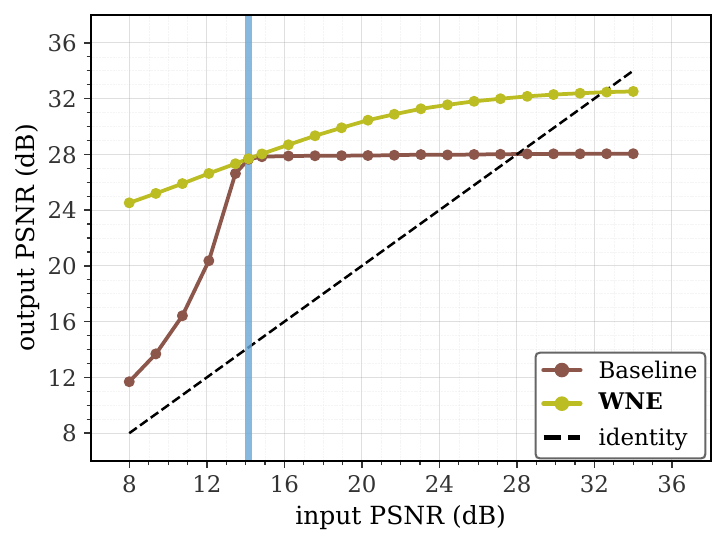}
    }
    \caption{\textbf{Noise-level mismatch on Restormer.}
        Same diagnostic as Figures~\ref{fig:dncnn_psnr}--\ref{fig:swinir_psnr}, applied to a second transformer backbone.
        The \methodBaseline{} Restormer degrades under mismatch, while \methodWNE{} stabilizes performance across test noise levels.
        Dashed line: no denoising (\(\hat x = y\)).}%
    \label{fig:restormer_psnr}
\end{figure}

\paragraph{Results.}
Figure~\ref{fig:restormer_psnr} shows the same qualitative mismatch pattern as for DnCNN and SwinIR\@: the unwrapped baseline is noise-level specific, while \methodWNE{} yields smoother and more stable extrapolation across test noise levels.
Table~\ref{tab:psnr_matched_restormer} reports matched-noise PSNR\@; unlike DnCNN and SwinIR, \methodWNE{} is fully competitive with the baseline at all three noise levels on this backbone.

\begin{table}[!htbp]
    \centering
    \setlength{\belowcaptionskip}{3pt}
    \caption{\textbf{Matched-noise PSNR (dB) on Restormer (grayscale).}
        Average PSNR on Set12 and BSD68 for AWGN with \(\sigma\in\{10,25,50\}\) (8-bit units).
        Both variants share the same Restormer backbone; \emph{\methodWNE{}} wraps the unmodified architecture.}%
    \label{tab:psnr_matched_restormer}
    \small
    \setlength{\tabcolsep}{3.5pt}
    \renewcommand{\arraystretch}{1.05}
    \begin{tabular}{lcccccc}
        \toprule
                                            & \multicolumn{3}{c}{Set12} & \multicolumn{3}{c}{BSD68}                                 \\
        Noise level \(\sigma\) (8-bit)      & 10                        & 25                        & 50    & 10    & 25    & 50    \\
        \midrule
        \emph{\methodBaseline{}}            & 34.79                     & 30.76                     & 27.65 & 33.71 & 29.29 & 26.42 \\
        \textbf{\emph{\methodWNE{}}} (ours) & 35.05                     & 30.79                     & 27.69 & 34.00 & 29.38 & 26.45 \\
        \bottomrule
    \end{tabular}
\end{table}

\section{Color Denoising (Restormer on CBSD68)}%
\label{app:restormer_color}
To test whether the wrapper extends to multi-channel data, we train Restormer with and without \methodWNE{} on 3-channel color images under the same fixed-\(\sigma_{\mathrm{train}}\) mismatch protocol, evaluating on CBSD68.
The wrapper pools \(\mu(y)\) and \(\std(y)\) jointly over all channels and pixels, matching the global affine action \(y \mapsto ay + b\ones\) without introducing per-channel statistics.

\begin{figure}[!htbp]
    \centering
    \subfigure[\(\sigma_{\mathrm{train}} = 10\)]{
        \includegraphics[width=0.3\textwidth]{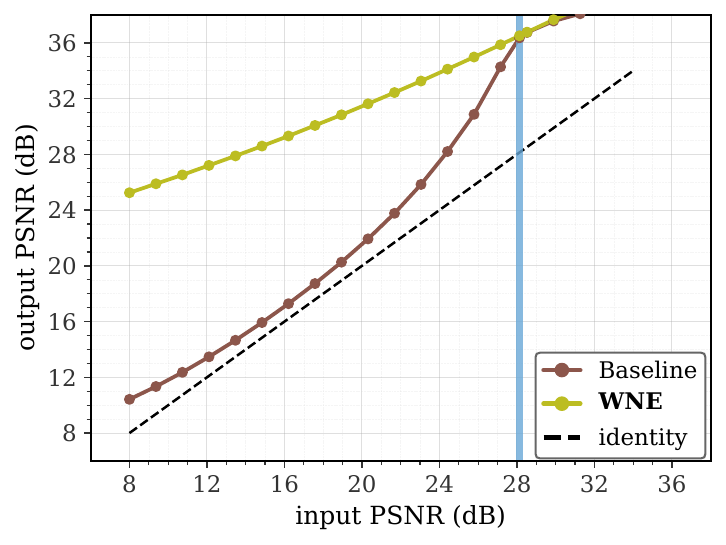}
    }\hfill
    \subfigure[\(\sigma_{\mathrm{train}} = 25\)]{
        \includegraphics[width=0.3\textwidth]{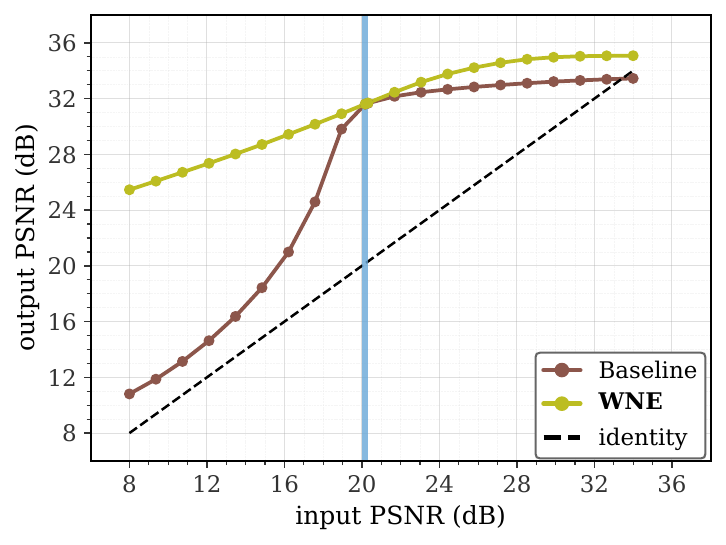}
    }\hfill
    \subfigure[\(\sigma_{\mathrm{train}} = 50\)]{
        \includegraphics[width=0.3\textwidth]{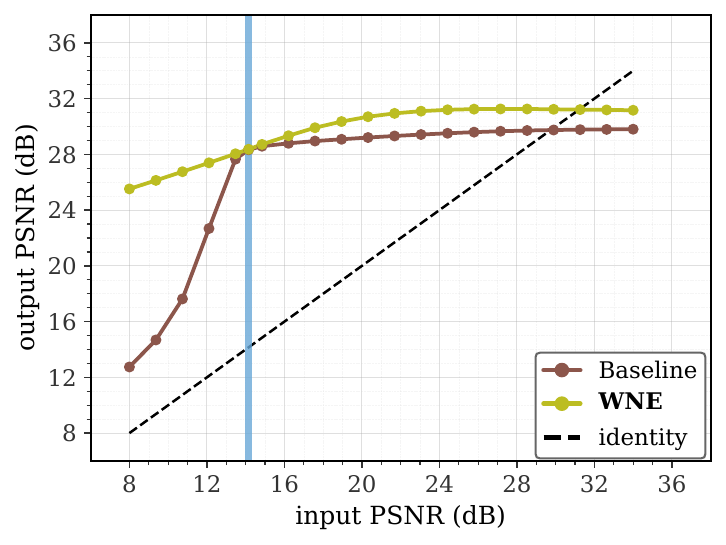}
    }
    \caption{\textbf{Noise-level mismatch on color Restormer (CBSD68).}
        Same diagnostic as Figures~\ref{fig:dncnn_psnr}--\ref{fig:swinir_psnr}, now applied to 3-channel color images.
        The \methodBaseline{} Restormer degrades under mismatch, while \methodWNE{} stabilizes performance across test noise levels.
        Dashed line: no denoising (\(\hat x = y\)).}%
    \label{fig:restormer_color_psnr}
\end{figure}

\paragraph{Results.}
Figure~\ref{fig:restormer_color_psnr} shows the same qualitative pattern as the grayscale Restormer experiments: the baseline is noise-level specific, while \methodWNE{} stabilizes performance under mismatch.
Table~\ref{tab:psnr_matched_restormer_color} reports matched-noise PSNR on CBSD68; \methodWNE{} remains competitive with the baseline at all three noise levels.

\begin{table}[!htbp]
    \centering
    \setlength{\belowcaptionskip}{3pt}
    \caption{\textbf{Matched-noise PSNR (dB) on color Restormer (CBSD68).}
        Average PSNR on CBSD68 for AWGN with \(\sigma\in\{10,25,50\}\) (8-bit units).
        Both variants share the same Restormer backbone trained on color patches; \emph{\methodWNE{}} wraps the unmodified architecture with statistics pooled jointly over all channels and pixels.}%
    \label{tab:psnr_matched_restormer_color}
    \small
    \setlength{\tabcolsep}{5pt}
    \renewcommand{\arraystretch}{1.05}
    \begin{tabular}{lccc}
        \toprule
                                            & \multicolumn{3}{c}{CBSD68}                 \\
        Noise level \(\sigma\) (8-bit)      & 10                         & 25    & 50    \\
        \midrule
        \emph{\methodBaseline{}}            & 36.36                      & 31.60 & 28.33 \\
        \textbf{\emph{\methodWNE{}}} (ours) & 36.51                      & 31.61 & 28.36 \\
        \bottomrule
    \end{tabular}
\end{table}

\section{Real-Noise Validation on SIDD}%
\label{app:sidd}
We validate that the wrapper does not degrade performance on real sensor noise by training SwinIR \methodBaseline{} and SwinIR-\methodWNE{} on the SIDD sRGB track~\citep{abdelhamed2018highquality}.
Unlike the controlled fixed-\(\sigma_{\mathrm{train}}\) AWGN protocol of Section~\ref{sec:experiments}, SIDD exhibits sensor-, ISP-, and scene-dependent noise, so this setting complements rather than replicates the mismatch diagnostic: it tests that enforcing NE remains compatible with a realistic noise distribution rather than isolating the NE mechanism.

\paragraph{Protocol.}
We train SwinIR \methodBaseline{} and SwinIR-\methodWNE{} on pre-extracted SIDD sRGB training patches using the same backbone in both conditions, with only the outer wrapper changed.
SIDD training uses random \(64\times 64\) crops, batch size \(32\), AdamW, and a cosine-decayed learning rate from \(3\times 10^{-4}\) to \(10^{-6}\) over \(400{,}000\) updates.
Evaluation uses the standard SIDD validation blocks, corresponding to \(1280\) sRGB blocks in total.
For reporting, predictions and targets are quantized to uint8 sRGB\@.
PSNR is computed with the standard 8-bit formula, and SSIM is computed in RGB with Wang-style settings~\citep{wang2004image}.

\begin{table}[!htbp]
    \centering
    \setlength{\belowcaptionskip}{3pt}
    \caption{\textbf{SIDD validation (sRGB).}
        Average PSNR/SSIM over the \(1280\) official SIDD validation blocks.
        Both variants use the same SwinIR backbone; \emph{\methodWNE{}} only adds the outer normalize-process-denormalize wrapper.}%
    \label{tab:sidd_validation}
    \small
    \setlength{\tabcolsep}{8pt}
    \renewcommand{\arraystretch}{1.05}
    \begin{tabular}{lcc}
        \toprule
                                            & PSNR (dB) & SSIM   \\
        \midrule
        \emph{\methodBaseline{}}            & 39.27     & 0.9150 \\
        \textbf{\emph{\methodWNE{}}} (ours) & 39.26     & 0.9156 \\
        \bottomrule
    \end{tabular}
\end{table}

\paragraph{Results.}
Table~\ref{tab:sidd_validation} shows that the two SwinIR variants are matched on SIDD validation: \methodWNE{} is within \(0.01\) dB of the baseline in PSNR and essentially identical in SSIM.\@
Combined with the AWGN matched-noise results in Table~\ref{tab:psnr_matched}, this indicates that the small matched-AWGN trade-off visible for SwinIR on Set12/BSD68 does not persist on real sensor noise.

\section{Reusing Noise2Noise-Trained Denoisers for Sampling and Inverse Problems}%
\label{app:n2n_wne}
The normalized-coordinate analysis in Section~\ref{sec:analysis} predicts that exact NE should help single-\(\sigma\)-trained denoisers remain usable across sampler noise schedules: changing raw noise scale primarily shifts which normalized difficulties are queried, rather than changing what the backbone sees at fixed difficulty.
We test this prediction by training denoisers with Noise2Noise (N2N)~\citep{lehtinen2018noise2noise} at one noise level and reusing the frozen maps for iterative sampling and constrained inverse-problem reconstruction, alongside a one-step mismatch diagnostic.

We train the same SwinIR backbone with the N2N objective: given two independently corrupted observations \(y_1=x+\sigma_{\mathrm{train}}\eta_1\) and \(y_2=x+\sigma_{\mathrm{train}}\eta_2\) of the same image, both variants minimize \(\|f_\theta(y_1)-y_2\|_2^2\) at \(\sigma_{\mathrm{train}}=10\).
The \methodBaseline{}-N2N model uses the unwrapped predictor; \methodWNE{}-N2N wraps the same backbone, changing only the parameterization.
This is distinct from N2N-based reconstruction and RED/PnP-style methods that train self-supervised priors for the target inverse problem from split, repeated, or undersampled measurements~\citep{hendriksen2020noise2inverse,liu2020rare,desai2023noise2recon,huang2024dured}; here the learned map is trained only as an ordinary single-step denoiser and then reused unchanged.
Unless otherwise stated, AWGN training and denoising noise levels are reported in 8-bit units; for example, \(\sigma=10\) means \(10/255\) when images are scaled to \([0,1]\).
We evaluate the resulting denoisers in three settings: one-step mismatch across unseen test noise levels, residual-stopped iterative denoising, and constrained random inpainting from sparse observations.
The unified sampler used for the two iterative settings is given in Appendix~\ref{app:n2n_iterative_stability}.

\paragraph{One-step mismatch.}
Figure~\ref{fig:n2n_noise_sweep} applies the same input-output PSNR mismatch diagnostic used in the main experiments to the two N2N-trained models.
\methodWNE{}-N2N is more stable than the unwrapped baseline across the test sweep, including far from \(\sigma_{\mathrm{train}}\).
The mismatch benefit is therefore not tied to clean-target supervision.

\begin{figure}[!htbp]
    \centering
    \includegraphics[width=0.72\linewidth]{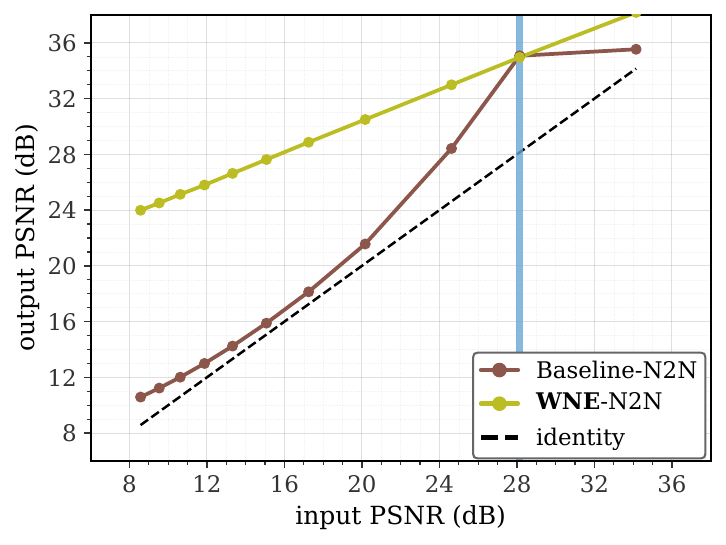}
    \caption{\textbf{Noise2Noise one-step mismatch check on SwinIR.}
        Both models are trained with the N2N objective at \(\sigma_{\mathrm{train}}=10\) in 8-bit units and evaluated using the same output-PSNR versus input-PSNR mismatch diagnostic as the main experiments.
        The wrapped \methodWNE{}-N2N model is more stable than the unwrapped Baseline-N2N model across the test sweep.}%
    \label{fig:n2n_noise_sweep}
\end{figure}

\subsection{Iterative Reuse: A Unified Sampler and the Residual-Stopped Case}\label{app:n2n_iterative_stability}

Both iterative experiments use the same Kadkhodaie-Simoncelli-style denoiser-residual sampler~\citep{kadkhodaie2021stochastic}, parameterized by an orthogonal projector \(P\) onto an observed measurement subspace.
Let \(R(y):=D(y)-y\) be the denoiser residual and \(x_c\) the observed measurement, with \(P=0\) and \(x_c=0\) for unconstrained denoising and \(P=MM^\top\), \(x_c=Px^\star\) for measurement-constrained reconstruction.
At iterate \(y_{t-1}\), the update direction and effective noise level are
\[
    u_t := (I-P)R(y_{t-1}) + x_c - Py_{t-1},
    \qquad
    \widehat{\sigma}_t := \frac{\|u_t\|_2}{\sqrt d}.
\]
The sampler updates
\[
    y_t = y_{t-1} + h_t u_t + \gamma_t z_t,
    \qquad z_t\sim\mathcal{N}(0,I_d),
\]
with step schedule \(h_t = h_0 t/(1+h_0(t-1))\) and injected-noise amplitude
\[
    \gamma_t^2 =
    \bigl[{(1-\beta h_t)}^2-{(1-h_t)}^2\bigr]\widehat{\sigma}_t^2.
\]
For \(P=0\), the direction reduces to \(u_t=R(y_{t-1})\) and \(\widehat{\sigma}_t=\|R(y_{t-1})\|_2/\sqrt d\).
This is the update used in the residual-stopped denoising experiment below.
Equivalently, it is Algorithm~\ref{alg:kadkhodaie_linear_inverse} with \(P=0\), using the denoising initialization \(y_0=x^\star+\sigma_{\mathrm{init}}\eta\).
For \(P=I\), the denoiser term vanishes and the update becomes a pure data-consistency correction; the random-inpainting experiment below uses the intermediate case.

\paragraph{Why exact NE helps under iterative reuse.}
The relevant mechanism is the normalized-coordinate reduction analyzed in Section~\ref{sec:analysis}.
Under NE, the denoiser factors as
\[
    D(y)=\std(y)\,g(\Tne(y))+\mu(y)\ones.
\]
Changing the raw noise scale therefore does not require learning an unrelated raw-space map; it changes which normalized difficulties \(\Delta=\|\tilde y-\tilde x\|_2\) are queried.
Section~\ref{sec:analysis} shows that these difficulty distributions overlap substantially across noise levels and that the normalized error is largely \(\Delta\)-driven.
The N2N sampler experiments stress the same mechanism in an iterative setting: a denoiser trained from noisy pairs at only \(\sigma_{\mathrm{train}}=10\) is repeatedly queried across the sampler's effective noise scales, in the inpainting case starting about \(25\times\) above the training scale.
\methodWNE{} does not make the N2N denoiser Bayes-optimal at every scale; it makes the repeated calls use the same normalized regression \(g\), with the residual and residual-based stopping scale rescaled analytically.
The unwrapped denoiser has no corresponding normalized-coordinate reduction when used off scale.

\paragraph{Residual-stopped iterative denoising.}
We instantiate the unconstrained case (\(P=0\), \(\beta=1\), so \(\gamma_t=0\)).
Each Set12 image is initialized with AWGN at \(\sigma_{\mathrm{init}}=10\), and the trajectory is stopped when \(\widehat{\sigma}_t\le 1/255\).
For each trajectory we record the one-pass denoiser PSNR, the best PSNR attained along the trajectory, and the residual-stopped final PSNR\@.
The one-pass PSNR is computed on \(D(y_0)\), whereas trajectory PSNRs are computed on the sampler states \(y_t\).
The best PSNR is an oracle diagnostic computed using the clean image; it is not used by the sampler.
We summarize trajectory collapse by
\[
    \Delta_{\mathrm{stab}}
    =
    \mathrm{PSNR}_{\mathrm{final}}
    -
    \mathrm{PSNR}_{\mathrm{best}}.
\]
A stable residual stopping rule has \(\Delta_{\mathrm{stab}}\approx 0\); a trajectory that peaks early and degrades has a large negative gap.
We hold the N2N training regime fixed across both variants: same backbone, same noisy pairs, single noise level.
A multi-noise N2N model would require noisy pairs collected or synthesized at multiple noise levels, which is a different data regime.
The question is therefore whether exact NE improves off-scale iterative reuse under identical N2N supervision.

\begin{table}[!htbp]
    \centering
    \small
    \caption{
        Iterative sampling evaluation on Set12 under N2N training at a single noise level.
        Both models use the same SwinIR backbone and are trained at \(\sigma_{\mathrm{train}}=10\) in 8-bit units.
        The sampler is initialized at \(\sigma_{\mathrm{init}}=10\) in 8-bit units and stopped when \(\widehat{\sigma}_t\le 1/255\).
        Values are means over the 12 Set12 images.
        The initial noisy PSNR is \(28.16\) dB for both rows.
        The best trajectory PSNR is an oracle diagnostic and is not used for stopping.
    }\label{tab:ks_stability}
    \begin{tabular}{lrrrrrr}
        \toprule
        Model
         & One-pass
         & Best traj.
         & Final
         & Final \(-\) best
         & Final \(-\) one-pass
         & Best/final step      \\
        \midrule
        \methodBaseline{}-N2N
         & 35.07
         & 34.20
         & 29.49
         & \(-4.71\)
         & \(-5.58\)
         & \(16.67/40.08\)      \\
        \methodWNE{}-N2N
         & 34.96
         & 34.89
         & 34.79
         & \(-0.10\)
         & \(-0.16\)
         & \(22.75/25.00\)      \\
        \bottomrule
    \end{tabular}
\end{table}

Table~\ref{tab:ks_stability} shows that the unwrapped baseline is not weaker as a one-step denoiser at the training noise level.
Its one-pass PSNR is slightly higher than \methodWNE{}-N2N (\(35.07\) dB versus \(34.96\) dB).
However, when reused iteratively, the baseline trajectory peaks early and then collapses under the residual stopping rule, ending \(4.71\) dB below its best trajectory iterate and \(5.58\) dB below its one-pass output.
In contrast, \methodWNE{}-N2N remains stable.
The stopped final iterate is only \(0.10\) dB below the best trajectory iterate and \(0.16\) dB below the one-pass output.
The advantage of \methodWNE{} here is not better one-step PSNR at \(\sigma_{\mathrm{train}}\), but predictable global-scale behavior under reuse, consistent with the normalized-coordinate mechanism above.

\subsection{Constrained Random Inpainting}\label{sec:kadkhodaie_inverse}

We instantiate the same sampler with a random-inpainting projector \(P=MM^\top\) and observation \(x_c=Px^\star\).
Algorithm~\ref{alg:kadkhodaie_linear_inverse} gives the full procedure.
The reported sampler parameters are on the \([0,1]\) image scale:
\[
    \sigma_0 = 1,\qquad
    \sigma_L = 0.01,\qquad
    h_0 = 0.01,\qquad
    \beta = 0.01.
\]
Images are scaled to \([0,1]\), so in 8-bit units this corresponds to \(\sigma_0=255\) and \(\sigma_L=2.55\).
The maximum budget is \(1000\) iterations.

\begin{algorithm}[!htbp]
    \caption{Kadkhodaie-Simoncelli linear inverse sampler}\label{alg:kadkhodaie_linear_inverse}
    \begin{algorithmic}[1]
        \REQUIRE{} Denoiser \(D\), projector \(P\), observation \(x_c=P x^\star\), parameters \(\sigma_0,\sigma_L,h_0,\beta,T_{\max}\)
        \ENSURE{} Projected reconstruction \(\hat x\)
        \STATE{} Draw \(z\sim\mathcal{N}(0,I_d)\)
        \STATE{} \(y \gets x_c + 0.5(I-P)\ones + \sigma_0 z\)
        \FOR{\(t=1,\ldots,T_{\max}\)}
        \STATE{} \(u \gets (I-P)\bigl(D(y)-y\bigr) + x_c - Py\)
        \STATE{} \(\widehat{\sigma} \gets \|u\|_2/\sqrt{d}\)
        \IF{\(\widehat{\sigma}\leq\sigma_L\)}
        \STATE{} \textbf{break}
        \ENDIF{}
        \STATE{} \(h \gets h_0t/(1+h_0(t-1))\)
        \STATE{} \(\gamma \gets {\left({(1-\beta h)}^2-{(1-h)}^2\right)}^{1/2}\widehat{\sigma}\)
        \STATE{} Draw \(z\sim\mathcal{N}(0,I_d)\)
        \STATE{} \(y \gets y + h u + \gamma z\)
        \ENDFOR{}
        \STATE{} \textbf{return} \(\hat x = x_c+(I-P)y\)
    \end{algorithmic}
\end{algorithm}

This setting is deliberately challenging under the fixed N2N training regime.
Both denoisers are trained at \(\sigma_{\mathrm{train}}=10\), but the sampler initializes from \(\sigma_0=1\), about \(25\times\) above the training scale.
The unwrapped denoiser is therefore queried far outside its training scale at initialization, while \methodWNE{} normalizes the input and exactly rescales the output residual.

We evaluate random inpainting on Set12 with only \(10\%\) observed pixels.
Final reconstructions are projected onto the measurement affine set before computing PSNR and SSIM, so both methods satisfy the measurements exactly at evaluation time.
The one-pass column applies the denoiser once to the sampler initialization \(y_0\) and evaluates the projected estimate \(x_c+(I-P)D(y_0)\).

\begin{figure}[!htbp]
    \centering
    \subfigure[Clean \(x^\star\)]{
        \includegraphics[width=0.235\textwidth]{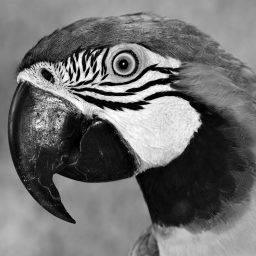}
    }\hfill
    \subfigure[Observed \(P x^\star\)]{
        \includegraphics[width=0.235\textwidth]{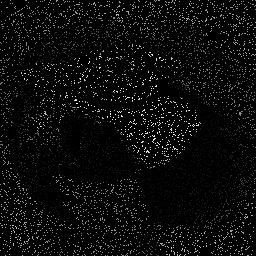}
    }\hfill
    \subfigure[Baseline-N2N]{
        \includegraphics[width=0.235\textwidth]{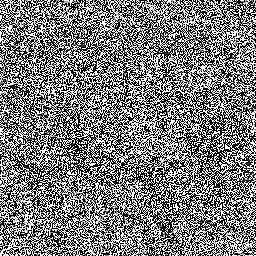}
    }\hfill
    \subfigure[\methodWNE{}-N2N]{
        \includegraphics[width=0.235\textwidth]{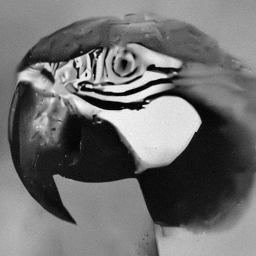}
    }
    \caption{\textbf{Reusing Noise2Noise-trained denoisers in a constrained inverse-problem sampler.}
        Both SwinIR denoisers are trained from noisy pairs at a single noise level, \(\sigma_{\mathrm{train}}=10\), then reused in a random-inpainting sampler with \(10\%\) observed pixels.
        The Baseline-N2N denoiser stays close to the sparse observation, while \methodWNE{}-N2N gives a coherent reconstruction.}%
    \label{fig:n2n_inverse}
\end{figure}

\begin{table}[!htbp]
    \centering
    \caption{
        Kadkhodaie-Simoncelli constrained sampler on Set12 random inpainting with \(10\%\) observed pixels.
        Both denoisers are trained with N2N at a single noise level, \(\sigma_{\mathrm{train}}=10\) in 8-bit units.
        The sampler uses \(\sigma_0=1\), \(\sigma_L=0.01\), \(h_0=0.01\), and \(\beta=0.01\).
        \methodWNE{}-N2N reaches \(23.87\) dB versus \(6.08\) dB for the unwrapped baseline.
        Numbers are means over Set12, with standard deviations in parentheses.
        The column ``gap'' reports final PSNR minus the best PSNR along the trajectory.
    }\label{tab:kadkhodaie_random_missing}
    \newcommand{\tabmeanstd}[2]{\shortstack{\(#1\)\\[-0.5pt]{\((#2)\)}}}
    \setlength{\tabcolsep}{0pt}
    \renewcommand{\arraystretch}{1.12}
    \begin{tabular*}{0.84\linewidth}{@{\extracolsep{\fill}}lccccc@{}}
        \toprule
        Method
        & \shortstack{One-pass}
        & Final
        & SSIM
        & Gap
        & Steps \\
        \midrule
        Observed \(P x^\star\)
        & \textemdash{}
        & \tabmeanstd{5.92}{0.93}
        & \tabmeanstd{0.020}{0.008}
        & \textemdash{}
        & \textemdash{} \\
        \addlinespace[1pt]
        Baseline-N2N
        & \tabmeanstd{6.87}{0.28}
        & \tabmeanstd{6.08}{0.27}
        & \tabmeanstd{0.009}{0.002}
        & \tabmeanstd{-0.91}{0.03}
        & \tabmeanstd{96.8}{1.3} \\
        \addlinespace[1pt]
        \methodWNE{}-N2N
        & \tabmeanstd{14.21}{1.31}
        & \tabmeanstd{\mathbf{23.87}}{2.62}
        & \tabmeanstd{\mathbf{0.716}}{0.056}
        & \tabmeanstd{\mathbf{0.00}}{0.00}
        & \tabmeanstd{277.6}{26.1} \\
        \bottomrule
    \end{tabular*}
\end{table}

Table~\ref{tab:kadkhodaie_random_missing} shows that, under this high-noise initialization and fixed sampler configuration, the unwrapped denoiser stays near the observed projection and ends at \(6.08\) dB.
\methodWNE{} makes the same denoiser usable inside the sampler and reaches \(23.87\) dB, an improvement of about \(17.8\) dB using the same backbone, training pairs, and sampler hyperparameters.
The result instantiates the normalized-coordinate mechanism in the inverse-problem regime.
The original sampler is a coarse-to-fine procedure initialized at full-scale Gaussian noise: natural for the sampler, but far outside the N2N training scale.
Under this initialization, the vanilla denoiser produces an uninformative residual.
\methodWNE{} preserves the same training objective and architecture, but routes every call through the same normalized regression \(g\), so the denoiser can be reused from the high-noise initialization and called across the sampler's noise schedule without changing the sampler hyperparameters.
\section{Input-Only Normalization Versus the Wrapper (IN vs WNE)}%
\label{app:in_vs_wne}
\newcommand{\methodIN}{IN}
In this appendix section we test whether applying an instance-normalization style transform at the input only can replicate the robustness benefits of the full input-output NE wrapper.
All figures in this section use a DnCNN backbone and single-noise training at \(\sigma_{\mathrm{train}}=10\).
Recall the NE normalization map \(\Tne\) from Section~\ref{sec:method} and let \(g_\theta\) be a standard DnCNN backbone.

\paragraph{\methodWNE{}.}
We use the wrapped predictor
\[
    f_\theta^{\mathrm{WNE}}(y)
    := \std(y)\,g_\theta(\Tne(y))+\mu(y)\ones,
\]
which is normalization-equivariant by construction (Proposition~\ref{prop:wrapper_is_ne}).

\paragraph{\methodIN{} (input-only).}
We define the input-only normalization baseline by applying the same normalization at the input and not inverting it:
\[
    f_\theta^{\methodIN}(y) := g_\theta(\Tne(y)).
\]
The map \(\Tne(y)=(y-\mu(y)\ones)/\std(y)\) has the same algebraic form as instance normalization~\citep{ulyanov2016instance}.
Our statistics \(\mu(y)\) and \(\std(y)\) are pooled jointly over all entries (all channels and pixels), rather than per-channel, because the NE group action \(y\mapsto ay+b\ones\) acts identically on all entries.

\paragraph{Structural limitation of \methodIN{}.}
Because \(\Tne\) is invariant to global contrast and brightness,
\[
    \Tne(ay+b\ones)=\Tne(y)
    \qquad (a>0,\ b\in\R)
\]
(Proposition~\ref{prop:mu_sigma_affine}),
the input-only baseline satisfies
\[
    f_\theta^{\methodIN}(ay+b\ones)=f_\theta^{\methodIN}(y),
\]
so it discards global shift and scale information that an NE denoiser must propagate to the output.
In contrast, \methodWNE{} re-injects \(\mu(y)\) and \(\std(y)\) analytically at the output.

\paragraph{Protocol.}
We follow the same noise-level mismatch diagnostic as in Section~\ref{sec:experiments}:
for each \(\sigma_{\mathrm{test}}\), we plot output PSNR versus input PSNR on Set12.

\paragraph{Results.}
Figure~\ref{fig:dncnn_sigma10_in_vs_wne_psnr} shows that \methodIN{} has lower PSNR than \methodBaseline{}, including at matched noise (\(\sigma_{\mathrm{test}}=\sigma_{\mathrm{train}}\)), while \methodWNE{} remains stable under mismatch.
Figure~\ref{fig:dncnn_sigma10_in_vs_wne_qual} illustrates the same effect qualitatively.

\begin{figure}[!htbp]
    \centering
    \includegraphics[width=0.45\linewidth]{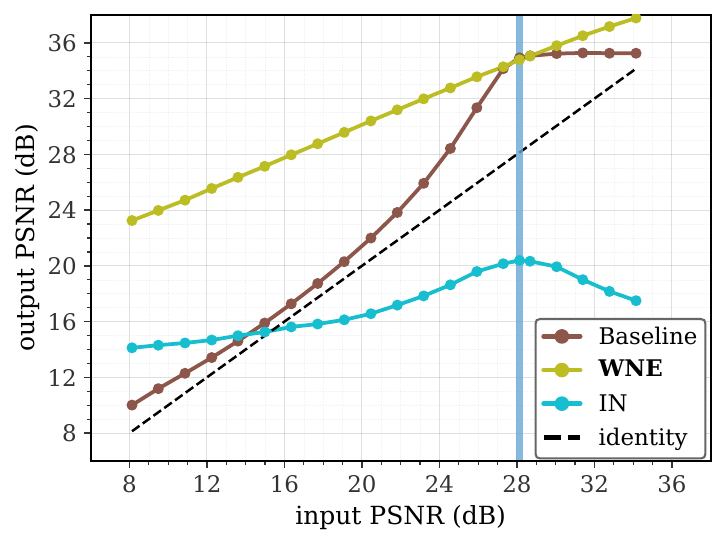}%
    \caption{\textbf{DnCNN:\@ input-output PSNR diagnostic at \(\sigma_{\mathrm{train}}=10\) (\methodIN{} vs \methodWNE).}
        Output PSNR versus input PSNR on Set12 as \(\sigma_{\mathrm{test}}\) varies.
        The vertical line marks \(\sigma_{\mathrm{train}}=10\).
        The dashed identity line corresponds to no denoising (\(\hat x=y\)).
        Input-only normalization (\methodIN) performs poorly even at matched noise and degrades further under mismatch, while \methodWNE{} remains robust.}%
    \label{fig:dncnn_sigma10_in_vs_wne_psnr}
\end{figure}

\begin{figure}[!htbp]
    \centering
    \qualpanel{0.30\textwidth}{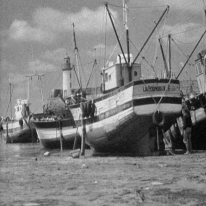}{\shortstack{Noisy image\\\(\sigma=10\)}}\hfill
    \qualpanel{0.30\textwidth}{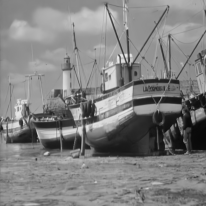}{DnCNN-\methodWNE{}}\hfill
    \qualpanel{0.30\textwidth}{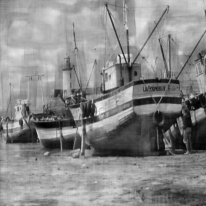}{DnCNN-\methodIN{}}%
    \caption{\textbf{DnCNN qualitative comparison (train \(\sigma_{\mathrm{train}}=10\)).}
        The \methodIN{} baseline applies \(\Tne\) only at the input (no analytic inverse), discarding global brightness and contrast information, and fails severely.
        The \methodWNE{} wrapper re-injects \(\mu(y)\) and \(\std(y)\) at the output and remains qualitatively stable.}%
    \label{fig:dncnn_sigma10_in_vs_wne_qual}
\end{figure}

\section{Non-Gaussian Corruptions}%
\label{app:nongaussian}

This appendix tests whether the robustness benefits of enforcing NE extend beyond AWGN.\@
We consider three additive non-Gaussian noises (Uniform, Laplace, Rayleigh) and a non-additive corruption (JPEG artifacts), using the same input-output PSNR diagnostic as in Section~\ref{sec:experiments}.
We evaluate both the DnCNN family and SwinIR.\@
For the DnCNN-family architectural references, FDnCNN is used for \methodSEarch{} and \methodNEarch{}, while the standard DnCNN backbone is kept for \methodBaseline{} and \methodWNE{}.
For SwinIR, we compare \methodBaseline{} versus \methodWNE{}.\@

\paragraph{Noise models.}
For the additive cases we generate \(y=x+\sigma\varepsilon\), where \(\varepsilon\) is scaled to have unit standard deviation and \(\sigma\) controls corruption strength.
Uniform and Laplace noises are zero-mean.
For Rayleigh, we do not recenter the noise: it has nonzero mean even after scaling to unit standard deviation.
Thus this Rayleigh setting introduces a systematic positive brightness bias in expectation, in addition to pixelwise randomness.
JPEG artifacts are produced by compressing and decompressing the image at a chosen quality factor.

\paragraph{Protocol.}
For the additive noises, each model is trained at a single corruption strength (vertical reference line) and evaluated at multiple strengths by varying \(\sigma_{\mathrm{test}}\).
We plot output PSNR versus input PSNR on Set12; the dashed identity line corresponds to no denoising (\(\hat x=y\)).
For JPEG, we plot output PSNR versus quality factor, and we evaluate a simple transfer setting by applying the model trained on Uniform noise to JPEG-corrupted inputs.

\begin{figure}[H]
    \centering
    \subfigure[\textbf{Uniform} (train \(\sigma=14\))]{
        \includegraphics[width=0.42\textwidth]{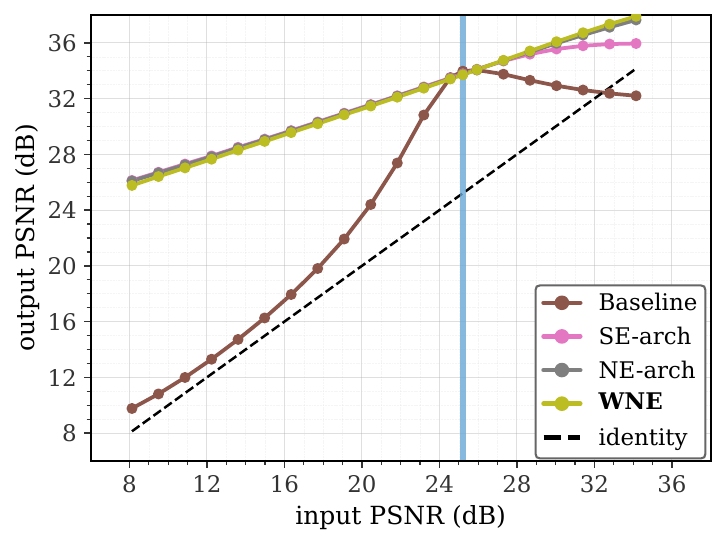}
    }\hfill
    \subfigure[\textbf{Laplace} (train \(\sigma=20\))]{
        \includegraphics[width=0.42\textwidth]{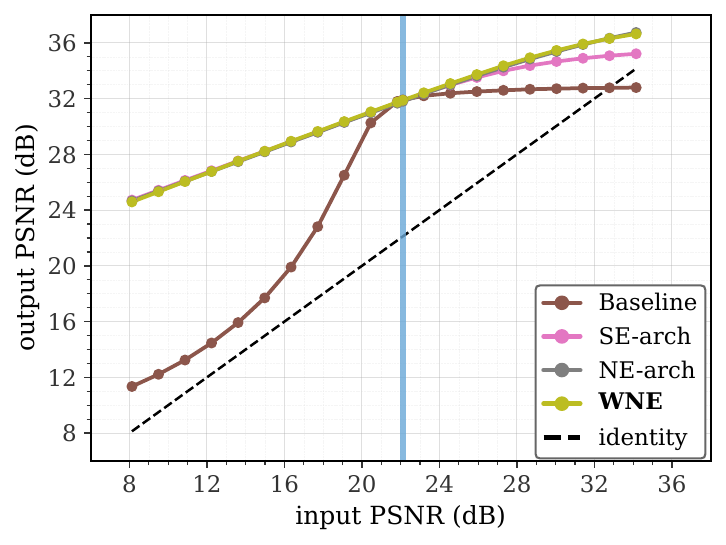}
    }\\[-1pt]
    \subfigure[\textbf{Rayleigh} (train \(\sigma=17\); not recentered)]{
        \includegraphics[width=0.42\textwidth]{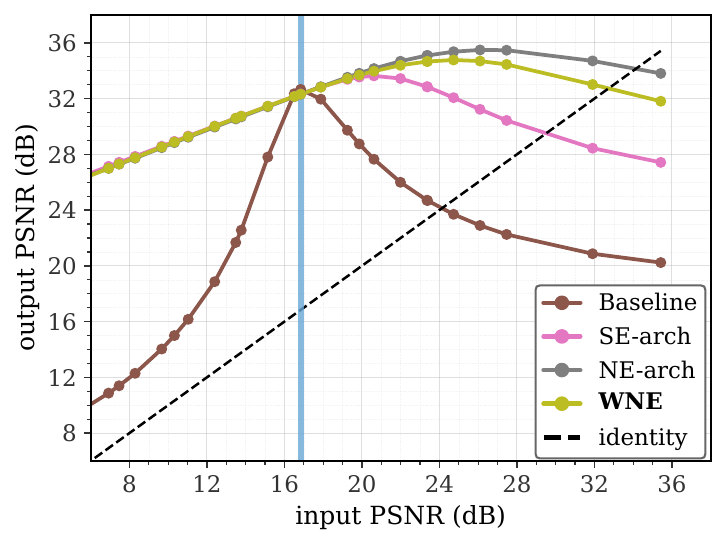}
    }\hfill
    \subfigure[\textbf{JPEG} (quality factor)]{
        \includegraphics[width=0.42\textwidth]{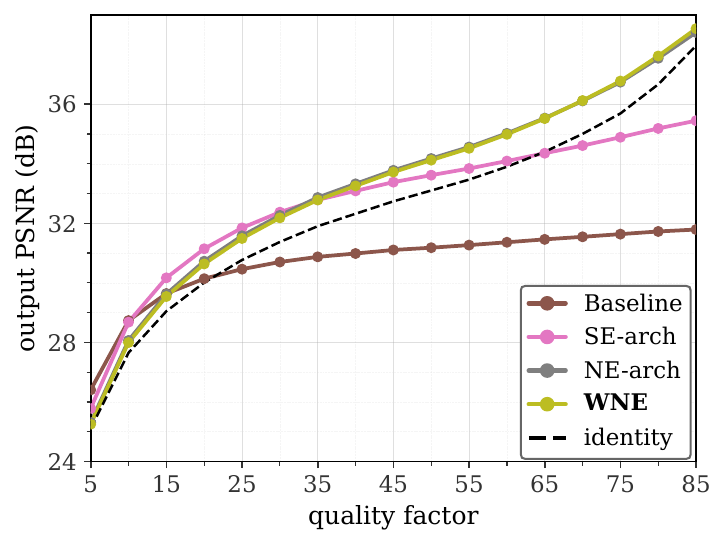}
    }
    \caption{\textbf{Non-Gaussian corruptions (DnCNN family).}
        Input-output PSNR diagnostic on Set12 for three additive non-Gaussian noises (Uniform, Laplace, Rayleigh) and robustness to JPEG artifacts.
        All models are blind.
        For additive noises, the vertical reference line marks the training corruption strength and moving away corresponds to testing at unseen strengths.
        The dashed identity line is the no-denoising baseline \(\hat x=y\).
        Rayleigh noise is not recentered in our implementation (nonzero mean), so it also induces an expected brightness bias.
        For JPEG, we plot output PSNR versus quality factor and evaluate transfer by applying the model trained on Uniform noise to JPEG-corrupted inputs.}%
    \label{fig:app_nongaussian}
\end{figure}

\paragraph{Results on the DnCNN family.}
Figure~\ref{fig:app_nongaussian} summarizes the DnCNN-family findings.
For Uniform and Laplace, \methodWNE{} closely matches the architectural normalization-equivariant baseline \methodNEarch{} and improves over \methodSEarch{} and the unwrapped baseline.
For JPEG artifacts, \methodWNE{} also remains competitive and yields a consistent improvement over the baseline across quality factors in this setting.
Rayleigh is the only setup where \methodWNE{} results in a lower PSNR curve than \methodNEarch{}.
This is also the only noise model here that is not zero-mean in our implementation (we do not recenter the Rayleigh samples).

\paragraph{Results on SwinIR.}
Figure~\ref{fig:swinir_nongaussian} shows the same extension on the transformer backbone.
Across Uniform, Laplace, Rayleigh, and JPEG corruptions, \methodWNE{} improves over the SwinIR baseline, consistent with the DnCNN-family results above.

\begin{figure}[H]
    \centering
    \subfigure[\textbf{Uniform} (train \(\sigma=14\))]{
        \includegraphics[width=0.42\textwidth]{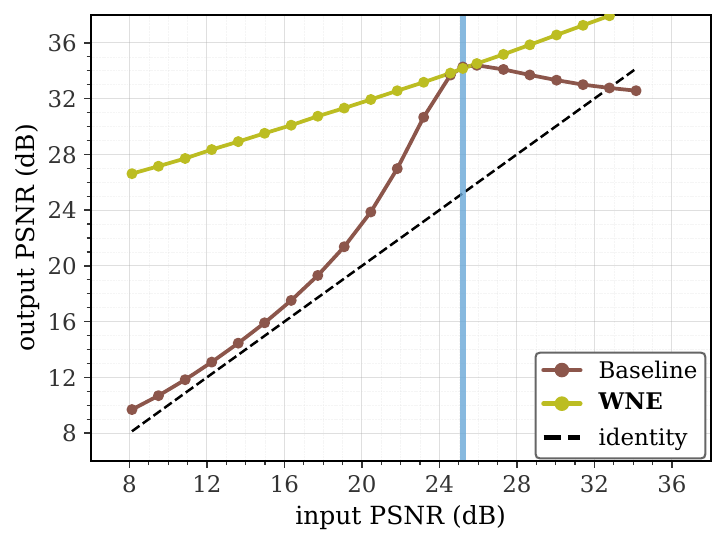}
    }\hfill
    \subfigure[\textbf{Laplace} (train \(\sigma=20\))]{
        \includegraphics[width=0.42\textwidth]{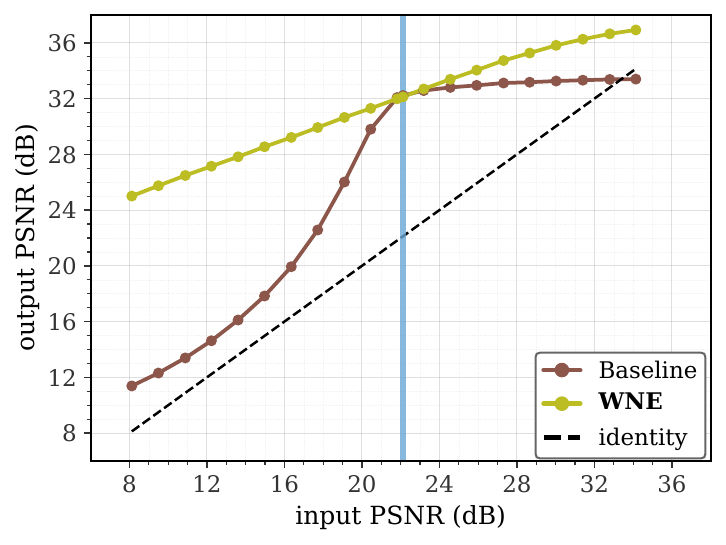}
    }\\[-1pt]
    \subfigure[\textbf{Rayleigh} (train \(\sigma=17\); not recentered)]{
        \includegraphics[width=0.42\textwidth]{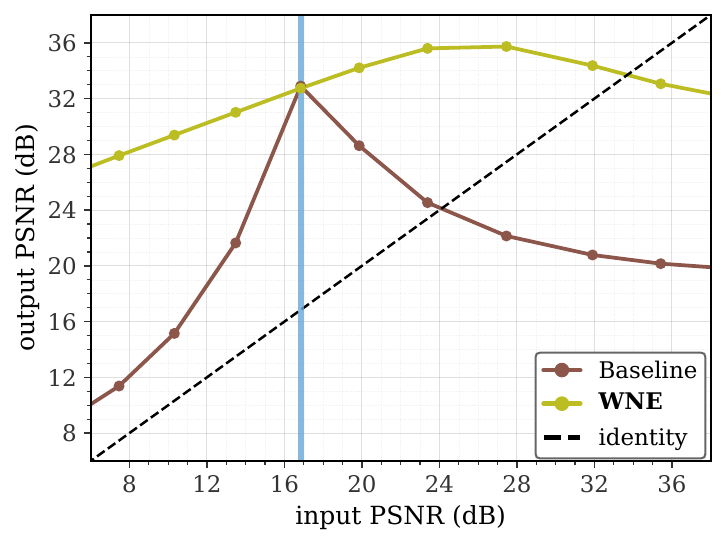}
    }\hfill
    \subfigure[\textbf{JPEG} (quality factor)]{
        \includegraphics[width=0.42\textwidth]{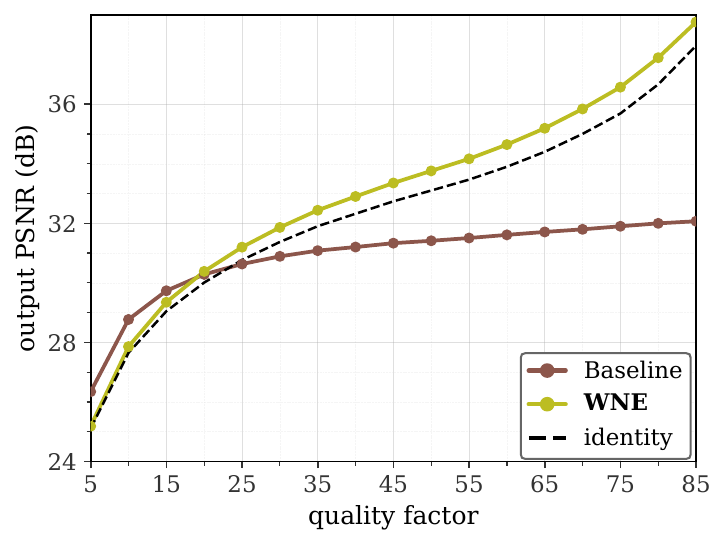}
    }
    \caption{\textbf{Non-Gaussian corruptions (SwinIR).}
        Same diagnostic as Figure~\ref{fig:app_nongaussian}, applied to SwinIR \methodBaseline{} versus SwinIR-\methodWNE{}.
        \methodWNE{} improves over the baseline across all four corruption types, consistent with the DnCNN-family results.}%
    \label{fig:swinir_nongaussian}
\end{figure}

\end{document}